\documentclass[11pt, oneside]{article}
\usepackage{geometry}

\usepackage{times}
\usepackage{comment}
\usepackage{enumitem}
\usepackage{amsmath, amsthm, amssymb, natbib, graphicx, url, algorithm2e}
\usepackage{bbm}
\usepackage{xcolor}
\usepackage{hyperref} 
\usepackage{url}  
\usepackage[nameinlink,capitalize]{cleveref}
\definecolor{dgreen}{rgb}{0,0.5,0}
\hypersetup{
  colorlinks=true,
  linkcolor=blue,
  filecolor=blue,
  citecolor = dgreen,      
  urlcolor=cyan,
}

\newcommand{\keygen}{\text{KeyGen}}
\newcommand{\verify}{\text{Verify}}
\newcommand{\sign}{\text{Sign}}
\newcommand{\msg}{\text{msg}}
\newcommand{\sig}{\text{sig}}
\newcommand{\pk}{\text{pk}}
\newcommand{\sk}{\text{sk}}
\newcommand{\R}{\mathbb{R}}
\newcommand{\Z}{\mathbb{Z}}
\newcommand{\N}{\mathbb{N}}
\newcommand{\supp}{\text{supp}}
\newcommand{\poly}{\text{poly}}
\newcommand{\polylog}{\text{polylog}}
\theoremstyle{plain}
\newtheorem{theorem}{Theorem}
\newtheorem*{theorem*}{Theorem}

\newtheorem{lemma}[theorem]{Lemma}
\newtheorem{corollary}[theorem]{Corollary}

\newtheorem{proposition}[theorem]{Proposition}

\theoremstyle{definition}
\newtheorem{definition}[theorem]{Definition}

\newtheorem{remark}[theorem]{Remark}

\newcommand{\Var}{\mathrm{Var}}
\newcommand{\E}{\mathbb{E}}
\newcommand{\CDH}{\mathrm{CDH}}
\newcommand{\norm}[1]{\| #1\|}
\newcommand{\valid}{\text{valid}}
\newcommand{\set}[1]{\{#1\}}

\begin{document}
\title{A computational phase transition for learning-to-sample\\
from Ising models}
\author{Andrej Risteski\thanks{Machine Learning Department, Carnegie Mellon University. \texttt{aristesk@andrew.cmu.edu}} \and Thuy-Duong Vuong\thanks{Department of Computer Science and Engineering, UC San Diego. \texttt{thvuong@ucsd.edu}}\footnote{Authors are listed in alphabetical order}}
\maketitle

\begin{abstract}%
We study \emph{learning-to-sample}---a basic algorithmic task underlying generative modeling---for Ising models, a standard testbed for algorithmic ideas in both theoretical
computer science and machine learning. Given i.i.d.\ samples of an unknown target distribution, the goal of learning-to-sample is to learn a computationally efficient generation procedure that produces new samples following approximately the same distribution.
We construct a family of Ising models of constantly bounded-width which lie just beyond the spectral threshold $\lambda_{\max}(J)-\lambda_{\min}(J)=1$, and show that learning-to-sample for this family is computationally hard under standard cryptographic assumptions, even when the learner is given both polynomially many i.i.d.\ samples from the model and explicit access to its parameters. Combined with results of \citep{AJKPV24,KLV25} showing tractability of learning-to-sample below the spectral threshold, this establishes a sharp computational phase transition at the spectral threshold. Moreover, combined with prior results on parameter learning for bounded-width Ising models \citep{KM17,WSD19,VML20}, this shows that learning-to-sample can be more difficult than parameter learning. Finally, we show that any efficient learner for these hard instances exhibits a natural memorization-hallucination dichotomy: the learner must either output configurations that, after a simple transformation, match the (transformed) training data or place substantial mass on configurations of negligible probability under the target distribution.
\end{abstract}

\section{Introduction}

Modern generative models---including GANs \citep{goodfellow2014generative}, variational autoencoders \citep{kingma2013auto}, and diffusion models \citep{song2020score}---can all be viewed, at a high level, as attempts to solve the same algorithmic task: from i.i.d.\ samples of an unknown target distribution, learn an efficient procedure that generates new samples following approximately the same law. We study this task in its most abstract form, known as \emph{learning-to-sample}\footnote{This task has been studied since the 1990s under the name ``learning a generator" \cite{KMRRSS94}; however, we adopt the more modern and popular term ``learning-to-sample," introduced by \cite{moitra2021learning}}. By adopting this viewpoint, we can study the underlying principles of generative modeling, and its core challenges such as memorization and hallucination, without being constrained by architecture-specific details.

From a complexity-theoretic perspective, learning-to-sample is distinct from both \emph{parameter learning} and \emph{density estimation}. In parameter learning, the goal is to estimate the underlying model from data; in density estimation, the goal is to output a function that approximates the target density. By contrast, learning-to-sample asks for an efficient generator of fresh samples. A natural strategy is to first learn the parameters and then sample from the learned model. But for many natural families, sampling remains computationally hard even when the parameters are known exactly. A classical example is the Ising model, for which approximate sampling is intractable in general \citep{sly2010computational,SS12}.

This raises a basic question: can access to training samples, together with the freedom to design a new sampler, circumvent the computational barriers that arise when one is given only the model parameters? We show that the answer is no. In fact, we prove a stronger statement: learning-to-sample remains computationally intractable even when the learner is given both polynomially many samples from the target distribution and explicit access to the model parameters. Our hard instances moreover lie in a regime where parameter learning is known to be efficient.

Precisely, our main result concerns Ising models\footnote{See \cref{def:Ising model} for further details} near the spectral threshold $\lambda_{\max}(J)-\lambda_{\min}(J)=1,$
which marks a tractable regime for sampling with parameter access \citep{EKZ20,AJKPV21} and, more recently, for learning-to-sample from polynomially many samples \citep{AJKPV24,KLV25}\footnote{See the second main result of \cite{KLV25} on learning-to-sample for Ising models, which applies for a broader class of Ising models that includes those where the eigenvalues of the interaction matrix lie in the interval $[0,1).$ Note that for any Ising model with $\lambda_{\max}(J)-\lambda_{\min}(J)< 1,$ by adding a multiple of the identity to the interaction matrix $J$, we can ensure that the eigenvalues of the interaction matrix lie in the interval $[0,1)$ while not changing the underlying distribution.}. We show that for every constant $\gamma>1$, there are Ising models whose interaction matrix $J$ and external field $h$ satisfy:
\[
1< \lambda_{\max}(J)-\lambda_{\min}(J)\le \gamma
\qquad\text{and}\qquad
\norm{J}_{\infty}+\norm{h}_{\infty}=O(1)
\]
for which learning-to-sample is still computationally hard. Thus, from the perspective of learning-to-sample with access to both samples and parameters, there is a sharp computational transition at the same spectral threshold.

\subsection{Contributions}

\noindent\textbf{Hardness arbitrarily close to the spectral threshold.}
Under standard cryptographic assumptions, we construct Ising models that lie an arbitrarily small additive factor beyond the tractable regime $\lambda_{\max}(J)-\lambda_{\min}(J)\le 1$, yet remain hard for learning-to-sample even when the learner is given the model parameters together with polynomially many training samples.

\medskip
\noindent\textbf{A separation from parameter learning.}
Our hard instances lie in a regime where parameter learning is efficient by known results \citep{KM17,WSD19,VML20}. To the best of our knowledge, this gives the first separation under standard cryptographic assumptions in which parameter learning is computationally easy but learning-to-sample is hard.

\medskip
\noindent\textbf{A formal memorization--hallucination dichotomy.}
One empirical concern in generative modeling is that a learned generator may either memorize its training set or hallucinate unrealistic outputs. We prove an analogue of this phenomenon for our hard Ising instances: any efficient learner must either output configurations that, after a simple transformation, match the transformed training data (i.e. memorize) or place substantial mass on configurations of negligible probability under the target distribution (i.e. hallucinate). Specifically, our hard instances are Ising models on $N$ vertices, where there exists a simple, efficiently computable map $\Psi:\set{\pm 1}^N \to \set{0,1}^d$ with $d= \Omega(N^{\Theta(1)})$, so that any efficient learner must either generate configurations whose images under $\Psi$ match those of the training data, or must output configurations with negligibly small mass under the target distribution.
In particular, this phenomenon rules out the possibility of an efficient algorithm that can reliably learn-to-sample\footnote{See \cref{remark:relating hard for learning to generalize with learning to sample,remark:relating hard for learning to generalize in family of distribution vs for single distribution,remark:learning pk from samples and relation to sample amplification}.\label{footnoteref:connection to learning to sample and existence of efficient discriminator}}, as well as the possibility of sample amplification \citep{AGSV20}\footnote{See \cref{remark:relating hard for learning to generalize with sample amplification,remark:relating hard for learning to generalize in family of distribution vs for single distribution,remark:learning pk from samples and relation to sample amplification}\label{footnoteref:connection to sample amplification}}.
\newpage
Informally, our main result is as follows (the formal result is in \cref{cor:main hardness of near critical Ising model from computational diffie hellman}): 

\begin{theorem*}[Informal]\label{thm:main-informal}
Assume a standard hardness hypothesis for the computational Diffie--Hellman problem. Then for every constant $\gamma>1$ there exists an explicit family of Ising models on $N$ vertices such that for every Ising model $\mu_{J,h}$ in this family, the parameters $J$, $h$ satisfy:
\[
1< \lambda_{\max}(J)-\lambda_{\min}(J)\le \gamma
\qquad\text{and}\qquad
\norm{J}_\infty+\norm{h}_\infty=O(1),
\]
but  no efficient algorithm can reliably learn to generate fresh
typical samples from a model $\mu_{J,h}$ randomly drawn from this family, even when the algorithm is
given both the parameters $(J,h)$ and polynomially many i.i.d.\ samples from
$\mu_{J,h}$.


Moreover, there exists an efficiently computable map $\Psi$ on $\{\pm1\}^N$ and for each distribution $\mu_{J,h}$, a set $\Lambda_{J,h}\subseteq\{\pm1\}^N$ with
\[
\mu_{J,h}(\Lambda_{J,h})\le \exp(-N^{\Omega(1)})
\]
such that for every efficient learner that takes the parameter $(J,h)$ of an Ising model randomly drawn from this family and a training multiset $X$ of polynomially many i.i.d.\ samples from $\mu_{J,h}$ as input, the output $y$ satisfies
\[
\Pr\!\Big[y\in \Psi^{-1}(\Psi(X))\cup \Lambda_{J,h}\Big]\ge 1-\exp(-N^{\Omega(1)}).
\]
\end{theorem*}

The set $\Lambda_{J,h}$ should be thought of as the \emph{hallucination set}: it has negligible mass under the target law $\mu_{J,h}$. The preimage $\Psi^{-1}(\Psi(X))$ captures the \emph{memorization set}. The full reduction in \cref{cor:main hardness of near critical Ising model from computational diffie hellman} is more general and tracks the tradeoff between the assumed hardness of CDH and the allowable sample/time budget of the learner (see also \cref{remark:concrete hardness parameters based on CDH}). Moreover, there exists an efficient discriminator that, given access to the training multiset, can distinguish between samples drawn from the target distribution and those produced by any efficient learner (see \cref{footnoteref:connection to learning to sample and existence of efficient discriminator}). The existence of such a discriminator is particularly relevant in the context of  Generative Adversarial Networks (GANs).  

We defer further discussions to \cref{sec:mainresults} and \cref{sec:puttingtogether}.

\subsection{Related work}
The first work to formalize the learning-to-sample problem is \cite{KMRRSS94}, albeit under the name ``learning a generator". Additionally, \cite{KMRRSS94} constructs a hard instance for this problem in the setting where the learner only has sample access to the target distribution, based on the existence of a pseudorandom function. Specifically, if $f:\set{0,1}^n \to \set{0,1}^n$ is a pseudorandom function, the hard instance is the distribution $\mathcal{D}_f$ on $\set{0,1}^{2n}$, where the first $n$ bits are uniformly random and the last $n$ bits are the result of applying $f$ to the first $n$ bits. Note that $\mathcal{D}_f$  is \emph{not} hard if the learner has an explicit description\footnote{Here, explicit description means a circuit that implements the function $f.$} of the function $f$. Furthermore, unlike with our hard instance, the algorithm that simply samples uniformly from $\set{0,1}^{2n}$ can convince any efficient discriminator that it has successfully learned to sample from $\mathcal{D}_f,$
since samples from $\mathcal{D}_f$ are computationally indistinguishable from uniformly random strings in $\set{0,1}^{2n}.$  

\cite{BMS08} introduced the idea of embedding a general circuit into Markov random fields (MRFs), a class of distributions that generalizes Ising models, but \cite{BMS08}'s results are on the hypothesis testing and total-variation estimation problems for \emph{censored} MRFs, i.e. the induced distribution on a subset of vertices in an MRF. More recently, \cite{moitra2021learning} revived the interest in the learning-to-sample problem, and they are also the first in the TCS literature to name it ``learning-to-sample''. Building on \cite{BMS08}, \cite{moitra2021learning} show the hardness of learning-to-sample on censored MRFs in the sample-only setting. 

Our result differs from and strengthens prior works in several ways: (1) we work with ordinary Ising models rather than censored ones; (2) our hard instance holds even when the learner is given the Ising model's parameters in addition to samples; (3) we establish a hard instance in the near-critical regime, and thus show a sharp computational phase transition for the learning-to-sample problem; (4) our hard instance enjoys an efficient discriminator that differentiates between samples from the true data distribution and those produced by any efficient learner.


A complementary result of \cite{KLV25} shows that learning-to-sample can be easier than parameter learning on certain Ising-model families. Our result goes in the opposite direction: we give a regime in which parameter learning is efficient but learning-to-sample is hard. This shows that these two tasks can be separated in both directions.

Recently, there has been interest in proving hardness for other, more specific generative frameworks ---mostly those closely related to diffusion models. For example, \citep{ghio2024sampling, song2024cryptographic, MV25, chewi2025ddpm} study the hardness of estimating the denoising score functions---the crucial object used to train and run a diffusion model---in the sample-only setting. Our focus here is different: rather than focusing on a particular generative model, we study the general learning-to-sample task. In particular, our result addresses the question of showing hardness for learning-to-sample raised by \cite[Open Problem 5]{chewi2025ddpm}, though for Ising models rather than for Gaussian mixtures.

The memorize-or-hallucinate property of our hard instances is also related to the notion of \emph{sample amplification}, first introduced by \cite{AGSV20}. In sample amplification, the goal is not to output a sampler, but merely to transform $k$ samples from an unknown distribution into a larger batch of $\ell>k$ samples whose joint law is close to that of $\ell$ i.i.d.\ draws. This is formally weaker than learning-to-sample, but is also ruled out for our hard instances; see \cref{footnoteref:connection to sample amplification}.

\subsection{Technical overview}
The construction of the Ising models in \cref{thm:main-informal} is based on two key ideas. 
The first idea is to embed a function into the Hamiltonian of an Ising model, where the function is computationally easy to evaluate, but sampling from its pre-image remains computationally hard, even with access to pre-image samples. This asymmetry matches the structure of digital signature schemes. Consequently, using the circuit embedding technique of \citep{moitra2021learning,BMS08}, we embed the verification algorithm of a signature scheme, which checks the validity of a message-signature pair, into a family of Ising models. The key insight is that sampling from these Ising models corresponds to generating valid message-signature pairs---or equivalently, forging signatures. By the standard security definition of signature schemes, no efficient algorithm can forge a signature for a new message $\msg^*$ after querying signatures for messages $\msg^1, \dots, \msg^k$, where $\msg^* \neq \msg^i$ for all $i \in \{1, \dots, k\}$. For our main result, we use a specific signature scheme that is secure under the computational Diffie-Hellman (CDH) assumption; consequently, the constructed Ising models are hard for learning-to-sample under the same assumption\footnote{Our techniques also apply to other signature schemes; see \cref{thm:main hardness of near critical Ising based on signature scheme} and \cref{remark:relaxing the regularity condition}.}.
We will sketch the idea, and we refer to \cref{sec:embed general signature scheme into Ising} for details of the construction and the formal proof.

Suppose there exists an algorithm $\mathcal{A}$ that can learn-to-sample Ising models from this family: we can then construct an adversary $\mathcal{B}$ that forges signatures. Namely, $\mathcal{B}$ computes the parameters of the target distribution using the code of the verification algorithm, queries for signatures of multiple messages, and provides these message-signature pairs as training samples  to $\mathcal{A}$, along with the model parameters. If $\mathcal{A}$ generates a valid message-signature pair for a new message not in the queried set, it breaks the security of the signature scheme. Otherwise, $\mathcal{A}$ either has to hallucinate (i.e. produce invalid message-signature pairs) or memorize (i.e. produce valid message-signature pairs only for previously queried messages).
We formalize this hallucination-memorization phenomenon as hardness for \emph{learning-to-generalize}, and discuss its relationship with learning-to-sample in \cref{subsec:learning to generalize prelim}.

The second idea is to embed an Ising model $\mu$ from this family into a new Ising model $\tilde{\mu}$ with constantly-bounded width that is near the phase transition threshold, while preserving the hardness of learning-to-generalize. We sketch this construction, and refer to \cref{sec:embed general Ising into near critical Ising} for full details.
While our construction draws inspiration from prior works \citep{sly2010computational, SS12, GKK24}, it is more general: we can handle an Ising model $\mu$ with arbitrary external fields and interactions, whereas \citep {sly2010computational, SS12, GKK24} focus on MAXCUT instances (i.e., no external fields and only anti-ferromagnetic interactions). 

We identify an Ising model with a weighted graph in which the interaction matrix specifies the edge weights and the external field specifies the vertex weights. Our construction transforms the weighted graph $H$ corresponding to an Ising model $\mu$ into the weighted graph $\mathcal{H}$ corresponding to $\tilde{\mu}$ as follows. For each vertex $v$ in $H,$ we construct a \emph{gadget graph} $G_v$ on $W_v = S^{\text{edge}}_v \cup S^{\text{vertex}}_v \cup R_v $, where $S^{\text{edge}}_v , S^{\text{vertex}}_v, R_v $ are pairwise disjoint, and $G_v$ is the complete graph on $W_v$ with edges between vertices in $S^{\text{edge}}_v \cup S^{\text{vertex}}_v $ removed.
Each vertex in $S^{\text{vertex}}_v$ receives a uniform vertex weight with sufficiently small absolute value and the same sign as $h_v,$ the vertex weight of $v$ in $H.$ 
The set $ S^{\text{edge}}_v$ is further partitioned into disjoint subsets $S_v^u$, one for each neighbor $u$ of $v$ in $H.$
For neighboring vertices $u,v$ in $H$, we require  $|S_v^u|=|S_u^v|$, and add a perfect matching between $S_v^u$ and $S_u^v;$ each matching edge has the same weight, where the weight has sufficiently small absolute value and the same sign as $J_{uv}$, the weight of the edge $\set{u,v}$ in $H.$
The graph $\mathcal{H}$ is then defined as the union of the gadget graphs $ G_v$ for all $v$ in $H$ together with the edges in the matchings.

For a configuration $\sigma$ of $\tilde{\mu},$ we let $y_v(\sigma)=\text{sign}(\sum_{i\in R_v}\sigma_i)$\footnote{We choose $ |R_v|$ to be odd, ensuring that $\sum_{i\in R_v}\sigma_i$ is never $0$.} be the \emph{phase} of the gadget graph $G_v,$ and let $Y(\sigma) = (y_v(\sigma) )_{v\in V_H}$ be the \emph{phase vector} of $\sigma$. For a given constant $\epsilon\in (0,1),$
by appropriately choosing the size of $S_v^{\text{vertex}}$, $S_v^u$ for neighbors $u$ of $v,$ and $R_v,$ we can ensure that the resulting model $\tilde{\mu}:\set{\pm 1}^N \to \R_{\geq 0}$ approximates $\mu: \set{\pm 1}^m\to \R_{\geq 0}$ \footnote{where $\mu, \tilde{\mu}$ are unnormalized density functions of the Ising model corresponding to $H, \mathcal{H}$ respectively, and $N$ has polynomial dependence on $m$ and $1/\epsilon$.} in the following sense:

\begin{equation}\label{eq:new Ising model approximate original Ising model informal}
 \forall y\in \set{\pm 1}^m:  (1-\epsilon) \Pr_{y\sim \mu}[y]\leq  \Pr_{\sigma \sim \tilde{\mu}}[Y(\sigma) = y] \leq (1+\epsilon) \Pr_{y\sim \mu}[y]
\end{equation}
 Moreover, denoting $Z_{\tilde{\mu}}(y) =\sum_{\sigma\in \set{\pm 1}^N: Y(\sigma) = y} \tilde{\mu}(\sigma),  $ we can compute a normalization factor $A$ so that:
\begin{equation}\label{eq:new Ising model approximate original Ising model without normalization informal}
  \forall y\in \set{\pm 1}^m:  (1-\epsilon) \mu(y) \leq  \frac{Z_{\tilde{\mu}}(y)}{A}  \leq (1+\epsilon) \mu(y)\footnote{Where $\tilde{\mu}(\sigma)$ and $\mu(y)$ refer to the unnormalized densities.}
\end{equation}

Next, we show that if the hallucination-memorization phenomenon occurs in $\mu$, it must also occur in $\tilde{\mu}$. Equivalently, we show that an algorithm $\tilde{\mathcal{A}},$ which uses $\tilde{k}$ samples from $\tilde{\mu}$ and avoids the hallucination-memorization phenomenon, can be used to construct an algorithm $\mathcal{A}$ which uses $k \approx \tilde{k}$ samples from $\mu$ while also avoiding the hallucination-memorization phenomenon. This reduction is a key technical contribution, and formally corresponds to part (3) of \cref{thm:set parameter choice}. A crucial step in the reduction is to show that, given $k$ samples from $\mu$, we can efficiently produce $\tilde{k}$ samples from $\tilde{\mu}$, with only (negligibly) small errors in total variation distance.

Motivated by \cref{eq:new Ising model approximate original Ising model informal}, a natural strategy is to treat a sample $y$ drawn from $\mu$ as a sample from $Y_*\tilde{\mu}$, the push-forward distribution of $\tilde{\mu}$ by the map $Y$, then to sample from the conditional distribution $\tilde{\mu}(\cdot |Y(\sigma) = y).$ However, this naive strategy fails since the distribution $Y_*\tilde{\mu}$ is not sufficiently close to $ \mu$: $ d_{TV}(Y_*\tilde{\mu}, \mu) \approx \epsilon,$ where we recall that $\epsilon = \Omega(1)$ is the constant from \cref{eq:new Ising model approximate original Ising model informal}. Hence, if $ \tilde{k} = \omega(1),$ then the distribution of the $\tilde{k}$ samples produced by this strategy will be very far from that of $\tilde{k}$ i.i.d. samples from $\tilde{\mu}.$ We circumvent this difficulty using rejection sampling. Rather than naively treating samples from $\mu $ as samples from $Y_*\tilde{\mu} ,$ we apply rejection sampling to a stream of samples from $\mu$: we accept a sample $y $ drawn from $\mu$ with probability $\approx \frac{Z_{\tilde{\mu}} (y) }{A\cdot \mu(y)},$ where $A$ is the normalization factor in \cref{eq:new Ising model approximate original Ising model without normalization informal}; implementing this boils down to approximately computing $Z_{\tilde{\mu}} (y)$. The accepted samples are (approximately) distributed according to $Y_*\tilde{\mu}$. Given an accepted sample $y,$ we then sample from the conditional distribution $\tilde{\mu}(\cdot |Y(\sigma) = y).$ 

Implementing rejection sampling and sampling from the conditional distribution are both technically non-trivial. First, in \cref{lem:main construction probability of phase gadget}, we show that the  \emph{conditional marginals} of $\tilde{\mu}(\cdot |Y(\sigma) = y)$ can be approximated within a constant multiplicative factor. Given access to these approximate \emph{conditional marginals}, in \cref{prop:sample from conditional distribution on phase}, we invoke the classical sampling-to-counting reduction of~\citep{JS89} to produce a sampler for $\tilde{\mu}(\cdot |Y(\sigma) = y).$ 

Next, we discuss how to implement rejection sampling, or equivalently, how to approximately compute $ Z_{\tilde{\mu}}(y)$ for a given $y\in \set{\pm 1}^m.$ Since  $ Z_{\tilde{\mu}}(y)$ is the partition function of the distribution $\tilde{\mu}(\cdot |Y(\sigma) = y),$ in \cref{prop:estimate phase probability}, we show how to approximately compute $Z_{\tilde{\mu}}(y)$ using the standard counting-to-sampling reduction of \citep{JS89}\footnote{see \cref{prop:counting-to-sampling} for a detailed statement.}, which involves sampling from the conditional distributions of $\tilde{\mu}(\cdot |Y(\sigma) = y) $ that are obtained by pinning vertices to arbitrary assignments. Similar to what we do in \cref{prop:sample from conditional distribution on phase},  we can sample from the conditional distributions of $\tilde{\mu}(\cdot |Y(\sigma) = y) $ by using the sampling-to-counting reduction of \citep{JS89} together with the approximate conditional marginals from \cref{lem:main construction probability of phase gadget}. 

It remains to bound the number of samples used by $\mathcal{A}$ and to relate the memorization-hallucination behavior of algorithm $\tilde{\mathcal{A}}$ operating on $\tilde{\mu}$ to that of the algorithm $\mathcal{A}$ which operates on $\mu$. This requires a subtle technical argument, which we describe in detail in \cref{subsec:proof of relating hardness of original to near critical Ising model}. The key observation is that each sample from $\mu$ is accepted by the rejection-sampling step with approximately uniform probability, i.e., \[\Pr_{y\sim \mu}[y \text{ is accepted}] \approx p\]
where $p$ is sufficiently large, so for $k \approx \tilde{k}$, the rejection sampling procedure described above will produce at least $\tilde{k}$ samples from $\tilde{\mu}.$
Furthermore, the approximately uniform acceptance probability implies that the accepted subset constitutes an approximately uniform random subsample of the training set. Hence, any memorization-hallucination--avoidance guarantee for $\tilde{\mathcal{A}}$ on the accepted samples lifts to a corresponding guarantee for $\mathcal{A}$ on the entire training corpus.

\section{Organization}
In \Cref{sec:preliminaries}, we will review and introduce the necessary concepts to formally state the results of the paper. In \Cref{sec:mainresults}, we formally state the main result. In \Cref{sec:embed general signature scheme into Ising}, we cover the first part of the proof: how to embed a signature scheme into an Ising model so that the learning-to-sample problem on the resulting Ising model is hard.
In \Cref{sec:embed general Ising into near critical Ising}, we cover the second part of the proof: how to embed a general Ising model into one near the critical threshold, while preserving the difficulty of the learning-to-sample problem. Finally, in \Cref{s:conclusion}, we briefly conclude and suggest some open directions.

\section{Preliminaries} 
\label{sec:preliminaries}
\subsection{Definitions related to distributions} \label{subsec:basic defs}
We first recall the definition of an Ising model.

\begin{definition}[Ising model]\label{def:Ising model}
For a symmetric matrix $J\in \R^{n\times n} $ and vector $h\in \R^n$, let the Ising model parameterized by $J$ and $h$ be the distribution $\mu_{J,h}$ supported on $ \set{\pm 1}^n,$ where
\begin{equation}\label{eq:Ising model}
    \mu_{J,h}(x) \propto \exp\left(\frac{1}{2}\langle x, J x \rangle + \langle h, x \rangle\right)
\end{equation}
We call $J$ the \emph{interaction matrix} and $h$ the \emph{external field} of the Ising model. We note the assumption $J$ is symmetric is without loss of generality, since replacing $J$ with $(J+J^\intercal)/2$ does not change the underlying distribution.

We let $ \norm{J}_{\infty} = \max_{i}\sum_{j} |J_{ij}|$ and $ \norm{h}_{\infty}=\max_i |h_i|.$ Let us also define the \emph{width} of the Ising model (see, e.g. \cite{KM17}) as: \[\operatorname{wd}(J,h) := \max_i (\sum_{j\neq i} |J_{ij}| + |h_i|) \leq \norm{J}_{\infty} + \norm{h}_{\infty}\] 
Since $J$ is symmetric, its eigenvalues are real. Let  $ \lambda_{\max} (J), \lambda_{\min}(J)$ be the maximum and minimum eigenvalues of $J$ and let us define the \emph{spectral width} as $\lambda_{\max} (J)-\lambda_{\min} (J)$.

\end{definition}

We recall several concepts related to distributions---namely pushforwards, tilts, and pinning/conditioning. Given a measurable map \(f : X \to Y\) (deterministic or random), the \emph{push-forward} of a measure \(\mu\) on \(X\) by $f$
is the measure \(f_*\mu\) on \(Y\) satisfying
\[
\int_Y g(y)\, d(f_*\mu)(y)
\;=\;
\int_X g(f(x))\, d\mu(x)
\]
for all measurable \(g : Y \to \mathbb{R}\).

For a measure $ \mu:\R^n \to\R_{\geq 0},$ the \emph{tilt} of $\mu$ by vector $ w\in \R^n$ is the measure $\mathcal T_w \mu: \R^n \to \R_{\geq 0}$ defined by $\mathcal T_w \mu (x) \propto \exp(\langle w, x\rangle) \mu(x) .$

Let $\mu: \{\pm 1\}^n \to \mathbb{R}_{\geq 0}$ be a measure, $S \subseteq [n]$ a subset of indices, and $\tau \in \{\pm 1\}^S$ an assignment to the variables indexed by $S$. The \emph{pinning} of $\mu$ to $\tau$ is the measure $\mu^\tau: \{\pm 1\}^n \to \mathbb{R}_{\geq 0}$ defined as:
\[
\mu^\tau(\sigma) = \mathbf{1}[\sigma_S = \tau] \mu(\sigma),
\]
where $\mathbf{1}[\sigma_S = \tau]$ is the indicator function that equals 1 if the restriction of $\sigma$ to $S$ matches $\tau$, and 0 otherwise.

The \emph{conditional distribution} of $\mu$ by $\tau$ is the distribution $ \mu (\cdot |\sigma_S =\tau) :\{\pm 1\}^{[n]\setminus S} \to \R_{\geq 0}$ defined by $  \mu (\sigma_{[n]\setminus S}|\sigma_S =\tau) = \frac{\mu^\tau(\sigma)}{Z}$ where $Z = \sum_{\sigma'}\mu^\tau(\sigma'). $ 

Throughout, we use the same notation for the unnormalized density function and its corresponding (normalized) distribution. The intended meaning is often clear from context.
\subsection{Learning-to-sample and related concepts}\label{subsec:learning to generalize prelim}

We begin by recalling the notion of ``learning to sample" (see \cite{moitra2021learning,KLV25} for details):
\begin{definition}[Learning to sample]\label{def:learn to sample}
    We say an algorithm $ \mathcal{A}$ $(k,\epsilon_{TV}, \delta)$-learns-to-sample from a distribution $\nu \equiv \nu_{\theta}:\set{\pm 1}^n \to \R_{\geq 0}$ parameterized by $\theta$, if, given the parameter $\theta$ and a tuple $X$ of $k$ i.i.d. samples from $\nu$, with probability $\geq 1-\delta$ over $X$, the output distribution of $\mathcal{A}$ conditioned on its input, denoted by $\mathcal{A}_{\theta,X},$ is $\epsilon_{TV}$-close to $\nu_{\theta}$ in total variation distance:
    \[d_{TV} (\mathcal{A}_{\theta,X}, \nu_{\theta})\leq \epsilon_{TV}.  \]
\end{definition}

Note that the distance between the output of $\mathcal{A}$ and $\nu_{\theta}$ is computed \emph{after} we condition on the input samples. Thus, $\mathcal{A}$ cannot simply parrot the input samples. 

We will in fact use a slightly more convenient concept we call learning-to-generalize, which captures the hallucination-memorization phenomenon:  
\begin{definition}[Learning-to-generalize]\label{def:hard distribution for learning to generalize}
    Consider a distribution $ \nu\equiv \nu_{\theta}: \Omega \to \R_{\geq 0}$ parameterized by $\theta$, a deterministic map $ \Psi$ on $\Omega,$ and a subset of configurations $ \Lambda \subseteq \Omega.$ 
    We say an algorithm $\mathcal{A}$ can $(\eta,k,t,\Psi, \Lambda)$-\emph{learn-to-generalize} $\nu$ if
    $\mathcal{A}$ takes in the parameter $\theta$ of $\nu$, a tuple $X= (x^{(i)})_i$ of $ k$ i.i.d. samples from $ \nu$, runs in $t$ time, and the probability that $ \mathcal{A}$ outputs a sample that is in $\Lambda $ or is in $\Psi^{-1} (\Psi(X))$ is $ \leq 1-\eta.$ 

      Formally, denote the output distribution of $\mathcal{A}$ on input parameter $\theta$ and tuple $X\in \Omega^k$ by $ \mathcal{A}_{\theta, X}$ then:
    \[ \Pr_{X \sim \nu_{\theta}^{\otimes k}, y \sim \mathcal{A}_{\theta, X}} [ y \in \Psi^{-1} (\Psi(X))\text{ or } y \in \Lambda  ] \leq 1-\eta.\]
    
    We say $\nu$ is $(\eta, \zeta, k, t, \Psi)$-\emph{hard for learning-to-generalize} if there exists $\Lambda\subseteq \Omega$ where
    \begin{enumerate}
        \item $\Pr_{y\sim \nu}[y\in \Lambda] \leq \zeta$
        \item No algorithm can $(\eta,k,t,\Psi, \Lambda)$-learn-to-generalize $\nu,$ i.e.,  for every algorithm $\mathcal{A}$ that takes in the parameter $\theta$ of $\nu\equiv \nu_{\theta}$, a tuple $X= (x^{(i)})_i$ of $ k$ i.i.d. samples from $ \nu$,  the output distribution $\mathcal{A}_{\theta, X}$ satisfies:
    \[ \Pr_{X \sim \nu_{\theta}^{\otimes k}, y \sim \mathcal{A}_{\theta,X}} [ y \in \Psi^{-1} (\Psi(X))\text{ or } y \in \Lambda ] > 1-\eta\]

    \end{enumerate}
    
\end{definition}

We make two remarks on this definition: 
\begin{remark}[Relationship with learning-to-sample]\label{remark:relating hard for learning to generalize with learning to sample}
If $ \nu \equiv \nu_\theta$ is $(\eta, \zeta, k, t, \Psi)$-hard for learning-to-generalize, then any algorithm that takes as input $\theta,$ $ k$ samples from $\nu$ and runs in $t$ time cannot sample from a distribution that is close to $ \nu$ in total variation distance---that is, cannot ``learn-to-sample'' in the sense of \cref{def:learn to sample}. 
Indeed, let $X$ be the tuple of $k$ input samples to the algorithm $\mathcal{A},$ and $\mathcal{A}_X \equiv \mathcal{A}_{\theta,X} $ be the output distribution. Then, we have:
\[ \E_{X\sim \nu^{\otimes k}} [d_{TV} (\mathcal{A}_X,\nu)] \geq \E_{X\sim \nu^{\otimes k}}[\mathcal{A}_X (\Psi^{-1}(\Psi(X)) \cup \Lambda  ) -\nu(\Psi^{-1}(\Psi(X)) \cup \Lambda )]\geq  1-\eta -\zeta - k (\Psi_*\nu)_{\max}, \]
where $(\Psi_*\nu)_{\max} = \max_{c} \Psi_*\nu (c)$ is the maximum mass of an element in the push-forward distribution $\Psi_*\nu.$ 

On the other hand, if $\mathcal{A}$ can $(k,\epsilon_{TV},\delta)$-learn-to-sample from $\nu $ then
\[\E_{X\sim \nu^{\otimes k}} [d_{TV} (\mathcal{A}_X,\nu)] \leq \delta + \epsilon_{TV}.\]
Hence, for $ \eta ,\zeta , \delta , \epsilon_{TV}, k (\Psi_*\nu)_{\max}$ sufficiently small so that $ \eta +\zeta +\delta + \epsilon_{TV}+ k (\Psi_*\nu)_{\max} < 1$,  $(\eta, \zeta, k, t, \Psi)$-hardness for learning-to-generalize implies $(k,\epsilon_{TV},\delta)$-hardness for learning-to-sample.
\end{remark}
\begin{remark}[Relationship with sample-amplification]\label{remark:relating hard for learning to generalize with sample amplification}
Suppose that the map $\Psi$ and membership testing for $ \Lambda$ can both be performed efficiently. 
If a distribution $\nu$ is hard for learning-to-generalize, then sample amplification (as defined in \cite{AGSV20}) is impossible as well.

To see why,
consider an algorithm $\mathcal{A}'$ that takes a tuple $X$ of $k $ i.i.d. samples from $ \nu,$ runs in time $t$ and outputs a tuple $Y\in \Omega^{\ell}$ for $\ell>k.$ We can then construct an algorithm
$\mathcal{A}$ that learns to generalize from $\nu$ by returning any element
$y \in Y \setminus \bigl(\Psi^{-1}(\Psi(X)) \cup \Lambda\bigr)$ if such a $y$
exists, otherwise outputs an arbitrary element of $\Omega.$

Let $t_{\text{check}}$ be the time to compute $ Y \setminus \bigl(\Psi^{-1}(\Psi(X)) \cup \Lambda\bigr)$ and $t' $ be the runtime of $\mathcal{A}',$ then the runtime of $\mathcal{A}$ is $ t=t' +t_{\text{check}}.$

Specifically, we say a tuple $X'\in \Omega^\ell$ is \emph{bad} if it exhibits at least one of the
following \emph{bad properties}:
\begin{enumerate}
    \item There exist $x, x' \in X'$ such that $\Psi(x) = \Psi(x')$.
    \item There exists $x \in X'$ such that $x \in \Lambda$.
\end{enumerate}
The probability that $X'\sim \nu^{\otimes \ell}$ does not exhibit the first bad property is at least
\[ \prod_{i=1}^{\ell-1} (1- i\cdot (\Psi_{*}\nu)_{\max})  \geq \exp(-O( \sum_{i=1}^{\ell-1} i\cdot (\Psi_{*}\nu)_{\max})) = \exp(-O(\ell^2(\Psi_{*}\nu)_{\max})) \geq 1 - O(\ell^2(\Psi_{*}\nu)_{\max}). \]
Thus, the probability that $ X'\sim \nu^{\otimes \ell}$ exhibits either bad property is $   O\!\left(\ell^{2}(\Psi_{*}\nu)_{\max} + \ell\,\zeta\right)$.

Now consider the output $Y$ of $\mathcal{A}'$. 
We argue that: if $Y$ is not bad, then
$Y \setminus \bigl(\Psi^{-1}(\Psi(X)) \cup \Lambda\bigr)$ is non-empty and
$\mathcal{A}$ successfully returns $y \in\Omega \setminus \bigl(\Psi^{-1}(\Psi(X)) \cup \Lambda\bigr)$.  

Indeed, suppose $Y$ is not bad and $ Y \setminus \bigl(\Psi^{-1}(\Psi(X)) \cup \Lambda\bigr)$ is empty, then   $\Psi(Y) \subseteq \Psi(X)$, and by the pigeonhole principle, there exists
$y, y' \in Y$ with $\Psi(y) = \Psi(y') = \Psi(x)$ for some
$x \in X$, contradicting the assumption that $Y $ is not bad. 

Suppose $\nu$ is $(\eta, \zeta, k, t, \Psi)$-hard for learning-to-generalize, the output $Y$ of $\mathcal{A}'$ must
be bad with high probability,
i.e.,
\[  \Pr[Y \text{ is not bad}]  \leq \Pr[\mathcal{A}\text{ returns } y \in \Omega\setminus \bigl(\Psi^{-1}(\Psi(X)) \cup \Lambda\bigr)]\leq \eta\]
On the other hand, a fresh
draw $X' \sim \nu^{\otimes\ell}$ is bad with probability at most
$O\!\left(\ell^{2}(\Psi_{*}\nu)_{\max} + \ell\,\zeta\right)$. Hence,
\[ \E_{X\sim \nu^{\otimes k}} [d_{TV} (\mathcal{A}'_X,\nu^{\otimes \ell})] \geq 1 -\eta -O\!\left(\ell^{2}(\Psi_{*}\nu)_{\max} + \ell\,\zeta\right)\]
In other words, the output
of $\mathcal{A}'$ is far from $\nu^{\otimes \ell}$ in total variation distance, ruling
out sample amplification.

\end{remark}
We extend the notion of ``hardness of learning-to-generalize'' for a single distribution in \cref{def:hard distribution for learning to generalize} to hardness for a family of distributions as follows:

\begin{definition}[Learning-to-generalize a family of distributions]\label{def:hard family of distribution for learning to generalize}
    Consider a family of distributions $ (\nu_{\theta})_{\theta \sim \mathcal{P}}$ on $\Omega$, where the parameter $\theta$ is drawn from a distribution $\mathcal{P},$ and a deterministic map $ \Psi$ on $\Omega.$ We say the family of distributions $ (\nu_{\theta})_{\theta \sim \mathcal{P}}$  is $(\eta, \zeta, k, t, \Psi)$-hard for learning-to-generalize if: For each $ \theta \in \supp(\mathcal{P}),$ there exists $ \Lambda_{\theta} \subseteq \Omega$ where 
\begin{enumerate}
    \item $\forall \theta\in \supp(\mathcal{P}): \Pr_{y\sim \nu_{\theta}}[y\in \Lambda_{\theta}] \leq \zeta$
    \item  No algorithm can $(\eta,k,t,\Psi, (\Lambda_\theta)_{\theta \sim \mathcal{P}})$-learn-to-generalize $(\nu_{\theta})_{\theta \sim \mathcal{P}},$ i.e., for any algorithm $\mathcal{A}$ that takes in the parameter $\theta$ of $\nu_{\theta}$, a tuple $X= (x^{(i)})_i$ of $ k$ i.i.d. samples from $ \nu_{\theta}$, and runs in $t$ time, the output distribution $\mathcal{A}_{\theta, X}$ of the algorithm $\mathcal{A}$ on such input satisfies:
    \[ \Pr_{\theta \sim \mathcal{P}, X \sim \nu_{\theta}^{\otimes k}, y \sim \mathcal{A}_{\theta,X}} [ y \in \Psi^{-1} (\Psi(X))\text{ or } y \in \Lambda_{\theta}  ] > 1-\eta\]
\end{enumerate}
\end{definition}
In the next remark, we discuss the relations to \cref{def:hard distribution for learning to generalize,remark:relating hard for learning to generalize with learning to sample,remark:relating hard for learning to generalize with sample amplification}. 
\begin{remark}\label{remark:relating hard for learning to generalize in family of distribution vs for single distribution}
Consider a family of distributions $ (\nu_{\theta})_{\theta \sim \mathcal{P}}$ that is $(\eta, \zeta, k, t, \Psi)$-hard for learning-to-generalize with respect to the family $(\Lambda_{\theta})_{\theta \sim \mathcal{P}}$ of subsets of $\Omega.$
For every fixed efficient learner $\mathcal{A}$, and $\eta_0,\eta_1\in (0,1),\eta=\eta_0\eta_1,$ with probability $ 1-\eta_0$ over the parameter $\theta,$ the algorithm $\mathcal{A}$ cannot $(\eta_1,k,t,\Psi, \Lambda_{\theta})$-learn-to-generalize $\nu_{\theta}.$ 

Next, we discuss the connection with learning-to-sample.
By a similar argument to that in \cref{remark:relating hard for learning to generalize with learning to sample}, we also have that for any efficient learner $ \mathcal{A}$:
\begin{align*}
    \E_{\theta\sim \mathcal{P}, X\sim \nu_\theta^{\otimes k}} [d_{TV} (\mathcal{A}_{\theta,X},\nu_{\theta})] &\geq \E_{\theta\sim \mathcal{P}, X\sim \nu_\theta^{\otimes k}}[\mathcal{A}_{\theta,X} (\Psi^{-1}(\Psi(X)) \cup \Lambda_{\theta}  ) -\nu_\theta(\Psi^{-1}(\Psi(X)) \cup \Lambda_{\theta} )]\\
    &\geq  1-\eta -\zeta - k \Psi_{\max},  
\end{align*}
where $\Psi_{\max} = \sup_{\theta,c} \Psi_*\nu_{\theta} (c).$ 
On the other hand, if $\mathcal{A}$ can $ (k, \epsilon_{TV},\delta)$-learn-to-sample $\nu_{\theta}$ with probability $ 1-\eta'$ over $\theta$ then
\[  \E_{\theta\sim \mathcal{P}, X\sim \nu_\theta^{\otimes k}} [d_{TV} (\mathcal{A}_{\theta,X},\nu_{\theta})] \leq \delta +\epsilon_{TV} + \eta'.\]

Finally, we discuss the connection to sample amplification.
Let $t_{\text{check}}$ be as defined in \cref{remark:relating hard for learning to generalize with sample amplification}. By a similar argument to that in \cref{remark:relating hard for learning to generalize with sample amplification}, we also have that for any sample-amplification algorithm $ \mathcal{A}'$ that takes $\theta$ and $ X\sim \nu_{\theta}^{\otimes k}$ as input, runs in time $t'=t-t_{\text{check}},$ and outputs a tuple $Y\in \Omega^\ell$:
\[  \E_{\theta\sim \mathcal{P}, X\sim \nu_\theta^{\otimes k}} [d_{TV} (\mathcal{A}'_{\theta,X},\nu_\theta^{\otimes \ell})] \geq 1 -\eta -O\!\left(\ell^{2}\Psi_{\max} + \ell\,\zeta\right)\]
\end{remark}
\subsection{Signature schemes and computational Diffie-Hellman} \label{subsec:prelim signature scheme}

Our reduction relies on the security of signature schemes. Recall, a signature scheme is defined as follows: 
\begin{definition}[Signature scheme] A signature scheme consists of three algorithms KeyGen, Sign and Verify: 
    \begin{itemize}
    \item \textbf{$\keygen$}: This algorithm generates a public key $\pk$ and a private key $\sk$. The public key is shared publicly, while the private key is kept secret.
    \item $\textbf{\sign}(\sk, \msg)$: for a given message $\msg,$  this algorithm produces a digital signature $\sigma = \sign(\sk, \msg)$ using the private key $\sk.$
    \item $\textbf{\verify}(\pk, \msg,\sigma)$: This algorithm takes the message $\msg$, the signature $\sigma$, and the public key $\pk$ as input and verifies whether the signature is valid for the given message $\msg$ and public key $\pk.$
\end{itemize}
\end{definition}

The standard notion of security for a signature scheme is \emph{existential unforgeability} under an adaptively chosen message attack \citep{GMR88}.  It is defined in terms of the following game between a challenger and an adversary~$\mathcal{A}$:
\begin{definition}[Adaptively chosen message attack]\label{def:adaptive CMA game}
    The setup is as follows: the challenger runs algorithm \textsf{KeyGen} to obtain a public key $\pk$ and a private key $\sk$. The adversary $\mathcal{A}$ is given $\pk$. 
    
    Proceeding adaptively, $\mathcal{A}$ requests signatures on at most $q$ messages of his choice $\msg^1, \ldots, \msg^{q} $, under $\pk$. The challenger responds to each query with a signature $\sigma_i = \sign (\sk,\msg^i)$. Eventually, $\mathcal{A}$ outputs a pair $(\msg, \sigma)$ and wins the game if $\msg$ is not any of $\msg^1, \ldots, \msg^{q}$, and $\verify(\pk, \msg, \sigma) = 1$.
\label{d:attackgame}    
\end{definition}

We define the advantage of an adversary $\mathcal{A}$ in attacking the signature scheme as
the probability that $\mathcal{A}$ wins the above game, taken over the random bits of the
challenger and the adversary:

\begin{definition}\label{def:EUF-CMA}
A signature scheme is said to be \emph{$(t,q,\varepsilon)$-existentially unforgeable under an adaptive chosen-message attack} (or \emph{$(t,q,\varepsilon)$-EUF-CMA secure}) if no adversary $\mathcal{A}$, running in time at most $t$ and making up to $q$ signature queries, can achieve an advantage of at least $\varepsilon$ in the adaptively chosen message attack game (Definition~\ref{d:attackgame}).
\end{definition}

We will consider signature schemes that satisfy two additional properties: uniformity and regularity.  

\begin{definition}[Uniformity]
For a key pair $ (\pk,\sk)$ generated by $\keygen,$ and a message $\msg,$ let
\[
    \Sigma(\pk,\msg)
    :=
    \{\sigma : \verify(\pk,\msg,\sigma)=1\}
\]
denote the set of signatures accepted by the verification algorithm. 

    We say the signature scheme satisfies \emph{uniformity} if for every key pair $(\pk,\sk)$ and message $\msg$, $\sign(\sk,\msg)$ outputs a uniformly random element of
$\Sigma(\pk,\msg)$.
\end{definition}
\begin{remark}
   Uniformity is closely related to the notion of re-randomizable signatures introduced by \cite[Definition 4.1]{HJK12}. A signature scheme is re-randomizable if there is an efficient algorithm $\mathrm{ReRand}$ such that,
given $\pk,\msg$ and any valid signature $\sigma\in\Sigma(\pk,\msg)$,
$\mathrm{ReRand}(\pk,\msg,\sigma)$
outputs a uniformly random element of $\Sigma(\pk,\msg)$.  For a re-randomizable scheme, the game in \cref{def:adaptive CMA game} is equivalent to the one in which the challenger returns a uniformly random element of $\Sigma(\pk,\msg),$ since the adversary can replace the challenger's response $\sigma= \sign(\sk,\msg)$ with the output of
$\mathrm{ReRand}(\pk,\msg,\sigma).$
\end{remark}
\begin{definition}[Regularity]
    We say the signature scheme satisfies \emph{exact regularity} if the following holds, for each public key $\pk$: there exists $ s$ such that for all message  $\msg$, $|\Sigma(\pk,\msg)|=s.$

\end{definition}

The specific signature scheme(s) discussed in this paper are constructed using \emph{bilinear maps}.  Below, we revisit the definition of a bilinear map.

\begin{definition}[Bilinear map]
    Let \(\mathbb{G}_1\) and \(G_2\) be two (multiplicative) cyclic groups of prime order \(p\).  
Let \(g_1\) be a generator of \(\mathbb{G}_1\) and \(g_2\) a generator of \(\mathbb{G}_2\). A \emph{bilinear map} is a function $e : \mathbb{G}_1 \times \mathbb{G}_2 \rightarrow \mathbb{G}_T$, 
where \(\mathbb{G}_T\) is a cyclic group of the same order \(p\), and \(e\) satisfies:

\begin{enumerate}
    \item \textbf{Bilinearity}: for all \(u \in \mathbb{G}_1\), \(v \in \mathbb{G}_2\) and all \(a,b \in \mathbb{Z}\), we have $e(u^a, v^b) = e(u, v)^{ab}.$ 
\item \textbf{Non-degeneracy}: $e(g_1, g_2) \neq 1.$ 
\item
 \textbf{Efficient computability}: the map $e$ and the group action in $\mathbb{G}_1, \mathbb{G}_2, \mathbb{G}_T$ can be computed efficiently.
\end{enumerate}
\end{definition}
\begin{definition}\label{def:bilinear friendly group}
    We say a group $ \mathbb{G}$ of prime order $p$ is \emph{bilinear-pairing-friendly} if it supports a bilinear map $ e: \mathbb{G} \times \mathbb{G} \to \mathbb{G}_T$ where $| \mathbb{G}_T| = p.$
\end{definition}

The computational hardness assumption underlying the signature schemes used in this work is the \emph{computational Diffie-Hellman problem}. 

\begin{definition}[Computational Diffie-Hellman problem] \label{def:computational diffie hellman assumption} Let $\mathbb{G}$ be a group of order \(p\). We say that algorithm \(\mathcal{A}\) \((\varepsilon, t)\)-solves the computational
Diffie-Hellman problem in \(\mathbb{G}\) if
\[
\Pr\!\left[\, \mathcal{A}(g, g^\alpha, g^\beta) = g^{\alpha \beta} \,\right] \ge \varepsilon,
\]
where the group elements \((g, g^\alpha, g^\beta) \in \mathbb{G}^3\) are randomly chosen, and \(\mathcal{A}\) runs in time \(t\). We say the $(\varepsilon,t)$-computational Diffie-Hellman ($(\varepsilon,t)$-CDH) assumption holds for $\mathbb{G} $ if no algorithm can $ (\varepsilon,t)$-solve the computational Diffie-Hellman problem in $\mathbb{G}.$
\end{definition}
We make several remarks about this assumption:
\begin{remark}\label{remark:advantage parameter in the CDH assumption}
    The algorithm that outputs a random element of $ \mathbb{G}$ can guess $g^{\alpha\beta}$ with probability $ 1/p,$ hence we can assume without loss of generality that in the $(\varepsilon,t)$-CDH assumption, $\varepsilon \geq 1/p.$ 
\end{remark}

\begin{remark}\label{remark:parameters for CDH}
The Computational Diffie-Hellman (CDH) assumption for bilinear-pairing-friendly groups is a standard cryptographic assumption that underlies the security of pairing-based cryptographic schemes (e.g. \cite{BF01,BLS01,BB04,water05,BSW06}).  In the generic group model, it has been shown that the $ (\varepsilon,t)$-CDH assumption holds in a group of prime order $p$ when $\varepsilon = \Omega(t^2/p)$ \cite{shoup1997lower}. 
\end{remark}

\begin{remark}
The well-known \cite{water05}'s signature scheme (see \cref{sec:signature}) satisfies uniformity and exact regularity (see \cref{lem:water scheme}), and is secure under the computational Diffie-Hellman assumption on bilinear-pairing-friendly groups (see \cref{thm:water scheme security}). Our main result is based on this scheme and thus relies on the hardness of the computational Diffie-Hellman assumption on bilinear-pairing-friendly groups.
\end{remark}

\section{Main result: formal statement} 
\label{sec:mainresults}

With these definitions in hand, we are finally ready to state the main result. We note this result is in fact a corollary of a more general result (see \cref{thm:main hardness of near critical Ising based on signature scheme} in \cref{sec:embed general Ising into near critical Ising}).  

\begin{theorem}[Main result] \label{cor:main hardness of near critical Ising model from computational diffie hellman}

Consider a bilinear-pairing-friendly group $ \mathbb{G}$ of prime order $p,$ for which the $(\varepsilon_{\CDH}, t_{\CDH
})$-computational Diffie-Hellman (CDH) assumption holds with $\varepsilon_{\CDH}\geq 1/p$\footnote{The assumption $\varepsilon_{\CDH}\geq 1/p$ is without loss of generality, see \cref{remark:advantage parameter in the CDH assumption}}. Let $ d= \lceil\log p\rceil.$ Fix $\gamma >1.$
There exists a family $(\tilde{\mu}^{\pk})_{\pk\sim \keygen}$ of Ising models on $N = \mbox{poly}(d)$ vertices along with a map $\tilde{\Psi}_{\msg}:\set{\pm 1}^N\to \set{0,1}^d $ that is computable in $O(N)$ time, so that
\begin{itemize}
    \item For every distribution $\tilde{\mu}^{\pk}$ in this family, its parameter $ (\tilde{J}^{\pk}, \tilde{h}^{\pk})$ satisfies 
     \[1< \lambda_{\max}(\tilde{J}^{\pk}) -\lambda_{\min}(\tilde{J}^{\pk}) \leq \gamma\quad \text{and} \quad \norm{\tilde{J}^{\pk}}_\infty  + \norm{\tilde{h}^{\pk}}_\infty = O(1).\]
     \item The family is $ (\tilde{\eta}, \tilde{\zeta}, \tilde{k}, \tilde{t}, \tilde{\Psi}_{\msg})$-hard for learning-to-generalize\footnote{This implies that distributions from this family are hard for learning-to-generalize and learning-to-sample on average; see \cref{def:hard family of distribution for learning to generalize} and \cref{remark:relating hard for learning to generalize in family of distribution vs for single distribution}.\label{footnotemark:about hard for family vs single distribution}} where  
\begin{align*}
    \tilde{\eta} &=C \varepsilon_{\CDH} \tilde{k} \sqrt{d}  = C \varepsilon_{\CDH} \tilde{k} N^{\Theta(1)} \leq 1\\
    \tilde{\zeta} &= \varepsilon_{\CDH} \exp(-d) =  \varepsilon_{\CDH} \exp(-N^{\Theta(1)})\\
    \tilde{t} &= O(t_{\CDH})\\
    \tilde{k} &\leq \min \left\{ c\, 2^d, \left(\frac{t_{\CDH}}{C_{\gamma} d^{\Theta(1)}}\right)^{1/3}, \frac{1}{C  \varepsilon_{\CDH}  \sqrt{d} } \right \} =\min  \left\{c\exp(N^{\Theta(1)}), \left(\frac{t_{\CDH}}{C_{\gamma} N^{\Theta(1)}}\right)^{1/3}, \frac{1}{C  \varepsilon_{\CDH}  N^{\Theta(1)} }  \right\}.
\end{align*}
\end{itemize}
     
where $C_{\gamma}>0$ is a constant dependent only on $ \gamma,$ and $c,C>0$ are absolute constants.
\end{theorem}

We make a remark on the plausible range of the computational Diffie-Hellman parameters and their implications for our result: 
\begin{remark}\label{remark:concrete hardness parameters based on CDH}
We discuss the parameters $\varepsilon_{\CDH}, t_{\CDH}$ and their impact.
It is widely believed that there exists a constant $c_0>0$ such that the $(\varepsilon_{\CDH}, t_{\CDH})$-CDH assumption holds with $\varepsilon_{\CDH} = \exp(-d^{c_0})$ and $t_{\CDH} = d^C,$ for any constant $C>0.$ Under this assumption, the family $(\tilde{\mu}^{\pk})_{\pk\sim \keygen}$ is $ (\tilde{\eta}, \tilde{\zeta}, \tilde{k}, \tilde{t}, \tilde{\Psi}_{\msg})$-hard for learning-to-generalize with negligibly small $\tilde{\eta}, \tilde{\zeta} $ and polynomial $ \tilde{k}, \tilde{t}.$ Formally, we have:  \[\max\set{\tilde{\eta}, \tilde{\zeta}}= \exp(-N^{\Theta(1)}) ,\quad \max\set{\tilde{k},\tilde{t}} = \poly(N).\] 
Under the stronger but still plausible assumption that the $(\varepsilon_{\CDH}, t_{\CDH})$-CDH
  assumption holds against subexponential-time adversaries\footnote{This assumption holds in the generic group model, which is the primary framework for justifying the hardness of group-based cryptographic assumptions; see \cref{remark:parameters for CDH} for details}, the family $(\tilde{\mu}^{\pk})_{\pk\sim \keygen}$ is $ (\tilde{\eta}, \tilde{\zeta}, \tilde{k}, \tilde{t}, \tilde{\Psi}_{\msg})$-hard for learning-to-generalize with negligibly small $\tilde{\eta}, \tilde{\zeta} $ and sub-exponential $ \tilde{k}, \tilde{t}.$
Specifically,
  suppose the $(\varepsilon_{\CDH}, t_{\CDH})$-CDH
  assumption holds with $ \varepsilon_{\CDH} = \exp(-d^{c_0})$ and $ t_{\CDH} = \exp(d^{c_1})$ for some constants $c_0,c_1 >0,$ then $(\tilde{\mu}^{\pk})_{\pk\sim \keygen}$ is $ (\tilde{\eta}, \tilde{\zeta}, \tilde{k}, \tilde{t}, \tilde{\Psi}_{\msg})$-hard for learning-to-generalize with:
  \[\max\set{\tilde{\eta}, \tilde{\zeta}}= \exp(-N^{\Theta(1)}) ,\quad \max\set{\tilde{k},\tilde{t}} = \exp(N^{\Theta(1)}).\] 
\end{remark}
We defer further discussions to \cref{sec:puttingtogether}.

\section{Embedding a signature scheme into an Ising model}\label{sec:embed general signature scheme into Ising}

In this section, we show how to embed a signature verification task into an Ising model. Specifically, for a fixed public key $\pk$, we construct an Ising model such that its high-probability configurations encode valid message--signature pairs under $\pk$. Consequently, any learner that, given the model parameters and polynomially many training samples, could generate fresh typical samples from this model would also produce a fresh valid signature. This converts existential unforgeability into hardness of learning to generalize for a (dense) family of Ising models. 

The section is organized as follows. We first recall and adapt the machinery of embedding a circuit into an Ising model from \citep{moitra2021learning} in \Cref{subsec:generic embedding of circuit into Ising}. We then specialize it to the verification circuit and pin the public-key and acceptance bits, obtaining for each public key $\pk$ an Ising model $\mu^{\pk}$ that is exponentially close to the uniform distribution over valid encodings under $\pk$ in \Cref{subsec:embed signature scheme into a general Ising model}. 

\subsection{Embedding a circuit into an Ising model}\label{subsec:generic embedding of circuit into Ising}
Consider a Boolean-valued function $f:\set{0,1}^n\to \set{0,1}$ which can be implemented by a circuit $\mathcal{C}$ consisting of $r$ NAND gates, where each NAND gate has two input bits and one output bit. We concatenate the input together with the $r$ bits computed by the NAND gates, where the last bit corresponds to the output of $f$, to obtain the function $\Psi_f: \set{0,1}^n \to \set{0,1}^{n+r} .$ 
=

Let $ \Psi_{0} : \set{0,1}^* \to\set{\pm 1}^*$ be the map defined on each bit $b\in \set{0,1}$ by $\Psi_0(b) = 2b -1$, and let $\Phi_f = \Psi_0 \circ \Psi_f: \set{0,1}^n \to \set{\pm 1}^{n+r}.$ We say $\Phi_f$ is the \emph{representation} of the function $f.$

Following \cite{moitra2021learning}, we construct an Ising model that represents the function $f.$
\begin{lemma}\label{lem:embed general circuit into Ising}
Let $ m:= n+r.$
       For $w =\Omega(m),$  there exists an explicitly computable Ising model $\mu: \set{\pm 1}^{m}\to \R_{\geq 0 }$ parameterized by $(J,h)$
       so that:
       \begin{enumerate}
           \item  The interaction matrix $J =(J_{ij})_{i,j}$ and external field $h = (h_i)_i$ of this Ising model  can be computed in $O( m^2)$  time, and satisfy \[ \max\set{|J_{ij}|, |h_i|} \leq O( m w).\] 
           \item Let $\nu:\set{\pm 1}^m \to  \R_{\geq 0 }$ be the push-forward of the uniform distribution on $ \set{0,1}^n$ by the map $ \Phi_f$ then 
       \[d_{TV}(\nu, \mu) \leq \exp(-w).\]
       Moreover, for each $S \subseteq [m],$ and $\tau\in \set{\pm 1}^S$ where $\Pr_{x\sim \nu}[x_S = \tau]>0$, define $ \mathbf{\lambda}^{\tau}\in \R^m $ by \[\mathbf{\lambda}^{\tau}_i  =\begin{cases} 0 &\text{ if } i\not\in S \\
       w\tau_i &\text{ if } i\in S \end{cases} \] then the tilt of $\mu$ by $ \mathbf{\lambda}^{\tau}$ and the pinning of $\nu$ by $\tau$ are close in total variation distance, i.e.,
\[d_{TV} (\mathcal{T}_{\mathbf{\lambda}^{\tau}} \mu, \nu^\tau) \leq \exp(-w)\]
Note that $\mathcal{T}_{\mathbf{\lambda}^{\tau}} \mu$ is the Ising model parameterized by $ (J, h+ \mathbf{\lambda}^{\tau}).$
       \end{enumerate}
\end{lemma}
The above lemma follows from the same argument as in the proof of
\cite[Lemma 43]{moitra2021learning}. We defer the proof to \cref{sec:defer proofs from Section embedding signature into Ising}.

\subsection{Embed signature scheme into Ising model} \label{subsec:embed signature scheme into a general Ising model}

Fix a signature scheme $(\keygen,\sign,\verify)$ with public-key length $\ell_{\pk}$, message length $\ell_{\msg}$, and signature length $\ell_{\sig}$. Assume $\verify$ is deterministic and computable by a Boolean circuit with $r$ NAND gates.We write
\begin{equation}\label{eq:signature scheme param}
 n:=\ell_{\pk}+\ell_{\msg}+\ell_{\sig}
\qquad\text{and}\qquad
m:=n+r.   
\end{equation}

We view $(\pk,\msg,\sig)$ as an input in $\{0,1\}^n$, the verification algorithm as a deterministic map $\verify:\{0,1\}^n\to\{0,1\}$, and let $ \Phi_{\verify}: \set{0,1}^{n} \to \set{\pm 1}^{m}$ be the representation of $\verify$ obtained from \cref{subsec:generic embedding of circuit into Ising}.

The next lemma is a direct corollary of  \cref{lem:embed general circuit into Ising}. It realizes the distribution of valid message--signature pairs, for a fixed public key, as an explicit Ising model up to exponentially small total variation error:

\begin{lemma} \label{lem:dense-valid-pairs-fixed-pk}

For a public key $\pk$, let $\nu^{\pk}$ denote the uniform distribution over
\[
\Lambda^{\pk,\valid}
=
\left\{
\Phi_{\verify}(\pk,\msg,\sig):
\verify(\pk,\msg,\sig)=1
\right\}.
\]
For every $w=\Omega(m)$, there exists an explicit Ising model
$\mu^{\pk}$ on $\{\pm1\}^m$ parameterized by $(J,h^{\pk}),$ where $(J,h^{\pk})$ can be computed in $O(m^2)$ time,  satisfying
\[
    d_{\mathrm{TV}}(\mu^{\pk},\nu^{\pk})\le \exp(-w).
\]
Here $J$ is independent of $\pk$, $h^{\pk}$ depends on $\pk$, and
$   \max\set{|J_{ij}|, |h^{\pk}_i|} \leq O( m w). $
\end{lemma}
\begin{proof}
    Let $ \nu $ be the push-forward distribution of $U(\set{0,1}^n)$ by $ \Phi_{\verify}.$ Let $S= [\ell_{\pk}]\cup \set{m}$ be the vertices corresponding to the public key and output bit of the verification circuit. Clearly, for each public key $\pk,$ $\nu^{\pk}$ is obtained by pinning $ \nu$ according to $\pk$ and pinning the output bit to $1.$  Following \cref{lem:embed general circuit into Ising}, let $\mu$ be the Ising model corresponding to $\nu$, and $\mu^{\pk}$ be the tilt of $\mu$ that corresponds to the pinning $\nu^{\pk}$ of $\nu$, then the claim follows.
\end{proof}

\begin{theorem}\label{lem:Ising model is hard to generalize if signature scheme is secure} 
    Suppose the signature scheme is $(t',q', \varepsilon')$-secure, has deterministic verification, and satisfies uniformity and exact regularity.
Recall the map $ \Psi_{0} : \set{0,1}^* \to\set{\pm 1}^*$ from \cref{subsec:generic embedding of circuit into Ising} i.e. the map defined on each bit $b\in \set{0,1}$ by $\Psi_0(b) = 2b -1.$
Let $\Psi_{\msg} : \set{\pm 1}^m\to \set{0,1}^{\ell_{\msg}}, 
\Psi_{\msg}(y) = \Psi_0^{-1}(y_{\ell_\pk+1:\ell_{\pk} +\ell_{\msg}})$ be the map that views $y \in \set{\pm 1}^m$  as a tuple of (public key, message, signature), and maps it to the corresponding message in $ \set{0,1}^{\ell_{\msg}}.$

    Suppose $w=\Omega(m),$
    then the family of Ising models $(\mu^{\pk})_{\pk \sim \keygen}$ from \cref{lem:dense-valid-pairs-fixed-pk} is $(\eta,\exp(-w), k,t, \Psi_{\msg} )$-hard for learning-to-generalize, where
    \[\eta = 2\varepsilon', \quad \min\set{q', \varepsilon' \exp(w) }\geq k, \quad t  = t' - C (mk + m^2),\]
    for some absolute constant $C>0.$
\end{theorem}

\begin{proof}
We show that an algorithm $\mathcal{A}$ that learns-to-generalize from the family $(\mu^{\pk})_{\pk \sim \keygen}$ in the sense of \cref{def:hard family of distribution for learning to generalize} can be used to construct an adversary $\mathcal{B}$ that wins the Existential Unforgeability Challenge game.
\begin{itemize}[leftmargin=*]
    \item \textbf{Setup:} The challenger runs algorithm $\keygen$ to obtain a public key $\pk$ and a private key $\sk$. The adversary $\mathcal{B}$ is given $\pk$. 
    \item \textbf{Queries:} The adversary issues random message queries $\msg^1, \cdots, \msg^k \sim \text{Uniform}(\mathcal{M})$ where $\mathcal{M}$ denotes the domain of messages. The challenger responds with $\sig^i =\sign(\sk, \msg^i).$ Note that by the uniformity and exact regularity conditions, $(\msg^i, \sig^i)$ is distributed according to $\nu^{\pk}$ defined in \cref{lem:dense-valid-pairs-fixed-pk} i.e. the uniform distribution over valid message-signature pairs under the public key $\pk$.

\item \textbf{Output:} Let $x^{(i)} = \Phi_{\verify}(\pk, \msg^i, \sig^i)$ for $i\in[k].$ The adversary $\mathcal{B}$ computes the Ising model parameter $(J,h^{\pk})$ from the public key $\pk.$
Next, the adversary calls the algorithm $\mathcal{A}$ with Ising model parameter $(J,h^{\pk})$ and training sample tuple $X = (x^{(i)})_{i\in [k]}$  to obtain output $ y.$

The adversary $\mathcal{B}$ outputs the message-signature pair that corresponds to $y$ i.e. $(\Psi_{\msg}(y), \Psi_{\sig}(y))$ where $ \Psi_{\sig}(y)= \Psi_0^{-1}(y_{\ell_\pk+\ell_{\msg}+1:n}).$ 
\end{itemize}
For a public key $\pk,$
let $\Lambda^{\pk,\text{invalid} }= \set{\pm 1}^{m}\setminus \Lambda^{\pk,\valid}.$ 
By Lemma \ref{lem:dense-valid-pairs-fixed-pk} and the definition of $ \nu^{\pk}$, 
\[  \mu^{\pk}(\Lambda^{\pk,\text{invalid}}) \leq \nu^{\pk}(\Lambda^{\pk,\text{invalid}}) +\exp(-w) = \exp(-w). \]

Let $X'=(x^{'(i)})_{i\in [k]}$ be a tuple of $k$ i.i.d. samples from $\mu^{\pk} .$

 Suppose for contradiction that $\mathcal{A}$ can $(\eta, k, t, \Psi_{\msg}, (\Lambda^{\pk,\text{invalid} })_{\pk \sim \keygen})$-learn-to-generalize $(\mu^{\pk})_{\pk \sim \keygen}.$ Then, $\mathcal{A}$ 
 runs in time $t$ and
 \[ \Pr_{\pk \sim \keygen, X' \sim (\mu_{\pk})^{\otimes k}, y \sim \mathcal{A}_{(J,h^{\pk}) ,X'}} [ y \in \Psi_{\msg}^{-1} (\Psi_{\msg}(X'))\text{ or } y \in \Lambda^{\pk, \text{invalid}}  ] \leq 1-\eta\]

By Lemma \ref{lem:dense-valid-pairs-fixed-pk} and union bound, we can couple $X$ and $X'$ together so that the coupling succeeds (i.e.  $X = X' $)  with probability $ \geq 1 - \exp(-w) k.$
Suppose the coupling succeeds. The above gives that, with probability $\geq \eta,$ the output $ y$ of $\mathcal{A}_{X'} = \mathcal{A}_{X}$ is in $\Lambda^{\pk,\valid}$ and satisfies $ \Psi_{\msg}(y) \not\in \Psi_{\msg}(X) = (\msg^i)_{i=1}^k,$ which means $ (\Psi_{\msg}(y), \Psi_{\sig}(y))$ is a valid message-signature pair under the public key $\pk$ where the message $\Psi_{\msg}(y)$ hasn't been queried. Hence, $\mathcal{B}$ wins the game with probability $\geq \eta - \exp(-w) k \geq \varepsilon',$ assuming $ \eta = 2\varepsilon'$ and $ k\leq \exp(w) \varepsilon'.$ 

Clearly, $\mathcal{B}$ uses $k\leq q'$ message queries.
Next, we bound the runtime $t'$ of $\mathcal{B}.$ 
$\mathcal{B}$ takes $ O(m^2)$ time to construct $(J, h^{\pk}),$ $O(m)$ time to construct each sample $ x^{(i)},$ and $t$ time to run $\mathcal{A}.$ Hence,
\[t'= O(mk + m^2) +t. \]
Taking the contrapositive finishes the proof.
\end{proof}

\section{From arbitrary Ising models to near-critical Ising models}
\label{sec:embed general Ising into near critical Ising}

This section transforms the (potentially dense) Ising model construction from \Cref{sec:embed general signature scheme into Ising} to an Ising model in the near-critical regime of the main theorem. Namely, starting from an arbitrary Ising model $\mu=\mu_{J,h}$, we build a new Ising model $\tilde{\mu}$ whose interaction matrix lies arbitrarily close to the tractable threshold
\[
\lambda_{\max}(\tilde{J})-\lambda_{\min}(\tilde{J})\le \gamma,
\qquad \gamma>1,
\]
and with width bounded by $O(1)$.

The construction involves a kind of phase-gadget simulation. Specifically, each vertex of the original model is replaced by a gadget with two dominant phases, and the majority phase of that gadget plays the role of the original spin. The main point is not only that the phase vector of a sample from $\tilde{\mu}$ approximately follows the original distribution $\mu$, but also that this approximation is strong enough to lift the hardness for learning-to-generalize from $\mu$ to $\tilde{\mu}$.

The section is organized as follows. We first describe the gadget construction to transform a generic Ising model into one in the near-critical regime in~\Cref{subsec:sparsifying Ising model construction}. 
In \Cref{sec:reduction preserves properties}, we state the main result of the section --- namely, that the Ising instance after applying the transformation preserves the relevant properties related to the hardness of learning-to-generalize; recall that hardness for learning-to-generalize implies hardness for learning-to-sample (see \cref{remark:relating hard for learning to generalize with learning to sample,remark:relating hard for learning to generalize in family of distribution vs for single distribution}). In the subsequent sections we prove this result. Precisely, in \Cref{sec:nearcritical} we show that the reduction produces an Ising instance with interactions in the near-critical regime. In \Cref{subsec:proof of relating hardness of original to near critical Ising model}, we show that the reduction preserves the hardness of learning-to-generalize. In \cref{sec:puttingtogether}, we put all the components together to prove \Cref{cor:main hardness of near critical Ising model from computational diffie hellman}.

\subsection{Phase-gadget construction}\label{subsec:sparsifying Ising model construction}

Let $H=(V_H,E_H)$ be the weighted graph of the original Ising instance, with $|V_H|=m$. 
The construction has a simple interpretation. For each original vertex $v$, we build a gadget $G_v$ with two stable phases. A subset of spins in the gadget will encode the external field at $v$, while other spins are reserved for simulating the couplings between $v$ and its neighbors. The majority phase of the gadget will later be read as the effective spin of $v$.

To formalize the construction, first we replace each node \(v \in V_H\) with a gadget graph $G_v$ on a set of vertices $W_v$. For each \(v \in V_H\), consider disjointed subsets  \(S^{\text{edge}}_v , S^{\text{vertex}}_v\) of $W_v$. Let $ S_v : =  S^{\text{edge}}_v\cup S^{\text{vertex}}_v $ and $R_v = W_v\setminus S_v.$ Let  $|S^{\text{vertex}}_v| =t^0_v$, $|W_v| =n_v$, $|R_v| = r_v$ and  $|S_v| =t_v = n_v-r_v.$ We enforce that $r_v$ is odd for each $v.$
For vertices $i,j\in W_v$ that are not both in $S_v$, we add an edge with weight $ \tilde{J}_{ij} =\frac{\beta}{r_v}$ where $\beta > 1;$ these edges include a self-loop of weight $ \frac{\beta}{r_v}$ for each $ i\in R_v.$
For each vertex $i\in S^{\text{vertex}}_v,$
  add an external field of value \[\tilde{h}_i =\tilde{h}_v := \varphi_{\mathrm{vertex},\beta}^{-1}\left(\frac{h_v}{t_v^0}\right),\]
  where $\varphi_{\mathrm{vertex},\beta}^{-1}$ is as in \cref{eq:inverse of edge and weight map function}.
  
  Let $\hat{\mathcal{H}}$ be the disjoint union of 
  the \(G_v\)'s, one for each \(v \in H\).  
  Let $\mu_{\hat{\mathcal{H}}}$ be the corresponding Ising model, and $\hat{J} $ be its interaction matrix. We now describe how to encode the edges of \(H\) using connections between the gadgets, and we let $\mathcal{H}$ be the graph obtained by adding these edges to $\hat{\mathcal{H}}$. 
  
  For each neighbor $v$ of $u,$ i.e., $ v\neq u$ where $J_{uv} \neq 0,$ let $S^v_u$ be a subset of $S^{\text{edge}}_u$ of size $e_{uv},$ so that the subsets $S^v_u$ partition $ S^{\text{edge}}_u.$  We enforce that $ e_{uv} = e_{vu}.$ For each edge $\set{u,v}\in E_H$, we add a perfect matching 
  between \(S^{v}_u\) and  \(S^u_{v}\). Let $E(S^{v}_u, S^{u}_v)$ be the set of edges in the perfect matching. For each edge $ (i,j) \in E(S^{v}_u, S^{u}_v),$ add an edge with weight \[\tilde{J}_{ij} =\tilde{J}_{uv}: = \varphi_{\mathrm{edge},\beta}^{-1} \left(\frac{J_{uv}}{e_{uv}}\right)  ,\]
    where $\varphi_{\mathrm{edge},\beta}^{-1}$ is as in \cref{eq:inverse of edge and weight map function}.
    
Let $\tilde{J}, \tilde{h}, \tilde{\mu} \equiv \mu_{\mathcal{H}} $ be the interaction matrix, external field, and Ising model corresponding to $\mathcal{H}$, respectively; note that $\tilde{h}$ is also the external field of $\hat{\mathcal{H}}.$

At a high level, the roles of the gadget components are as follows. The vertices in $S^{\text{vertex}}_v$ encode the original external field $h_v$. The blocks $S_u^v, S_v^u$ and the matchings between them encode the original edge weights $J_{uv}$. The phase readout map $Y$ (which we will subsequently define) will recover the effective spin configuration from these gadgets.

\subsection{Auxiliary notation and definitions}\label{subsec:notation for Ising embedding}
The next block of notation introduces quantities required to make precise properties of the gadget construction in Section~\ref{subsec:sparsifying Ising model construction}. 

For a set of vertices $V,$ a configuration $ \sigma \in \set{\pm 1}^{V}$ and a subset $S\subseteq V,$ let $ m_S(\sigma) =\sum_{i\in S} \sigma_i$ and $ y_S(\sigma) = \sign(m_S(\sigma)).$  

Let $\mathcal{V} = \bigcup_{v\in V_H} W_v$, $ \mathcal{S} = \bigcup_{v\in V_H} S_v,$ and $\mathcal{R} = \mathcal{V}\setminus \mathcal{S} = \bigcup_{v\in V_H} R_v.$ 
For a configuration $\sigma $ on $\mathcal{V},$ let 
\begin{equation}
\label{def:phase gadget map Y}    
Y(\sigma) = (y_{R_v}(\sigma))_{v\in V_H} \in \set{ \pm 1}^m
\end{equation}  be the vector 
of phases of all the gadgets for the configuration $\sigma.$ For $Y \in \set{\pm 1}^{V_H},$ let
\begin{equation}\label{eq:define phase total function for hat H}   
Z_{\hat{\mathcal{H}}}(Y) :=\sum_{\sigma\in\{-1,+1\}^{\mathcal{V}}} \mu_{\hat{\mathcal{H}}} (\sigma) \mathbf{1}\{Y(\sigma)=Y\}
\end{equation}

\begin{equation}\label{eq:define phase total function for  H}   
Z_{\mathcal{H}}(Y) :=\sum_{\sigma\in\{-1,+1\}^{\mathcal{V}}} \mu_{\mathcal{H}} (\sigma) \mathbf{1}\{Y(\sigma)=Y\}
\end{equation}
where \[ \mu_{\mathcal{H}} (\sigma) = \exp(\frac{1}{2}\sigma^\intercal \tilde{J} \sigma + \sigma^{\intercal} \tilde{h} ) \quad \text{and} \quad \mu_{\hat{\mathcal{H}}} (\sigma) =\exp(\frac{1}{2}\sigma^\intercal \hat{J} \sigma + \sigma^{\intercal} \tilde{h} ). \] For $J,h$ be the interaction matrix and external field that corresponds to the graph $H$, we also let \[ \mu_H (Y) =\exp (\frac{1}{2} Y^\intercal J Y + Y^\intercal h). \]

Relatedly, for any $\mathcal{T}\subseteq \mathcal{V}$ and a configuration $\tilde{\sigma}_{\mathcal{T}} \in \set{\pm 1}^{\mathcal{T}},$ we denote 
\begin{equation}\label{eq:partial configuration condition on phase gadget}
\begin{split}
\mu_{\hat{\mathcal{H}}}(\tilde{\sigma}_{\mathcal{T}}; Y ) &=\sum_{\sigma\in\{-1,+1\}^{\mathcal{V}}} \mu_{\hat{\mathcal{H}}} (\sigma) \mathbf{1}\{\sigma_{\mathcal{T}}= \tilde{\sigma}_{\mathcal{T}} \text{ and } Y(\sigma)=Y \} \\
   \mu_{\mathcal{H}}(\tilde{\sigma}_{\mathcal{T}}; Y ) &=\sum_{\sigma\in\{-1,+1\}^{\mathcal{V}}} \mu_{\mathcal{H}} (\sigma) \mathbf{1}\{\sigma_{\mathcal{T}}= \tilde{\sigma}_{\mathcal{T}} \text{ and } Y(\sigma)=Y \}  
\end{split}
\end{equation}

Finally, we introduce some technical notation involving trigonometric functions, which are useful for the analysis of our gadget construction. For $h \in \R,$ recall that \[\cosh(h) = \frac{\exp(h) + \exp(-h)}{2}, \tanh (h) = \frac{\exp(h)-\exp(-h)}{\exp(h) +\exp(-h)}\in [-1,1].\] We note that $\tanh$ is a strictly increasing function on $\R,$ is invertible with $\tanh^{-1}(a) =  \frac{1}{2}\log \left(\frac{1+a}{1-a}\right),$ and $\tanh^{-1}:(-1,1)\to \R$ is also a strictly increasing function. In addition, $\tanh$ is $1$-Lipschitz, since $|\tanh'(h)| \leq 1\forall h.$

Let us denote
    \begin{equation}
    \label{eq:binary entropy function H}
      H(\alpha) = -\alpha \ln \alpha - (1-\alpha)\ln(1-\alpha)  
    \end{equation}
   and for $\beta > 1$, let
    \begin{equation}\label{eq:magnetization function f}
      f_\beta(\alpha) := H(\alpha) + \frac{\beta}{2}(2\alpha - 1)^2.  
    \end{equation}

    Note that derivatives of $f_{\beta}$ satisfy
    \begin{align}
        f_\beta'(\alpha) &= \ln\!\left(\frac{1-\alpha}{\alpha}\right) + 2\beta(2\alpha - 1),  \label{eq:derivative of f} \\
        f_\beta''(\alpha) &= -\frac{1}{\alpha(1-
        \alpha)} + 4\beta \label{eq:second derivative of f}
    \end{align}

    For $\beta >1,$ the function $f_\beta(\alpha)$ has a minimum at $ \alpha = 1/2,$ and two maxima at $q^+_\beta \in (1/2,1), q^-_\beta\in (0,1/2)$ which satisfy:
    \begin{equation}\label{eq:def of q+ and q-}
        f_\beta'(\alpha) = \ln\!\left(\frac{1-\alpha}{\alpha}\right) + 2\beta(2\alpha - 1) = 0
    \end{equation}

Note that $q^+_\beta + q^-_\beta =1.$ For proof of these facts, see \cref{prop:property of q^+}. We omit the subscript $\beta$ when it is clear from context.

Finally, we define the functions $\varphi_{\mathrm{vertex},\beta}, \varphi_{\mathrm{edge},\beta}$ on $\R$ by 
  \begin{equation}\label{eq:define weight map function}
  \begin{split}
      \varphi_{\mathrm{vertex},\beta}(h) &= \frac{1}{2}\ln \left(\frac{\cosh(h +\beta (2 q^+_\beta-1) )}{\cosh (h +\beta (2 q^-_\beta -1) 
      )}\right)
      \\
      \varphi_{\mathrm{edge},\beta}(w) &= \frac{1}{2}\ln \left(\frac{((q^+_\beta)^{2}+(1-q^+_\beta)^{2})\,\exp(w)
      + 2q^+_\beta(1-q^+_\beta)\,\exp(-w) }{((q^+_\beta)^{2}+(1-q^+_\beta)^{2})\,\exp(-w)
      + 2q^+_\beta(1-q^+_\beta)\,\exp(w) }\right)
      \end{split}
  \end{equation}
  Note that  $ \varphi_{\mathrm{vertex},\beta}, \varphi_{\mathrm{edge},\beta}$ are invertible on their respective ranges, and their inverse functions are defined by:
  \begin{equation}\label{eq:inverse of edge and weight map function}
  \begin{aligned}
   \forall h, |\tanh (h)| < 2q^+_\beta-1: \varphi_{\mathrm{vertex},\beta}^{-1}(h) &= \tanh^{-1} (\frac{\tanh(h) }{2q^+_\beta-1} )\\
   \forall w, |\tanh (w)|< (2q^+_\beta-1)^2: \varphi_{\mathrm{edge},\beta}^{-1}(w) &= \tanh^{-1}(\frac{\tanh(w)}{(2q^+_\beta-1)^2}) 
       \end{aligned}
  \end{equation}
For justification of this fact, see \cref{prop:inverse function of edge and vertex weight}.

For $ Y_v\in\set{\pm 1}$  and $\tilde{\sigma}_i\in \set{\pm 1}$, we also define:
\begin{equation}\label{eq:define vertex bias}
    Q_i^{Y_v}(\tilde{\sigma}_i) = \frac{\exp(\tilde{\sigma}_i (\tilde{h}_i + Y_v\beta (2 q^+_\beta-1)) ) }{ 2\cosh ( \tilde{h}_i  +Y_v \beta (2 q^+_\beta-1) )  }
\end{equation}

\subsection{Gadget construction preserves hardness of learning-to-generalize}
\label{sec:reduction preserves properties}
In this section, we show that any learning-to-generalize algorithm for $\tilde{\mu}$ would imply a learning-to-generalize algorithm for $\mu.$ 
Throughout, we write $A = (1\pm \epsilon)^kB$ as shorthand for $(1-\epsilon)^k B\leq A\leq (1+\epsilon)^k B.$
\allowdisplaybreaks
\begin{theorem}[Properties of gadget construction]\label{thm:set parameter choice}

   Let $H$ be a weighted graph on $ V_H,$ with $ |V_H|=m,$ and let $\mu \equiv \mu_H: \set{\pm 1}^{V_H} \to \R_{\geq 0}, J\in \R^{m\times m}, h\in \R^m$ be the Ising model, interaction matrix, and external field corresponding to the graph $H.$

   Fix a parameter $\gamma>1.$ Let $\widetilde{\gamma} =\min\set{\gamma, 2}$ and $\beta =\frac{1+\widetilde{\gamma}}{2}\in (1,2].$
  Fix a parameter $\epsilon \in (0,0.05).$

   In the construction of $\mathcal{H}$ in \Cref{subsec:sparsifying Ising model construction}, for $u, v\in V_H$, set 
   \[ t^0_v = \max\set{\left\lceil \frac{2}{2 q^+_\beta -1}\cdot  |h_v|\right\rceil,1} ,\quad e_{uv}= \max\set{\left\lceil \frac{\max\set{2,1/\tanh(\frac{\widetilde{\gamma}-1}{8})}}{(2 q^+_\beta -1)^2}\cdot |J_{uv}|\right\rceil,1}, \]
   \[r_v = \min\set{r \in \N  \mid r\text{ is odd and } r \geq C_{\gamma} \cdot t_v^3 m^2 \epsilon^{-2} \log(m/\epsilon)}  ,\]
where  $ C_{\gamma}>0$ is a constant dependent only on $ \gamma.$ 

   Let $\tilde{\mu}\equiv \mu_{\mathcal{H}}: \set{\pm 1}^{V_{\mathcal{H}}} \to \R_{\geq 0}, \tilde{J}\in\R^{|V_{\mathcal{H}}|\times |V_{\mathcal{H}}|}, \tilde{h}\in \R^{|V_{\mathcal{H}}|}$ be the Ising model, interaction matrix, and external field corresponding to the graph $ \mathcal{H}.$ Let $W = \max\set{\norm{J}_\infty  + \norm{h}_\infty,m}.$ The following properties hold:
   
   \begin{enumerate}
\item The number of vertices in $\mathcal{H}$, interaction matrix $\tilde{J}$ and external field $\tilde{h}$  satisfy:
\[ N: = |V_{\mathcal{H}} | = O_{\gamma}\left(m^3  W^3 \epsilon^{-2} \log(m/\epsilon)\right).\]
\[1< \lambda_{\max}(\tilde{J}) -\lambda_{\min}(\tilde{J}) \leq \gamma \quad \text{and} \quad \norm{\tilde{J}}_\infty  + \norm{\tilde{h}}_\infty = O(1).\]

    \item The phase readout map $Y$ satisfies:
    \begin{equation}\label{eq:gadget approximate original Ising formal (inside main reduction thm)}
         \forall y \in \set{\pm 1}^{V_H}: \Pr_{\sigma\sim\tilde{\mu}}[Y(\sigma) = y] = (1\pm \epsilon) \Pr_{y\sim \mu}[y]
    \end{equation}
    Moreover, there exists a normalization factor $A$,  that only depends on $J$ and $\mathcal{H}$, can be computed in $O(N)$ time, and satisfies: 
   \begin{equation}
     \forall y\in \set{\pm 1}^{V_H}: Z_{\mathcal{H}}(y)= (1\pm \epsilon)\cdot A \cdot \mu_H(y)   
   \end{equation}
   \item
Let $\Psi$ be a deterministic map on $\set{\pm 1}^{V_H},$ and $(\Psi_*\mu)_{\max} = \max_{c} \Psi_*\mu (c)$ be the maximum mass of an element in the push-forward distribution $\Psi_*\mu.$ Suppose $ (\Psi_*\mu)_{\max}\leq 1/3.$ 
     If $\mu$ is $ (\eta,\zeta,k,t, \Psi)$-hard for learning-to-generalize, $\tilde{\mu}$ is $(\tilde{\eta},\tilde{\zeta},\tilde{k}, \tilde{t}, \tilde{\Psi})$-hard for learning-to-generalize where 
     \[ \tilde{\eta} =5  \exp (5k (\Psi_*\mu)_{\max}) \eta, \quad \tilde{\zeta} =  (1+\epsilon)\zeta, \quad \tilde{k}= \lceil (k+1)/2\rceil, \quad \tilde{\Psi} = \Psi\circ Y   \]
and
\begin{align*}
   t = &\tilde{t}  \\+O_{\gamma} &( k^3 m^{10} W^{11} \epsilon^{-4} \log^2(m/\epsilon)  \log (k/\eta) \log (m k W )  \log k 
   + k m^8 W^9  \epsilon^{-4} \log^2(m/\epsilon) \log (k/\eta) \log (mkW/\eta) ). 
\end{align*}

In particular, let $\Lambda \subseteq \set{\pm 1}^m$ be such that $ \Pr_{y\sim \mu}[y\in \Lambda] \leq \zeta$ and that no algorithm can $ (\eta, k,t,\Psi,
\Lambda)$-learn-to-generalize $ \mu.$ 

Denoting $ \tilde{\Lambda} = Y^{-1}(\Lambda)\subseteq \set{\pm 1}^N,$  then $\Pr_{\sigma\sim \tilde{\mu}} [\sigma \in \tilde{\Lambda}] \leq (1+\epsilon)\zeta$ and no algorithm can $ (\tilde{\eta}, \tilde{k},\tilde{t},\tilde{\Psi},
\tilde{\Lambda})$-learn-to-generalize $ \tilde{\mu}.$ 
\end{enumerate}

\end{theorem}

\Cref{thm:set parameter choice} can be easily generalized into \cref{cor:sparisfy a family of Ising models}, which transforms a family of (potentially dense) Ising models into a family of bounded-width Ising models near the spectral threshold while preserving the hardness of learning-to-generalize.
\begin{theorem}\label{cor:sparisfy a family of Ising models}
     Let $(\mu_{J,h})_{(J,h)\sim \mathcal{P}}$  be a family of Ising models on $m$ vertices, where the model parameter $(J,h)$ is drawn from a distribution $\mathcal{P}.$ Suppose that there exists $\bar{J} =(\bar{J}_{uv})\in \R_{\geq 0}^{m \times m}, \bar{h}=(\bar{h}_u) \in \R_{\geq 0}^{m}$ where:
     \[\forall (J,h) \in \supp(\mathcal{P}): |J_{uv}|\leq \bar{J}_{uv}, |h_u|\leq \bar{h}_u \, \forall u,v\in [m] .\]
     Let $\bar{W} =\max\set{ \norm{\bar{J}}_\infty  + \norm{\bar{h}}_\infty,m}.$
     Fix a parameter $\gamma>1.$ Let $\widetilde{\gamma} =\min\set{\gamma, 2}$ and $\beta =\frac{1+\widetilde{\gamma}}{2}\in (1,2].$ Applying the construction in \cref{subsec:sparsifying Ising model construction} for each $(J,h) \in \supp(\mathcal{P})$ with $\epsilon = 0.04$ and
     \[ t^0_v =\max\set{  \left\lceil \frac{2}{2 q^+_\beta -1}\cdot  |\bar{h}_v|\right\rceil,1 } ,\quad e_{uv}= \max\set{\left\lceil \frac{\max\set{2,1/\tanh(\frac{\widetilde{\gamma}-1}{8})}}{(2 q^+_\beta -1)^2}\cdot |\bar{J}_{uv}|\right\rceil,1} , \]
   \[r_v = \min\set{r\in \N \mid r \text{ is odd and } r \geq C_{\gamma} \cdot t_v^3 m^2 \epsilon^{-2} \log(m/\epsilon)}  ,\]
where  $ C_{\gamma}>0$ is the constant dependent only on $ \gamma$ from \cref{thm:set parameter choice}, 
     we obtain $ (\tilde{J},\tilde{h})=\varphi(J,h)$ where
      $N= \tilde{O}_{\gamma}(m^3  \bar{W}^3 )$
      and  $\tilde{J}\in \R^{N\times N} , \tilde{h}\in \R^N$ satisfies
      \[1<\lambda_{\max}(\tilde{J}) -\lambda_{\min}(\tilde{J}) \leq \gamma\quad \text{and} \quad \norm{\tilde{J}}_\infty  + \norm{\tilde{h}}_\infty = O(1).\]
     Let  $ \tilde{\mathcal{P}} =\varphi_* \mathcal{P}$ be the push-forward distribution of $ \mathcal{P}$ by the map $\varphi.$ Let $\Psi$ be a deterministic map on $\set{\pm 1}^m,$ and $\Psi_{\max} = \max_{(J,h)\in \supp(\mathcal{P})}\max_{c} \Psi_*\mu_{J,h} (c)$ be the maximum mass of an element in the push-forward distribution $\Psi_*\mu_{J,h},$ taken over all distributions $ \mu_{J,h}$ in the family $(\mu_{J,h})_{(J,h)\sim \mathcal{P}}.$ Suppose $ \Psi_{\max}\leq 1/3.$ 
     
     Then, if the family of distributions $(\mu_{J,h})_{(J,h)\sim \mathcal{P}}$  is $ (\eta,\zeta,k,t, \Psi)$-hard for learning-to-generalize, the family of distributions $ (\tilde{\mu}_{\tilde{J}, \tilde{h}})_{(\tilde{J},\tilde{h}) \sim  \tilde{P}}$ is $ (\tilde{\eta},\tilde{\zeta},\tilde{k}, \tilde{t}, \tilde{\Psi})$-hard for learning-to-generalize where: 
     \[ \tilde{\eta} = 5 \exp (5k \Psi_{\max} ) \eta,  \quad \tilde{\zeta} =  (1+\epsilon)\zeta, \quad \tilde{k}= \lceil (k+1)/2\rceil, \quad \tilde{\Psi} = \Psi\circ Y   \]
    and: 
\[ t = \tilde{t} +O_{\gamma} \left( k^3 m^{10} \bar{W}^{11} \log^2(m)  \log (k/\eta) \log (m k \bar{W} )  \log k + k m^8 \bar{W}^9  \log^2(m) \log (k/\eta) \log (mk\bar{W}/\eta) \right)\]

In particular, for $ (J,h)\in \supp(\mathcal{P})$, let $\Lambda_{J,h} \subseteq \set{\pm 1}^m$ be such that $ \Pr_{y\sim \mu_{J,h}}[y\in \Lambda_{J,h}] \leq \zeta$ and that for any algorithm $\mathcal{A}$ that takes the model parameter $(J,h)\sim \mathcal{P}$  and   $k$  i.i.d. samples from $\mu_{J,h},$ and runs in time $t$:
 \[ \Pr_{(J,h) \sim \mathcal{P}, X \sim \mu_{J,h}^{\otimes k}, y \sim \mathcal{A}_{(J,h),X}} [ y \in \Psi^{-1} (\Psi(X))\text{ or } y \in \Lambda_{J,h}  ] > 1-\eta\]

Then, for $(\tilde{J},\tilde{h}) = \varphi(J,h) $, define $ \tilde{\Lambda}_{\tilde{J},\tilde{h}}  = Y^{-1}(\Lambda_{J,h})\subseteq \set{\pm 1}^N$. We have: $\Pr_{\sigma\sim \tilde{\mu}_{\tilde{J}, \tilde{h}} } [\sigma \in \tilde{\Lambda}] \leq (1+\epsilon)\zeta$ and for any algorithm $\tilde{\mathcal{A}}$  that takes the model parameter $(\tilde{J},\tilde{h})\sim \tilde{\mathcal{P}}$  and   $\tilde{k}$  i.i.d. samples from $\tilde{\mu}_{\tilde{J},\tilde{h}},$ and runs in time $\tilde{t}$:
\[ \Pr_{(\tilde{J},\tilde{h}) \sim \tilde{\mathcal{P}}, X \sim \tilde{\mu}_{\tilde{J}, \tilde{h}}^{\otimes \tilde{k}}, y \sim \tilde{\mathcal{A}}_{(\tilde{J},\tilde{h}),X}} [ y \in \tilde{\Psi}^{-1} (\tilde{\Psi}(X))\text{ or } y \in \tilde{\Lambda}_{\tilde{J},\tilde{h}}  ] > 1-\tilde{\eta}\]
\end{theorem}

\begin{remark}
In the construction of the family of Ising models in \cref{cor:sparisfy a family of Ising models}, first, we set $t^0_v, e_{uv}, r_v$ for distinct $u,v \in [m]$ according to $(\bar{J},\bar{h}).$ For each vertex $v\in [m]$, construct the disjoint subsets $S^{\text{vertex}}_v$ of size $ t^0_v$, $R_v$ of size $r_v$, $S_v^u$ of size $e_{uv}$ for $u\neq v.$ Recall that $S^{\text{edge}}_v = \bigcup_{u\neq v}S_v^u ,$ $S_v = S^{\text{vertex}}_v  \cup S^{\text{edge}}_v .$ For each $v$, we add edges of weight $ \beta/r_v$ between $i,j$ in $R_v \cup S_v$ that are not both in $S_v,$ and set the external field of vertices in $ S^{\text{edge}}_v\cup R_v$ to $0.$ 
We add a perfect matching $E(S_u^v, S_v^u)$ between $ S_u^v$ and $S_v^u$ for each pair of distinct vertices $u,v$, with edge weight to be specified later. 

Then, for a given $(J,h),$ we construct $(\tilde{J},\tilde{h})$ according to \cref{subsec:sparsifying Ising model construction}, meaning that we set the external fields of vertices in $S^{\text{vertex}}_v$ according to $ h_v$ and set the weights of the edges in the perfect matching $E(S_u^v, S_v^u)$ between $ S_u^v$ and $S_v^u$ according to $ J_{uv}.$ Specifically:
\[\forall i \in S^{\text{vertex}}_v: \tilde{h}_i = \varphi_{\mathrm{vertex},\beta}^{-1}\left(\frac{h_v}{t_v^0}\right), \forall (i,j) \in E(S_u^v, S_v^u): \tilde{J}_{ij} = \varphi_{\mathrm{edge},\beta}^{-1} \left(\frac{J_{uv}}{e_{uv}}\right) \]
\end{remark}

\subsection{Proof of \cref{thm:set parameter choice}, parts (1) and (2)}
\label{sec:nearcritical}

We first prove the following generalization of \cite[Lemma 2.1]{GKK24}.
\begin{proposition}\label{prop:single gadget control}
Consider a set of vertices $V$ where $|V|=n$ and a subset $S\subseteq V$ where $ |S| = t,$ and $r= n-t$ is odd.
Consider the Ising model $\mu$ on $V$ with interaction matrix \[J =\frac{\beta}{r}( \mathbf{1}_V \mathbf{1}_V^\intercal -  \mathbf{1}_S   \mathbf{1}_S^\intercal)\footnote{where $\mathbf{1}_V$ denotes the all-1 vector, where $ \mathbf{1}_S = (\mathbf{1}[v\in S])_{v\in V}$ is the characteristic vector for $S.$}\] and  external field $\mathbf{h}= (h_v)_{v\in V}$ where $h_v =0 \forall v \not \in S.$ Consider a parameter $\epsilon \in (0,1).$
There exists a constant $c_\beta>0$ dependent only on $\beta$ so that if $r$ is odd and
\begin{equation}\label{eq:defining r in construction}
   r\geq c_\beta\cdot t^3 \epsilon^{-2} \log(1/\epsilon), 
\end{equation}
the following holds.
    Let
    \begin{equation}
    \begin{split}
        \lambda_{+}(\mathbf{h}) &=  \exp(r f_\beta(q^+_\beta)) \prod_{v\in S}  2\cosh ( h_v  +  \beta (2 q^+_\beta-1)  )\\ \lambda_{-}(\mathbf{h})  &=  \exp(r f_\beta(q^-_\beta)) \prod_{v\in S}  2\cosh ( h_v  + \beta (2 q^-_\beta -1)  )  
    \end{split}
    \end{equation}
        
Let
\begin{equation}\label{def:phase partition function}
        Z^+ =\sum_{ \sigma \in \{\pm 1\}^V : y(\sigma_{V\setminus S}) = +1 } \exp\!\left(\tfrac{1}{2}\, \sigma^\top J \sigma +\sigma^\top h \right ), \quad Z^- =\sum_{ \sigma \in \{\pm 1\}^V : y(\sigma_{V\setminus S}) = -1 } \exp\!\left(\tfrac{1}{2}\, \sigma^\top J \sigma +\sigma^\top h \right ) 
\end{equation}
    then there exists a constant $ a$ that can be computed in $ O(r)$ time such that:
    \begin{equation}\label{eq:phase partition function approximation}
       Z^+ = (1\pm \epsilon) \cdot a\cdot \lambda_{+} (\mathbf{h}), \quad   Z^- = (1\pm \epsilon) \cdot a\cdot \lambda_{-} (\mathbf{h})
    \end{equation}
Moreover, if for any $\tau_S\in \set{\pm 1}^S$, we denote
\begin{equation}\label{eq:product distr}
\begin{split}
    Q_S^+(\tau_S) &=\prod_{v\in S}  \frac{\exp(\tau_v (h_v +\beta (2 q^+_\beta -1)) ) }{ 2\cosh ( h_v +\beta (2 q^+_\beta-1) )  }\\
    Q_S^-(\tau_S) &=\prod_{v\in S}  \frac{\exp(\tau_v (h_v +\beta (2 q^-_\beta -1)) ) }{ 2\cosh( h_v +\beta (2 q^-_\beta -1) )  }
\end{split}
\end{equation}
it holds that:
\begin{equation}\label{eq:conditional distr on + become product}
   (1-\epsilon)  Q_S^+(\tau_S)\leq  \Pr[ \sigma_S = \tau_S| y_{V\setminus S} (\sigma) = +1]  \leq(1+\epsilon) Q_S^+(\tau_S)
\end{equation}
and
\begin{equation}\label{eq:conditional distr pn - become product}
   (1-\epsilon) Q_S^-(\tau_S)\leq  \Pr[ \sigma_S = \tau_S| y_{V\setminus S} (\sigma) = -1]  \leq(1+\epsilon)  Q_S^-(\tau_S)
\end{equation}

\end{proposition}
We need the following helper propositions.
\begin{proposition}\label{prop:deviation bound for entropy function H alpha}
    Let the function $H $ be as in \Cref{eq:binary entropy function H}, then there exists an absolute constant $C'$ s.t. for any $\alpha \in (0,1)$ where $r\alpha \in \N: $
    \[r H(\alpha) - C' \ln  r \leq \ln  \binom{r}{\alpha r} \leq r H(\alpha) + C' \ln  r\]
\end{proposition}
\begin{proof}
    The statement follows from $\binom{r}{\alpha r} = \frac{r!}{(\alpha r)! ((1-\alpha)r)!} $ and the Stirling estimation 
    \[ \ln n! = n \ln n - n  + O(\ln n)\]
\end{proof}

\begin{proposition}\label{prop:property of q^+}
    Let the function $f_{\beta}$ be as defined in \cref{eq:magnetization function f}. Suppose $\beta > 1.$ The function $f_\beta$ has two maximizers at $q^+_\beta\in (1/2,1), q^-_\beta\in (0,1/2)$ which satisfy:
    \begin{equation}
        f_\beta'(\alpha) = \ln\!\left(\frac{1-\alpha}{\alpha}\right) + 2\beta(2\alpha - 1) = 0
    \end{equation}
    Note that $q^+_\beta + q^-_\beta=1,$ and $ f_{\beta}''(q^+_\beta)< 0,f_{\beta}''(q^-_\beta)<0.$ The function $f_\beta$ is increasing in $(0,q^-_\beta),$ decreasing in $ (q^-_\beta, 1/2),$ increasing in $ (1/2, q^+_\beta)$ and decreasing in $ (q^+_\beta,1).$
    
    Moreover, $q^+_\beta$ is a strictly increasing function of $\beta$ i.e.  $q^+_\beta >q^+_{\beta'}\Leftrightarrow \beta >\beta'.$
\end{proposition}

\begin{proposition}\label{prop:deviation bound for f alpha for bounded beta}
    Let the function $f_{\beta}$ be as defined in \cref{eq:magnetization function f}, and $q^+_\beta \in (1/2,1)$ satisfies \cref{eq:def of q+ and q-} be the unique maximizer of $f_\beta$ on $(1/2,1).$
Suppose $\beta \in (1,2].$
There exist constants $C>0 $ and $ \delta_0 \in (0,1/2)$ that depend only on $\beta$  so that: for any $\delta \in (0,\delta_0]$  and $ \alpha, \alpha'\in [1/2,1)$ where $|\alpha' -q^+_\beta|\geq \delta$ and $ |\alpha - q^+_\beta| \leq \delta/2 :$
   \[ f_\beta(\alpha') - f_\beta(\alpha)\leq -C \delta^2\]
\end{proposition}
We leave the proofs of \cref{prop:property of q^+} and \cref{prop:deviation bound for f alpha for bounded beta} to \cref{sec:defer proof from ising sparsification}.

\begin{proof}[Proof of \cref{prop:single gadget control}]
Recall from the setup in \cref{thm:set parameter choice} that $\beta \in (1,2].$
We omit the subscript $\beta$ for convenience.
    Let $\alpha \in [0,1]$ be such that $\alpha r$ is an integer. For a configuration 
$\tau \in \{-1, +1\}^S$, let $Z^\alpha(\tau)$ be the contribution to the partition 
function by configurations $\sigma$ with $\alpha r$ spins from $V \setminus S$ set to $+1$, 
$(1-\alpha) r$ spins from $V \setminus S$ set to $-1$ and $\sigma_S = \tau$. Concretely, let
\begin{equation}
\begin{split}
    Z^\alpha(\tau) &= 
\sum_{\sigma \in \{\pm 1\}^V : \sigma_S = \tau,\; m_{V\setminus S}(\sigma) = (2\alpha - 1) r}
\exp\!\left(\tfrac{1}{2}\, \sigma^\top J \sigma +\sigma^\top h \right ) \\
&= \binom{r}{\alpha r} \exp\left(  \frac{\beta}{2}\cdot r (2\alpha-1)^2 + \sum_{v\in S}( h_v +  \beta (2\alpha-1) ) \tau_v\right).
\end{split}
\end{equation}

We also define
\begin{equation}
    \begin{split}
        Z^\alpha =\sum_{\tau\in \set{\pm 1}^S}  Z^\alpha (\tau) 
        &=  \binom{r}{\alpha r} \exp\left(\frac{\beta}{2}\cdot r (2\alpha-1)^2\right)\sum_{\tau\in \set{\pm 1}^S}  \exp\left( \sum_{v\in S}( h_v +  \beta (2\alpha-1) ) \tau_v \right)\\
        &= \binom{r}{\alpha r} \exp\left(\frac{\beta}{2}\cdot r (2\alpha-1)^2 \right)\prod_{v\in S}  2\cosh ( h_v  +  \beta (2\alpha-1)  )
    \end{split}
\end{equation}
Next, we define
\begin{equation}
    \begin{split}
        Z^+(\tau) =\sum_{k= \lceil r/2\rceil}^{r} Z^{k/r} (\tau)\quad &\text{and} \quad
         Z^-(\tau) =\sum_{k= 0}^ {\lfloor r/2\rfloor} Z^{k/r} (\tau)
    \end{split}
\end{equation}
We note that
\[  Z^+ =\sum_{k= \lceil r/2\rceil}^{r} Z^{k/r} \quad \text{and} \quad
         Z^- =\sum_{k= 0}^ {\lfloor r/2\rfloor} Z^{k/r} \]

Fix parameter $\delta >0$ to be chosen later.

Let \[L^+= (r/2,r]\cap \N\quad \text{and} \quad L^-=[0,r/2)\cap \N\],  \[K_\delta^+= \set{k\in L^+ : |k/r - q^+| \leq  \delta } \quad \text{and} \quad K_\delta^- =\set{k\in L^- : |k/r - q^-| \leq  \delta } .\]

Let $\delta_0, C $ be as defined in \cref{prop:deviation bound for f alpha for bounded beta}. Note that $\delta_0\in(0,1/2)$ and $C>0.$
Consider $\delta\in (0,\delta_0]$,  and $\ell, k\in L^+$ such that $ |\ell/r-q^+|>  \delta$ and $ |k/r- q^+|\leq \delta/2.$ By \cref{prop:deviation bound for entropy function H alpha} and \cref{prop:deviation bound for f alpha for bounded beta}, we have:
\begin{align*}
    \ln Z^{\ell/r} - \ln Z^{k/r} &\leq r (f(\ell/r)-f(k/r))+ \sum_{v\in S}\left| \ln\left(\frac{ \cosh ( h_v  + \beta (2\ell/r-1)  ) } {\cosh ( h_v  + \beta (2k/r-1) )}  \right)\right| +  O(\log r)\\
    &\leq -r \delta^2 C  + 2\beta  t \cdot \left|\frac{\ell-k}{r}\right| +O(\ln r)\\
    &\leq -r \delta^2 C + \beta t + O(\ln r)
\end{align*}
where we bound the second term using the fact that $ |S|=t$ and the boundedness of $\tanh:$
\[|\ln  \cosh(x) -\ln \cosh(y)| =|\int_{z\in (x,y)} \tanh(z) dz | \leq  |x-y|. \]
Assuming in addition that $ r \geq \frac{1}{\delta},$ then $K_{\delta/2}^+ \neq \emptyset. $ Indeed, if $K_{\delta/2}^+ = \emptyset$ then $ \lceil r q^+\rceil, \lfloor r q^+\rfloor$ are not in $K_{\delta/2}^+, $ thus
\[ r q^+ -  \lfloor r q^+\rfloor > \frac{\delta}{2} \cdot r\quad \text{and} \quad \lceil r q^+\rceil - r q^+ > \frac{\delta}{2} \cdot r \Rightarrow 1\geq  \lceil r q^+\rceil- \lfloor r q^+\rfloor >\delta r.  \]
Hence
\begin{align*}
    \frac{\sum_{k\in L^+\setminus K_\delta^+ } Z^{k/r} }{ \sum_{k \in K_\delta^+}Z^{k/r}} \leq \frac{(r- |K_\delta^+|) }{|K_{\delta/2}^+|} \cdot \exp(-r \delta^2 C + \beta t + O(\ln r)) \leq \exp(-\frac{8}{9}\cdot r\delta^2 C) 
\end{align*}
where we choose $ r \geq\max \set{ \frac{20\beta t }{\delta^2  C}, 10}.$ In this case,
\begin{equation}\label{eq:controlling concentration no pinning}
    1\leq \frac{\sum_{k\in L^+ } Z^{k/r} }{ \sum_{k \in K_\delta^+}Z^{k/r}} \leq (1+  \exp(-\frac{8}{9}\cdot r\delta^2 C) )
\end{equation}
Note that $ L^- = \set{r(1-k/r) | k \in L^+},$ and  since $q ^- = 1- q^+,$ $ K_{\delta}^- = \set{r(1-k/r) | k \in K_{\delta}^+}.$ By a similar argument,
\begin{equation}\label{eq:controlling concentration no pinning minus phase}
    1\leq \frac{\sum_{k\in L^- } Z^{k/r} }{ \sum_{k \in K_\delta^-}Z^{k/r}} \leq (1+  \exp(-\frac{8}{9}\cdot r\delta^2 C) )
\end{equation}
Next, for $k\in K_\delta^+$ we have:
\begin{equation}\label{eq:concentration on + phase non pinning}
\begin{split}
    \left|\ln \frac{Z^{k/r}}{\binom{r}{k} \exp\left(\frac{\beta}{2}\cdot r (2k/r-1)^2 \right) } - \ln \frac{\lambda_{+}(\mathbf{h})}{\exp(r f(q^+))}\right| &\leq  \sum_{v\in S} \left|\ln\left(\frac{ \cosh ( h_v  + \beta (2 k/r-1)  ) } {\cosh ( h_v  + \beta (2q^+-1) )}  \right) \right| \\
    &\leq 2\beta t |k/r-q^+|\leq 2 \beta t \delta. 
\end{split}
\end{equation}
Similarly, for $k\in K_{\delta}^-$
\begin{equation}\label{eq:concentration on - phase non pinning}
    \left|\ln  \frac{Z^{k/r}}{\binom{r}{k} \exp\left(\frac{\beta}{2}\cdot r (2k/r-1)^2 \right) } - \ln \frac{\lambda_{-}(\mathbf{h})}{\exp(r f(q^-))}\right| \leq  \sum_{v\in S} \left|\ln\left(\frac{ \cosh ( h_v  + \beta (2 k/r-1)  ) } {\cosh ( h_v  + \beta (2q^--1) )}  \right) \right|\leq 2 \beta t \delta. 
\end{equation}

Hence, by \cref{eq:controlling concentration no pinning}  and \cref{eq:concentration on + phase non pinning}: 
\begin{equation}
\begin{aligned}
    \sum_{k\in L^+ } Z^{k/r} &= (1 \pm \exp(-\frac{8}{9}\cdot r\delta^2 C) )   \sum_{k\in K_\delta^+ } Z^{k/r} \\
    &= (1 \pm \exp(-\frac{8}{9}\cdot r\delta^2 C) ) \exp(\pm 2\beta t \delta )\cdot \frac{\sum_{k\in K_\delta^+} \binom{r}{k} \exp\left(\frac{\beta}{2}\cdot r (2k/r-1)^2 \right)  }{\exp( rf (q^+)) }\cdot \lambda_{+}(\mathbf{h})
    \end{aligned}
\end{equation}

By \cref{eq:controlling concentration no pinning minus phase,eq:concentration on - phase non pinning}: 
\begin{equation}
\begin{aligned}
    \sum_{k\in L^- } Z^{k/r} &= (1 \pm \exp(-\frac{8}{9}\cdot r\delta^2 C) )   \sum_{k\in K_\delta^- } Z^{k/r} \\
    &= (1 \pm \exp(-\frac{8}{9}\cdot r\delta^2 C) ) \exp(\pm 2\beta t \delta )\cdot \frac{\sum_{k\in K_\delta^-} \binom{r}{k} \exp\left(\frac{\beta}{2}\cdot r (2k/r-1)^2 \right)  }{\exp( rf (q^-)) }\cdot \lambda_{-}(\mathbf{h})
    \end{aligned}
\end{equation}
Next, we set $\delta$ and $r:$
\begin{equation}\label{eq:parameter choice for single gadget}
\delta = \min \set{\frac{\epsilon}{24 \beta  t}, \delta_0}    \quad \text{and}\quad r =  \max \set{ \frac{20\beta t }{\delta^2  C}, \frac{2 \log(10/\epsilon) }{\delta^2  C} , \frac{1}{\delta},10} = \Theta_{\beta} (t^3 \epsilon^{-2} \log(1/\epsilon) ),
\end{equation}
then 
\begin{equation}\label{eq:bound approx error in construction}
    (1 + \exp(-\frac{8}{9}\cdot r\delta^2 C) )  \cdot \exp(4 \beta t \delta) \leq (1 + \epsilon/3)\, \text{and}\, (1 - \exp(-\frac{8}{9}\cdot r\delta^2 C) )  \cdot \exp(-4 \beta t \delta) \geq (1 - \epsilon/3)
\end{equation}
 Since $ f(q^-) = f(q^+)$ and $K_\delta^{-} = \set{r-k | k\in K_\delta^+},$ we have: \[ \sum_{k\in K_\delta^-} \binom{r}{k} \exp\left(\frac{\beta}{2}\cdot r (2k/r-1)^2 \right)  = \sum_{k\in K_\delta^+} \binom{r}{k}\exp\left(\frac{\beta}{2}\cdot r (2k/r-1)^2 \right).\] 
 
 Hence, set \[a: = \frac{\sum_{k\in K_\delta^+} \binom{r}{k}\exp\left(\frac{\beta}{2}\cdot r (2k/r-1)^2 \right) }{\exp( rf (q^+))} ,\] we can compute $a$ in $O(r)$ time and 
\begin{equation}\label{eq:approximate partition function of plus phase no pinning}
    \sum_{k\in L^+ } Z^{k/r} =  (1\pm \epsilon/3)  \cdot a \cdot \lambda_{+}(\mathbf{h}), \quad  \sum_{k\in L^- } Z^{k/r} = (1\pm \epsilon/3)\cdot a \cdot \lambda_{-}(\mathbf{h})
\end{equation}

This establishes \cref{eq:phase partition function approximation}.

Next, we show \cref{eq:conditional distr on + become product}.
Similar to above, using \cref{prop:deviation bound for entropy function H alpha} and \cref{prop:deviation bound for f alpha for bounded beta}, for any configuration $\tau = \tau_S$ of $S,$ and
for $\ell, k\in (r/2,r]$ such that $ |\ell/r-q^+|>  \delta$ and $ |k/r- q^+|\leq \delta/2,$ we have:
\begin{align*}
    \ln Z^{\ell/r}(\tau) - \ln Z^{k/r}(\tau)
    &\leq -r \delta^2 C + \beta t + O(\ln r)
\end{align*}
which implies
\begin{equation}\label{eq:controlling concentration of pinning}
    1\leq \frac{\sum_{k\in L^+ } Z^{k/r} (\tau) }{ \sum_{k \in K_\delta^+}Z^{k/r} (\tau) } \leq (1+  \exp(-\frac{8}{9}\cdot r\delta^2 C) )
\end{equation}

In addition, for any $k\in K_\delta^+$ and any two configurations $\tau, \tau'$ of $S,$ we have:
\begin{equation}\label{eq:controlling ratio between two pinning}
\begin{aligned}
    -4\beta t \delta \leq  \ln\frac{ Z^{k/r}(\tau)}{ Z^{k/r}(\tau') } - \ln \frac{Q_S^+(\tau)}{Q_S^+(\tau')}
     &= 2\beta (k/r-q^+) \sum_{v\in S}(\tau_v-\tau'_v)\leq 4 \beta t \delta 
\end{aligned}
\end{equation}

\cref{eq:controlling concentration of pinning} and \cref{eq:controlling ratio between two pinning} together imply:
\begin{equation}
\begin{split}
\frac{ \Pr[ \sigma_S = \tau\mid y_{V\setminus S} (\sigma) = +1] }{ \Pr[ \sigma_S = \tau'\mid y_{V\setminus S} (\sigma) = +1]  }   &=  \frac{\sum_{k\in L^+ } Z^{k/r} (\tau) }{ \sum_{k \in L^+}Z^{k/r} (\tau') } \\&=(1 \pm \exp(-\frac{8}{9}\cdot r\delta^2 C) )  \cdot \frac{\sum_{k\in K_\delta^+ } Z^{k/r} (\tau) }{ \sum_{k \in K_\delta^+}Z^{k/r} (\tau') }\\
   &= (1 \pm \exp(-\frac{8}{9}\cdot r\delta^2 C) )  \cdot \exp(\pm 4 \beta t \delta) \cdot \frac{Q_S^+(\tau)}{Q_S^+(\tau')}\\
   &= (1\pm \epsilon/3) \cdot \frac{Q_S^+(\tau)}{Q_S^+(\tau')}
   \end{split}
\end{equation}
where the last line follows from \cref{eq:bound approx error in construction}.
Hence
\begin{align*}
  \frac{1-\epsilon/3} {1+\epsilon/3}\cdot \frac{Q_S^+(\tau) }{\sum_{\tau'} Q_S^+(\tau')}    \leq  &\Pr[ \sigma_S = \tau| y_{V\setminus S} (\sigma) = +1]  \leq \frac{1+\epsilon/3}{1-\epsilon/3} \cdot \frac{Q_S^+(\tau) }{\sum_{\tau'} Q_S^+(\tau')} \\
\Rightarrow(1-\epsilon) Q_S^+(\tau)  \leq &\Pr[ \sigma_S = \tau| y_{V\setminus S} (\sigma) = +1]  \leq (1+\epsilon) Q_S^+(\tau)    
\end{align*}
This establishes \cref{eq:conditional distr on + become product}. \cref{eq:conditional distr pn - become product} follows from an analogous argument.
\end{proof}
As a consequence, we obtain the following corollary about the distribution of $Y(\sigma)$ where $\sigma \sim \mu_{\hat{\mathcal{H}}}.$
\begin{corollary}\label{cor:probability of configuration condition on phase in hat H}
    Let $ c_\beta >0$ be the constant in \cref{prop:single gadget control}. 
    Suppose that
    \begin{equation}\label{eq:defining r in full construction}
 \forall v\in V_H:  r_v \geq c_\beta\cdot t_v^3 \epsilon^{-2} \log(1/\epsilon)  \quad \text{and} \quad r_v \text{ is odd} 
\end{equation}
We can compute $A_0$ in $ O(|V_{\mathcal{H}}|)$ time where, for each configuration $Y\in \set{\pm 1}^{V_H},$ 
    
\[ Z_{\hat{\mathcal{H}}} (Y) =(1\pm\epsilon)^m  A_0 \cdot\exp(\sum_v h_v Y_v).\]
\end{corollary}
\begin{proof}
    For $v\in V_H,$ let $\nu_v$ be the Ising model corresponding to the graph $G_v$ in the construction, and $Z^+_v, Z^-_v$ be as defined in \cref{def:phase partition function} with respect to $G_v$.  Since there are no edges between the gadgets in $\hat{\mathcal{H}},$ $\mu_{\hat{\mathcal{H}}} = \bigotimes_{v\in V_H} \nu_v$ and $ Z_{\hat{\mathcal{H}}}(Y) = \prod_{v\in V_H} Z^{Y_v}_v.$ 
    Let
     \begin{equation}
    \begin{split}
        \lambda_{+,v} &= \exp(r_v f_\beta(q^+_\beta)) \prod_{i\in S_v}  2\cosh ( \tilde{h}_i  +  \beta (2 q^+_\beta-1)  )    \\ \lambda_{-,v} &=  \exp(r_v f_\beta (q^-_\beta))  \prod_{i\in S_v}  2\cosh ( \tilde{h}_i  + \beta (2 q^-_\beta -1)  )  
    \end{split}
    \end{equation}
    where we recall that $S_v =S^{\text{vertex}}_v\cup S^{\text{edge}}_v.$
    Since $ \tilde{h}_i = \tilde{h}_v \forall i\in S^{\text{vertex}}_v$ and $  \tilde{h}_i = 0 \forall i\in S^{\text{edge}}_v,$ we can compute $ \lambda_{+,v}$ and $\lambda_{-,v}$ in $O(1)$ time.
Note that $ f_\beta(q^+_\beta) = f_\beta (q^-_\beta).$ For $i\in S^{\text{edge}}_v$, $\tilde{h}_i=0$ and $(2 q^+_\beta-1) = -(2 q^-_\beta-1) $ so using $\cosh(x) = \cosh(-x)\forall x,$ we have: \[\cosh ( \tilde{h}_i  +  \beta (2 q^+_\beta-1)  )  =  \cosh ( \tilde{h}_i  +  \beta (2 q^-_\beta-1)  )  ,\] and $\forall i\in S^{\text{vertex}}_v: \tilde{h}_i = \tilde{h}_v = \varphi_{\mathrm{vertex},\beta}^{-1}\left(\frac{h_v}{t_v^0}\right).$   

Hence
 \[ \left(\frac{\lambda_{+,v}}{\lambda_{-,v}}\right)^{1/2} = \prod_{i\in S^{\text{vertex}}_v}  \left(\frac{\cosh ( \tilde{h}_i  +  \beta (2 q^+_\beta-1)  ) }{\cosh ( \tilde{h}_i  +  \beta (2 q^-_\beta-1)  ) } \right)^{1/2}= \exp( t^0_v\,  \varphi_{\mathrm{vertex},\beta} (\tilde{h}_i)) = \exp(h_v).\]

 Next,  by \cref{eq:phase partition function approximation} in  \cref{prop:single gadget control}, for each $v\in V_H,$ we can compute $a_v$ in $O(r_v)$ time so that $Z_v^{Y_v} = (1\pm \epsilon) a_v \lambda_{Y_v,v}.$ Hence: 
 \begin{align*}
    Z_{\hat{\mathcal{H}}} (Y) = (1\pm \epsilon)^m \prod_{v\in V_H}  (a_v \lambda_{Y_v,v} ) =  (1\pm \epsilon)^m A_0 \prod_{v\in V_H} \left(\frac{\lambda_{Y_v,v}}{\lambda_{-Y_v,v}}\right)^{1/2} =  (1\pm \epsilon)^m A_0\exp(\sum_v h_v Y_v)  
 \end{align*}
 where $A_0 = \prod_{v\in V_H}  (a_v (\lambda_{+,v} \lambda_{-,v})^{1/2}) .$ $A_0$ can be computed in $O(\sum_{v\in V_H}r_v) = O(|V_{\mathcal{H}}|)$ time.
\end{proof}

The following establishes part (2) of \cref{thm:set parameter choice}, and will also be used in the proof of part (3) of \cref{thm:set parameter choice}.
\begin{lemma}\label{lem:main construction probability of phase gadget}
Consider $ Y \in \set{\pm 1}^{V_H}. $  Recall that $\mathcal{S} =\bigcup_{v\in V_H} S_v.$ Recall the definitions of $\mu_{\mathcal{H}}(\tilde{\sigma}_{\mathcal{S}};Y )$ from \cref{eq:partial configuration condition on phase gadget}, $Z_{\mathcal{H}}(Y)$ from \cref{eq:define phase total function for  H}.

Let $ c_\beta >0$ be the constant that depends only on $\beta$ in \cref{prop:single gadget control}.
 Suppose that
    \begin{equation}
 \forall v\in V_H:  r_v \geq  30 c_\beta t_v^3 m^2 \epsilon^{-2} \log(m/\epsilon)  \quad \text{and} \quad r_v \text{ is odd} 
\end{equation}
Then we obtain the following properties:
\begin{itemize}
    \item  There exists a normalization factor $A$ that can be computed in $O(|V_{\mathcal{H}}| + m^2)$ time, where 
   \begin{equation}\label{eq:phase partition function approximate Ising model}
     \forall Y\in \set{\pm 1}^{V_H}: Z_{\mathcal{H}}(Y)= (1\pm \epsilon) \cdot A \cdot \mu_H(Y)   
   \end{equation}
   \item  
    $ \forall Y \in \set{\pm 1}^{V_H}: \Pr_{\sigma\sim\mu_{\mathcal{H}}}[Y(\sigma) = Y] = (1\pm \epsilon) \Pr_{Y\sim \mu_H}[Y]$
   \item For any
    $\mathcal{T}\subseteq \mathcal{S}$  and  configuration $ \tilde{\sigma}_{\mathcal{T}}\in \set{\pm 1}^{\mathcal{T}},$  can compute an approximation $\hat{\gamma} (\tilde{\sigma}_\mathcal{T}; Y)$ of $\gamma (\tilde{\sigma}_\mathcal{T}; Y)=\frac{\mu_{\mathcal{H}} (\tilde{\sigma}_\mathcal{T}; Y)}{Z_{\hat{\mathcal{H}}} (Y) } $ in $ O(|\mathcal{S}|^2)$ time where
    \[ \hat{\gamma} (\tilde{\sigma}_\mathcal{T}; Y) = (1\pm \epsilon/3) \gamma (\tilde{\sigma}_\mathcal{T}; Y)\]
    From \cref{cor:probability of configuration condition on phase in hat H}, \[\forall Y \in \set{\pm 1}^{V_H}: Z_{\hat{\mathcal{H}}} (Y) = (1\pm \epsilon/3) A_0 \exp(\sum_{v\in V_H}h_v Y_v) \] where the factor $A_0$ can be computed in  $ O(|V_\mathcal{H}|)  $ time. Consequently, \[\hat{\mu}_{\mathcal{H}} (\tilde{\sigma}_\mathcal{T}; Y) = (1\pm \epsilon) \mu_{\mathcal{H}} (\tilde{\sigma}_\mathcal{T}; Y)  \quad \text{where} \quad \hat{\mu}_{\mathcal{H}} (\tilde{\sigma}_\mathcal{T}; Y) =  \hat{\gamma} (\tilde{\sigma}_\mathcal{T}; Y) A_0 \exp(\sum_{v\in V_H}h_v Y_v) \] 
\end{itemize}

\end{lemma}

\begin{proof}[Proof of \Cref{lem:main construction probability of phase gadget}]
We  assume without loss of generality that the original interaction matrix $J$ corresponding to $\mu_H$ has zero diagonal; the nonzero diagonals do not change the distribution, and can be absorbed into the normalization factor $A.$

Let $\epsilon_0 >0$ to be chosen later. 
Suppose that $ \forall v\in V_H: r_v \geq c_\beta t_v^3 \epsilon_0^{-2} \log(1/\epsilon_0).$
Recall the definition of $Q_{i}^{ Y_v} (\tilde{\sigma}_{i})$ from \cref{eq:define vertex bias}. We first show that:
\begin{equation}\label{eq:single config condition on phase}
     \frac{\mu_{\mathcal{H}}(\tilde{\sigma}_{\mathcal{S}};Y )}{Z_{\hat{\mathcal{H}}}( Y) }  = (1\pm \epsilon_0)^m \left( \prod_{v\in V_H} \prod_{i\in S_v} Q_{i}^{ Y_v} (\tilde{\sigma}_{i})\right)  \prod_{\substack{\set{u,v} \in E_H\\ (i,j) \in E(S_u^v,S_v^u)}}\exp(\tilde{J}_{ij} \tilde{\sigma}_i \tilde{\sigma}_j) 
\end{equation}

For $v\in V_H,$ let $ \nu_v$ be the Ising distribution that corresponds to the  gadget graph $G_v.$
\cref{prop:single gadget control} implies that:
    \begin{align*}
        \frac{\mu_{\mathcal{H}}(\tilde{\sigma}_{\mathcal{S}};Y )}{Z_{\hat{\mathcal{H}}}( Y) } &= \frac{\mu_{\hat{\mathcal{H}}} (\tilde{\sigma}_{\mathcal{S}}; Y) }{Z_{\hat{\mathcal{H}}}( Y) } \cdot \prod_{{\set{i,j}\in E(\mathcal{H})\setminus E(\hat{\mathcal{H}})} } \exp(\tilde{J}_{ij}  \tilde{\sigma}_i \tilde{\sigma}_j) 
        \\
        & = \prod_{v\in V_H} \Pr_{x\sim \nu_v} [x_{S_v} = \tilde{\sigma}_{S_v}  | y_{R_v}(x) = Y_{v}  ]\prod_{\substack{\set{u,v} \in E_H\\ (i,j) \in E(S_u^v,S_v^u)}}\exp(\tilde{J}_{uv} \tilde{\sigma}_i \tilde{\sigma}_j)  \\
        &= (1\pm \epsilon_0)^m\left( \prod_{v\in V_H} \prod_{i\in S_v} Q_{i}^{ Y_v} (\tilde{\sigma}_{i})\right)  \left( \prod_{\substack{\set{u,v} \in E_H\\ (i,j) \in E(S_u^v,S_v^u)}} \exp(\tilde{J}_{uv} \tilde{\sigma}_i \tilde{\sigma}_j)\right) 
    \end{align*}
    This finishes the proof of \cref{eq:single config condition on phase}.
Next, we show \cref{eq:phase partition function approximate Ising model}. 

Combining \cref{eq:single config condition on phase} with the fact that
$Z_{\mathcal{H} }(Y) = \sum_{\tilde{\sigma}_{\mathcal{S}} \in\set{\pm 1}^{\mathcal{S}} } \mu_{\mathcal{H}}(\tilde{\sigma}_{\mathcal{S}}; Y), $
we have:
\begin{equation}\label{eq:edge weight of phase}
\begin{split}
    &\frac{Z_{\mathcal{H} }(Y)}{ Z_{\hat{\mathcal{H}}}(Y)} \\
    &= (1\pm \epsilon_0)^m \sum_{\tilde{\sigma}_{\mathcal{S}}} \left( \prod_{v\in V_H} \prod_{i\in S_v} Q_{i}^{ Y_v} (\tilde{\sigma}_{i})\right)  \prod_{\substack{\set{u,v} \in E_H\\ (i,j) \in E(S_u^v,S_v^u)}} \exp(\tilde{J}_{uv} \tilde{\sigma}_i \tilde{\sigma}_j) \\
    &= (1\pm \epsilon_0)^m  \left( \prod_{v\in V_H} \prod_{i\in S^{\text{vertex}}_v} (Q_{i}^{Y_v} (1) + Q_{i}^{Y_v} (-1) )\right)  \prod_{\substack{\set{u,v} \in E_H\\ (i,j) \in E(S_u^v,S_v^u)}} \left( \sum_{\tilde{\sigma}_i, \tilde{\sigma}_j} Q_i^{Y_u} (\tilde{\sigma}_i) Q_j^{Y_v} (\tilde{\sigma}_j) \exp(\tilde{J}_{uv} \tilde{\sigma}_i \tilde{\sigma}_j)\right)\\
    &=_{(1)} (1\pm \epsilon_0)^m  \prod_{\set{u,v} \in E_H} \left( \sum_{a,b \in \set{\pm 1}} q^{aY_u} q^{bY_v} \exp(\tilde{J}_{uv}ab)\right)^{e_{uv}}
    \end{split}
\end{equation}

In (1) we use the fact that $ Q_i^{Y_v} (1) + Q_i^{Y_v}(-1) =1\forall i,$ and that for $i\in S_{u}^v, j \in S_{v}^u,$ $ \tilde{h}_i=\tilde{h}_j= 0$ so $ Q_i^{Y_u} (a) = q^{a Y_u}$ and  $ Q_j^{Y_v} (b) = q^{b Y_v},$ where $q^{+1}\equiv q^+,q^{-1} \equiv q^-$ are as defined in \Cref{subsec:notation for Ising embedding}. Indeed, if $Y_u = +1$ and $ a=+1$ then by \cref{eq:define vertex bias} 
\[ Q_i^{Y_u}(a) = \frac{\exp( \beta(2q^+_\beta-1)) }{\exp( \beta(2q^+_\beta-1)) + \exp( -\beta(2q^+_\beta-1))} = \frac{\exp(2 \beta(2q^+_\beta-1)) }{\exp(2 \beta(2q^+_\beta-1)) + 1} = \frac{q^+_\beta/(1-q^+_\beta) }{ q^+_\beta/(1-q^+_\beta) + 1} = q^+_\beta\]
where the third equality is due to \cref{eq:def of q+ and q-}. The other cases follow from a similar argument.

Let
\begin{equation}
    \begin{split}
        \Psi_{+, uv} &= ( ((q^+)^2 + (q^-)^2 )\exp(\tilde{J}_{uv}) + 2 q^+ q^- \exp(-\tilde{J}_{uv}) )^{e_{uv}}\\ 
        \Psi_{-, uv} &= (2 q^+ q^-  \exp(\tilde{J}_{uv}) + ((q^+)^2 + (q^-)^2 ) \exp(-\tilde{J}_{uv}) )^{e_{uv}}
    \end{split}
    \end{equation}
    then by \cref{eq:define weight map function} and $ \tilde{J}_{uv} = \varphi_{\mathrm{edge},\beta}^{-1} (\frac{J_{uv}}{e_{uv}}), $ we have:
    \[ \left(\frac{ \Psi_{+, uv} }{ \Psi_{-, uv} }\right)^{1/2} = \exp ( e_{uv}  \varphi_{\mathrm{edge},\beta} (\tilde{J}_{uv}) ) = \exp(J_{uv}).\]
Let
\[ A_1 = \prod_{\set{u,v}\in E_H}   (\Psi_{+, uv} \Psi_{-, uv})^{1/2}\]
then 
\[ \frac{Z_{\mathcal{H} }(Y)}{ Z_{\hat{\mathcal{H}}}(Y)} = (1\pm \epsilon_0)^m A_1  \prod_{\set{u,v} \in E_H} \left(\frac{\Psi_{Y_u Y_v, uv}  }{\Psi_{-Y_u Y_v, uv}}\right)^{1/2} = (1\pm \epsilon_0)^m A_1\exp(\sum_{\set{u,v} \in E_H} J_{uv} Y_u 
Y_v) \]
and $A_1$ can be computed in $O(m^2)$ time.
Let $A_0$ be the constant in \cref{cor:probability of configuration condition on phase in hat H}, and $ A = A_0 A_1$ then
\[ Z_{\hat{\mathcal{H} }}(Y) = (1\pm \epsilon_0)^{m} A_0 \exp(\sum_{v\in V_H} h_v Y_v) \]
and 
\[Z_{\mathcal{H} }(Y) = (1\pm \epsilon_0)^{2m}  A \exp(\sum_{\set{u,v} \in E_H} J_{uv} Y_u Y_v + \sum_{v\in V_H} h_v Y_v)  = (1\pm \epsilon_0)^{2m} \, A\, \mu_H(Y)\]
and $A $ can be computed in $O(|V_{\mathcal{H}}| + m^2)$ time.

Next,
\[\Pr_{\sigma\sim\mu_{\mathcal{H}}}[Y(\sigma) = Y] = \frac{Z_{\mathcal{H} }(Y)}{\sum_{Y'} Z_{\mathcal{H} }(Y')  } = (1\pm \epsilon_0)^{4m} \frac{\mu_H(Y)}{\sum_{Y'}\mu_H(Y') } =  (1\pm \epsilon_0)^{4m} \Pr_{Y\sim \mu_H}[Y]  \]

For $ v\in V_H, i \in S_v,$ let $\mathcal{C}_i(\tilde{\sigma}_{\mathcal{T}}) =\begin{cases} \set{\pm 1} &\text{ if } i\not \in \mathcal{T} \\ \set{\tilde{\sigma}_i} &\text{ if } i\in \mathcal{T} \end{cases}.$

    By \cref{eq:single config condition on phase}, we have:
    \begin{align*}
         &\frac{\mu_{\mathcal{H}} (\tilde{\sigma}_\mathcal{T}; Y)}{Z_{\hat{\mathcal{H}}}(Y)}\\
         &= (1\pm \epsilon_0)^m \sum_{\sigma:\sigma_\mathcal{T} = \tilde{\sigma}_\mathcal{T} } \left( \prod_{v\in V_H} \prod_{i\in S_v} Q_{i}^{ Y_v} (\sigma_{i})\right)   \prod_{\substack{\set{u,v} \in E_H\\ (i,j) \in E(S_u^v,S_v^u)}} \exp(\tilde{J}_{uv} \sigma_i \sigma_j) \\
         &=(1\pm \epsilon_0)^m \prod_{v\in V_H} \prod_{i\in S^{\text{vertex}}_v} \left(\sum_{\sigma_i \in \mathcal{C}_i(\tilde{\sigma}_{\mathcal{T}}) }  Q_i^{Y_v} (\sigma_i) \right) \prod_{\substack{\set{u,v} \in E_H\\ (i,j) \in E(S_u^v,S_v^u)}} \left( \sum_{\substack{\sigma_i\in \mathcal{C}_{i}(\tilde{\sigma}_{\mathcal{T}})  \\
         \sigma_j\in\mathcal{C}_{j}(\tilde{\sigma}_{\mathcal{T}})}  } Q_i^{Y_u} (\sigma_i) Q_j^{Y_v} (\sigma_j) \exp(\tilde{J}_{uv} \sigma_i \sigma_j)\right)
    \end{align*}
    Let 
    \[\hat{\gamma} (\tilde{\sigma}_\mathcal{T}; Y) := \prod_{v\in V_H} \prod_{i\in S^{\text{vertex}}_v} \left(\sum_{\sigma_i \in \mathcal{C}_i(\tilde{\sigma}_{\mathcal{T}}) }  Q_i^{Y_v} (\sigma_i) \right) \prod_{\substack{\set{u,v} \in E_H\\ (i,j) \in E(S_u^v,S_v^u)}} \left( \sum_{\substack{\sigma_i\in \mathcal{C}_{i}(\tilde{\sigma}_{\mathcal{T}})  \\
        \sigma_j\in\mathcal{C}_{j}(\tilde{\sigma}_{\mathcal{T}})}  } Q_i^{Y_u} (\sigma_i) Q_j^{Y_v} (\sigma_j) \exp(\tilde{J}_{uv} \sigma_i \sigma_j)\right).\] 
        $\hat{\gamma} (\tilde{\sigma}_\mathcal{T}; Y) $ approximates $\gamma (\tilde{\sigma}_\mathcal{T}; Y)  $
    within $(1\pm \epsilon_0)^m$ factor, and can be computed in $O(|\mathcal{S}|^2)$ time, as it is a product of $ O(|\mathcal{S}|^2)$ terms, where each term can be computed in $O(1)$ time.
    Setting $ \epsilon_0 = \frac{\epsilon}{5m}$ finishes the proof.

\end{proof}

\begin{proof}[Proof of \cref{thm:set parameter choice}, parts (1) and (2)]
We assume without loss of generality that $\gamma \leq 2$ and thus $ \gamma = \widetilde{\gamma},$ otherwise we can replace $\gamma$ with $\widetilde{\gamma}.$
We first note that the construction is well-defined, since
\begin{align*}
  \left|\tanh \left(\frac{h_v}{t_v^0
}\right) \right|\leq \left|\frac{h_v}{t_v^0
}\right| \leq \frac{2 q^+_\beta-1}{2} &\Rightarrow \tilde{h}_v =  \varphi_{\mathrm{vertex},\beta}^{-1}\left(\frac{h_v}{t_v^0}\right) \text{ well-defined}\\
  \left|\tanh \left(\frac{J_{uv}}{e_{uv}}\right) \right|\leq \left|\frac{J_{uv}}{e_{uv}
}\right| \leq \frac{(2 q^+_\beta-1)^2}{2} &\Rightarrow \tilde{J}_{uv} =  \varphi_{\mathrm{edge},\beta}^{-1}\left(\frac{J_{uv}}{e_{uv}}\right) \text{ well-defined}
\end{align*}

    Recall that $\hat{J}$ is the interaction matrix that corresponds to $\hat{\mathcal{H}},$ then $ \hat{J}$ is a block matrix consisting of blocks $\mathbf{J}_v$ for $v\in V_H,$ where $\mathbf{J}_v = \frac{\beta}{r_v}( \mathbf{1}_{W_v} \mathbf{1}_{W_v}^\intercal -  \mathbf{1}_{S_v}   \mathbf{1}_{S_v}^\intercal) .$ It is easy to see that $\mathbf{J}_v$ has \[\beta \leq \lambda_{\max}(\mathbf{J}_v)-\lambda_{\min} (\mathbf{J}_v) \leq \beta (1+\frac{2t_v}{r_v}).\] Hence, for $ r_v \geq \frac{8\beta t_v}{ (\gamma-1)}\forall v \in V_H,$
    \[\beta \leq \lambda_{\max}(\hat{J}) -\lambda_{\min}(\hat{J}) \leq \beta + \frac{\gamma-1}{4}.\]
    The matrix $\tilde{J}-\hat{J}$ is a symmetric matrix that has at most one nonzero entry in each row, and the absolute value of this nonzero entry is bounded by \[ \max_{u,v} |\tilde{J}_{uv}|\leq \tanh^{-1} \left(\frac{(2 q^+_\beta -1)^2\min \set{\tanh(\frac{\gamma-1}{8}),0.5} }{(2 q^+_\beta -1)^2 }\right) \leq \min\set{\frac{\gamma-1}{8},\log 3},\] 
    where we use the fact that $ \tanh^{-1}$ is an increasing function. Similarly,
    \[\norm{\tilde{h}}_{\infty}= \max_{v} |\tilde{h}_v|\leq \tanh^{-1} \left(\frac{(2 q^+_\beta -1)0.5 }{(2 q^+_\beta -1) }\right) \leq\tanh^{-1}(0.5)\leq \log 3. \]
    Hence,  $\norm{\tilde{J}-\hat{J}}_{\mathrm{op}} \leq \frac{\gamma-1}{8}$ and by Weyl's eigenvalue perturbation bound,
    \[ 1 < \lambda_{\max}(\hat{J}) -\lambda_{\min}(\hat{J}) - 2\norm{\tilde{J}-\hat{J}}_{\mathrm{op}} \leq  \lambda_{\max}(\tilde{J}) -\lambda_{\min}(\tilde{J})\leq \lambda_{\max}(\hat{J}) -\lambda_{\min}(\hat{J}) + 2\norm{\tilde{J}-\hat{J}}_{\mathrm{op}} \leq \gamma. \]
Next, note that $\norm{\hat{J}}_{\infty}\leq \max_{v} \norm{\mathbf{J}_v}_{\infty} \leq \beta (1+ \frac{t_v}{r_v}) \leq 4$ for $r_v \geq t_v,$  and
\[ \norm{\tilde{J}}_{\infty} + \norm{\tilde{h}}_{\infty}\leq \norm{\hat{J}}_{\infty} +\max_{u,v} |\tilde{J}_{uv}| +  \norm{\tilde{h}}_{\infty}= O(1).\]
Clearly, for each $v\in V_H,$ $ t_v= O_{\gamma} (W).$
Let $ c_\beta >0$ be the constant that depends only on $\beta$ in \cref{prop:single gadget control}.
For $v\in V_H,$ set $r_v$ be the minimum odd number so that
\[r_v \geq \max\set{30 c_\beta t_v^3 m^2 \epsilon^{-2} \log(m/\epsilon) , \frac{8\beta t_v}{ (\gamma-1)} , t_v} .\]
then \[r_v = O_{\gamma}(t_v^3 m^2 \epsilon^{-2} \log(m/\epsilon) ) .\] 
The number of vertices in $\mathcal{H}$ is
\[ N= O\left(\sum_{v\in V_H} (r_v+t_v) \right) = O_{\gamma}\left(m^3 W^3\cdot \epsilon^{-2} \log(m/\epsilon)   \right) \]
Part (1) follows from the above analysis, and part (2) of \Cref{thm:set parameter choice} follows from \cref{lem:main construction probability of phase gadget}.
\end{proof}
\begin{remark}\label{remark:set zero diagonal entries}
The constructed matrix $\tilde{J}$
has nonzero diagonal entries. If an interaction matrix with zero diagonal entries is desired, we can consider the matrix obtained from $\tilde{J}$ by setting the diagonal to $0,$  which we denote by  $\tilde{J}^{o}.$ Replacing $\tilde{J}$ with $\tilde{J}^{o}$ does not change the distribution $\tilde{\mu},$ and only changes the normalization factor $A$ by a factor computable from $\tilde{J}$. Clearly, $\norm{\tilde{J}^{o}}_{\infty}\leq \norm{\tilde{J}}_{\infty},$ and
$ \norm{\tilde{J} - \tilde{J}^{o}}\leq \max_{i\in V_{\mathcal{H}}} |\tilde{J}_{ii}| \leq\max_{v\in V_H} \frac{\beta}{ r_v}.$ Assume $\gamma=\tilde{\gamma} \leq 2$ for simplicity.
The
 spectral width of $ \tilde{J}^{o}$ satisfies
 \[ \lambda_{\max}(\tilde{J}^{o}) -\lambda_{\min}(\tilde{J}^{o}) \leq \lambda_{\max}(\tilde{J}) -\lambda_{\min}(\tilde{J}) +  2\max_{v\in V_H} \frac{\beta}{ r_v} \leq \beta +2\max_{v\in V_H} \frac{\beta (1+ t_v)}{ r_v}  + \frac{\gamma-1}{4} \]
 and 
 \[ \frac{\gamma+ 3}{4} - 2\max_{v\in V_H} \frac{\beta}{ r_v}  \leq \lambda_{\max}(\tilde{J}^{o}) -\lambda_{\min}(\tilde{J}^{o})\]
By choosing $r_v$ sufficiently large, i.e., let $c'_\gamma >0$ be a constant dependent only on $\gamma$ so that $ t_v\leq c'_{\gamma} W\forall v\in V_H$, and set \[\forall v\in V_H: r_v = \min\set{r \in \N, r \text{ is odd}\mid  r\geq\max\set{ 30 c_{\beta }  (c'_\gamma W)^3 m^2 \epsilon^{-2} \log(m/\epsilon), \frac{ 16\beta (c'_\gamma W+1)}{\gamma -1} , c'_\gamma W} },\] 
we can ensure $  1< \lambda_{\max}(\tilde{J}^{o}) -\lambda_{\min}(\tilde{J}^{o}) \leq\gamma$ and $|V_{\mathcal{H}}| = O_{\gamma}\left(m^3 W^3\cdot \epsilon^{-2} \log(m/\epsilon)   \right).$ 
\end{remark}

\subsection{Proof of \cref{thm:set parameter choice},  part (3)} \label{subsec:proof of relating hardness of original to near critical Ising model}
Let $\Omega = \set{\pm 1}^{V_H}, \tilde{\Omega} =\set{\pm 1}^{V_{\mathcal{H}}}.$ Let $\mu \equiv\mu_{H}$ and $\tilde{\mu} \equiv \mu_{\mathcal{H}}$ be the unnormalized density of the Ising models corresponding to $H$ and $ \mathcal{H}$ respectively.
Let $Z_{\mathcal{H}} =\sum_{y\in \Omega} Z_{\mathcal{H}} (y) =\sum_{\sigma\in \tilde{\Omega}} \tilde{\mu}(\sigma)$ and $ Z_H = \sum_{y\in \Omega} \mu(y)$ be their partition functions.

Let $\Lambda \subseteq \Omega$ be s.t. $ \Pr_{y\sim \mu}[y\in \Lambda] \leq \zeta,$ and no algorithm can $ (\eta, k,t,\Psi,
\Lambda)$-learn-to-generalize $ \mu.$ Let $ \tilde{\Lambda} = Y^{-1}(\Lambda), $ then by \cref{thm:set parameter choice} part (2), we have \[ \Pr_{\sigma \sim \tilde{\mu}}[\sigma \in \tilde{\Lambda}]  \leq (1+\epsilon) \Pr_{y\sim \mu}[y\in \Lambda]  \leq (1+\epsilon)  \zeta. \]
Hence, we only need to show that no algorithm can $ (\tilde{\eta}, \tilde{k}, \tilde{t}, \tilde{\Psi}, \tilde{\Lambda})$-learn-to-generalize $\tilde{\mu}.$

Suppose for contradiction that there exists an algorithm $\tilde{\mathcal{A}}$ that  $ (\tilde{\eta}, \tilde{k}, \tilde{t}, \tilde{\Psi}, \tilde{\Lambda})$-learn-to-generalize $\tilde{\mu},$ that is: $\tilde{\mathcal{A}}$ takes the parameter $(\tilde{J},\tilde{h})$, a tuple $\Sigma= (\sigma^{(i)})_{i=1}^{\tilde{k}}\in \tilde{\Omega}^{\tilde{k}}$, runs in $\tilde{t}$ time, and its output distribution, denoted by $ \tilde{\mathcal{A}}_{\Sigma},$ satisfies:
\begin{equation}\label{eq:assumption on tilde A}
    \Pr [\Sigma \sim \tilde{\mu}^{\otimes \tilde{k}}, y \sim \tilde{\mathcal{A}}_{\Sigma}: y\not\in \tilde{\Psi}^{-1} (\tilde{\Psi}(\Sigma))\cup \tilde{\Lambda}  ] \geq \tilde{\eta}.
\end{equation}
We modify $\tilde{\mathcal{A}}$ so that it can handle input tuple $\Sigma= (\sigma^{(i)})_{i=1}^\ell$ of an arbitrary length $\ell$ as follows:
\begin{itemize}
\item If $ \ell \geq \tilde{k}$ then runs $ \tilde{\mathcal{A}}$ on the first $\tilde{k}$ elements of the tuple.
    \item If $ \ell <\tilde{k}$ then appends arbitrary elements of $\tilde{\Omega}$ to $\Sigma$ to create a tuple of length $\tilde{k},$ then runs $\tilde{\mathcal{A}}$ on the resulting tuple.
\end{itemize}

\subsubsection{Warm-up: An idealized setting}\label{subsubsection:warmup ideal setting for rejection sampling}
    First, we will handle an easier warm-up setting, in which for any given $Y \in \set{\pm 1}^{V_H}$ we can exactly compute $ Z_{\mathcal{H}} (Y) $ and exactly sample from $ \tilde{\mu} ( \cdot | Y(\sigma) = Y)$.
    
We construct an algorithm $\mathcal{A}$ that learns-to-generalize $\mu$. The algorithm takes as input the parameter $(J,h)$ and a tuple $U = (y^{(i)})_{i=1}^k \in \Omega^k$ and proceeds as follows:

\paragraph{Algorithm for idealized setting:}
Let $A $ be the normalization factor in \cref{lem:main construction probability of phase gadget}'s \cref{eq:phase partition function approximate Ising model}.
\begin{enumerate}
    \item  For each $i,$ with probability $ \frac{Z_{\mathcal{H}}(y^{(i)})}{(1+ \epsilon) \cdot A\cdot \mu(y^{(i)})}$, add $i$ to $\mathcal{I}_A.$
    \item For each $i$, sample $ \sigma^{(i)} \sim \tilde{\mu} ( \cdot | Y(\sigma) = y^{(i)})$ 
    \item Let $ \Sigma = (\sigma^{(i)})_{i\in \mathcal{I}_A}.$ Construct $(\tilde{J},\tilde{h})$ from $(J,h),$ run $\tilde{\mathcal{A}} $ on $(\tilde{J},\tilde{h})$ and $\Sigma,$ and obtain a sample $\tau\sim \tilde{\mathcal{A}}_{\Sigma},$ then output $x = Y(\tau).$
\end{enumerate}

We proceed to analyze this algorithm. For an input tuple $ U\in \Omega^k$ to $\mathcal{A}$ and $x \in \Omega,$ let $ \mathcal{A}_U(x)$ be the probability that $ \mathcal{A}_U $ outputs $x.$ Note that:
\begin{equation}
\begin{split}
  \Pr[y^{(i)}, i \in  \mathcal{I}_A ]  &=  \Pr_{\mu}[y^{(i)}] \cdot \Pr[i \in  \mathcal{I}_A |y^{(i)} ] = \frac{\mu(y^{(i)}) }{Z_H}\cdot \frac{Z_{\mathcal{H}}(y^{(i)})}{(1+ \epsilon) \cdot A\cdot \mu(y^{(i)})}  = \frac{Z_{\mathcal{H}} (y^{(i)}) } {(1+\epsilon) A \cdot Z_H}\\
  \Pr[i\in \mathcal{I}_A]   &=  \sum_{y^{(i)}}  \Pr[y^{(i)}, i \in  \mathcal{I}_A ]  =   
\frac{Z_{\mathcal{H}}  }{(1+\epsilon) A \cdot Z_H } \\
\Pr[y^{(i)} | i \in \mathcal{I}_A] &= \frac{\Pr[y^{(i)}, i \in  \mathcal{I}_A ] }{\Pr[i\in \mathcal{I}_A] } =\frac{Z_{\mathcal{H}} (y^{(i)} ) }{ Z_{\mathcal{H}} } 
\end{split}
\end{equation}

Let $p:= \Pr[i\not\in \mathcal{I}_A] = 1- \frac{Z_{\mathcal{H}}  }{(1+\epsilon) A \cdot Z_H }.$ 
For a realization $U = (y^{(i)})_{i=1}^k$, $ \mathcal{I}_A$ and $ \Sigma,$ we have: 
\begin{align*}
    &\Pr[U,\mathcal{I}_A , \Sigma  ]   \\
    &=  \prod_{i=1}^k \Pr_{\mu}[y^{(i)}] \left (\prod_{i\in \mathcal{I}_A} \Pr[ i\in \mathcal{I}_A | y^{(i)}] \Pr_{\sigma^{(i)} \sim \tilde{\mu} } [\sigma^{(i)} |Y(\sigma^{(i)}) = y^{(i)}]
    \right)\left(\prod_{i\in [k] \setminus \mathcal{I}_A} \Pr[ i\not\in \mathcal{I}_A | y^{(i)}]\right)\\
    &=  \left( \prod_{i\in \mathcal{I}_A}  \frac{Z_{\mathcal{H}} (y^{(i)} ) }{ (1+\epsilon) A\cdot Z_{H} } \cdot \frac{\tilde{\mu}(\sigma^{(i)}) }{Z_{\mathcal{H}}(y^{(i)}) } \right) \left(\prod_{i\in [k] \setminus \mathcal{I}_A } \Pr[y^{(i)}, i \not\in  \mathcal{I}_A ]  \right)\\
    &= (1-p)^{|\mathcal{I}_A|}\left(\prod_{i\in [k] \setminus \mathcal{I}_A } \Pr[y^{(i)}, i \not\in  \mathcal{I}_A ]  \right) \left( \prod_{i\in \mathcal{I}_A} \Pr_{\tilde{\mu}} [\sigma^{(i)}] \right) \\
    &= (1-p)^{|\mathcal{I}_A|}\left(\prod_{i\in [k] \setminus \mathcal{I}_A } R(y^{(i)})\right)\left( \prod_{i\in \mathcal{I}_A} \Pr_{\tilde{\mu}} [\sigma^{(i)}] \right) 
\end{align*}
where we let 
\[R(y^{(i)}) :=  \Pr[y^{(i)}, i \not\in  \mathcal{I}_A ] = \frac{\mu(y^{(i)}) }{Z_H } - \frac{Z_{\mathcal{H}} (y^{(i)}) } {(1+\epsilon) A \cdot Z_H} .\] Note that $R(y^{(i)})$ depends only on the value $y^{(i)}$ and is independent of the index $i.$

We note that in the algorithm above, $ x = Y(\tau) \not\in \Psi^{-1} (\Psi(U))$ iff $ \tilde{\Psi}(\tau) = \Psi \circ Y (\tau)\not\in \Psi(U),$ and $ x = Y(\tau) \not\in \Lambda$ iff $\tau \not\in Y^{-1}(\Lambda) = \tilde{\Lambda}.$
Hence, we can write
\allowdisplaybreaks
\begin{align*}
g(\mathcal{A}): &=\Pr[U\sim \mu^{\otimes k}, x\sim \mathcal{A}_U: x\not\in \Psi^{-1} (\Psi(U))\cup \Lambda ]\\
&=\Pr[U, \mathcal{I}_A, \Sigma, \tau \sim \tilde{\mathcal{A}}_{\Sigma}: \tilde{\Psi}(\tau)\not\in \Psi(U) \text{ and } \tau \not\in \tilde{\Lambda} ]\\
&= \sum_{U, \mathcal{I}_A, \Sigma} \Pr[U, \mathcal{I}_A, \Sigma] \sum_{\tau \in\tilde{\Omega}\setminus \tilde{\Lambda}: \tilde{\Psi}(\tau)\not\in \Psi(U) } \tilde{\mathcal{A}}_{\Sigma} (\tau)\\
&=\sum_{\substack{\mathcal{I}_A \subseteq [k]\\(y^{(i)})_{i} \in \Omega^{[k]\setminus \mathcal{I}_A}  \\(\sigma^{(i)})_{i} \in \tilde{\Omega}^{\mathcal{I}_A}  }}  (1-p)^{|\mathcal{I}_A|}  \left(\prod_{i\in [k] \setminus \mathcal{I}_A } R(y^{(i)})\right) \left( \prod_{i\in \mathcal{I}_A} \Pr_{\tilde{\mu}} [\sigma^{(i)}] \right) \sum_{\substack{\tau \in\tilde{\Omega}\setminus \tilde{\Lambda}\\ 
\tilde{\Psi}(\tau)\neq \Psi(y^{(i)})\forall i\not\in \mathcal{I}_A \\ \tilde{\Psi}(\tau) \neq  \tilde{\Psi}(\sigma^{(i)}) \forall i\in \mathcal{I}_A
}} \tilde{\mathcal{A}}_{(\sigma^{(i)})_{i\in \mathcal{I}_A}} (\tau)
\end{align*}

Suppose $\mathcal{I}_A = \set{i_1, \cdots, i_\ell}$ for   some $\ell \in\set{0,\cdots, k},$ and $i_1 < i_2 <\cdots < i_\ell.$ Let $ \tilde{\sigma}^{(r)} =\sigma^{(i_r)} $ for $ r\in [\ell].$ Similarly, write $[k] \setminus\mathcal{I}_A = \set{j_1, \cdots, j_{k-\ell}}  $ for $j_1 < j_2 < \cdots < j_{k-\ell} $ and let $\tilde{y}^{(s)}  = y^{(j_s)}.$ 
We note that each term in the summation above depends only on the $ \tilde{\sigma}^{(r)}$, $\tilde{y}^{(s)} $, and $\tau.$ 
By interchanging the order of summation, we can rewrite the above as:

\begin{align*}
g(\mathcal{A}) &= \sum_{\ell=0}^k \sum_{ (\tilde{\sigma}^{(r)})_{r=1}^{\ell}} 
\sum_{\substack{\tau \in\tilde{\Omega}\setminus \tilde{\Lambda} \\ 
\tilde{\Psi}(\tau) \neq  \tilde{\Psi}(\tilde{\sigma}^{(r)}) \, \forall r\in [\ell]
}}  \binom{k}{\ell} \left (\prod_{r=1}^\ell
\Pr_{\tilde{\mu}} [\tilde{\sigma}^{(r)}]\right)
\tilde{\mathcal{A}}_{(\tilde{\sigma}^{(r)})_{r=1}^{\ell}} (\tau)   (1-p)^\ell  \\
&\quad \times \left( \sum_{\substack{(\tilde{y}^{(s)})_s \in \Omega^{k-\ell} \\
\Psi(\tilde{y}^{(s)}) \neq  \tilde{\Psi}(\tau) \, \forall s\in [k-\ell] } }   
\prod_{s\in [k-\ell] } R(\tilde{y}^{(s)})\right)\\
&\geq \sum_{\ell=\tilde{k}}^k \sum_{ (\tilde{\sigma}^{(r)})_{r=1}^{\tilde{k}}} 
\sum_{\substack{\tau \in\tilde{\Omega}\setminus \tilde{\Lambda} \\ 
\tilde{\Psi}(\tau) \neq  \tilde{\Psi}(\tilde{\sigma}^{(r)}) \, \forall r\in [\tilde{k}]
}}  \binom{k}{\ell}    \left(\prod_{r=1}^{\tilde{k}}
\Pr_{\tilde{\mu}} [\tilde{\sigma}^{(r)}]\right)
\tilde{\mathcal{A}}_{(\tilde{\sigma}^{(r)})_{r=1}^{\tilde{k}}} (\tau)  (1-p)^\ell\\
&\quad \times \left( \sum_{ \substack{(\tilde{\sigma}^{(r)})_{r=\tilde{k}+1}^{\ell}\\\tilde{\Psi}(\tilde{\sigma}^{(r)}) \neq \tilde{\Psi}(\tau) \forall \tilde{k}< r \leq \ell}}  \prod_{r=\tilde{k}+1}^\ell  \Pr_{\tilde{\mu}} [\tilde{\sigma}^{(r)}] \right)\times \left( \sum_{\substack{(\tilde{y}^{(s)})_s \in \Omega^{k-\ell} \\
\Psi(\tilde{y}^{(s)}) \neq  \tilde{\Psi}(\tau) \, \forall s\in [k-\ell] } }   
\prod_{s\in [k-\ell] } R(\tilde{y}^{(s)})\right)
\end{align*}
where we use the fact for $\ell \geq \tilde{k},$ 
$\tilde{\mathcal{A}}_{(\tilde{\sigma}^{(r)})_{r=1}^\ell}$ is the same algorithm as $\tilde{\mathcal{A}}_{(\tilde{\sigma}^{(r)})_{r=1}^{\tilde{k}}}.$

Note that 
\begin{align*}
\sum_{ \substack{(\tilde{\sigma}^{(r)})_{r=\tilde{k}+1}^{\ell}\\\tilde{\Psi}(\tilde{\sigma}^{(r)}) \neq \tilde{\Psi}(\tau) \forall \tilde{k}< r \leq \ell}}  \prod_{r=\tilde{k}+1}^\ell \Pr_{\tilde{\mu}} [\tilde{\sigma}^{(r)}]  &=\left(\sum_{\hat{\sigma} \in \tilde{\Omega} \setminus \tilde{\Psi}^{-1}(\tilde{\Psi} (\tau)) } \Pr_{\tilde{\mu}}[\hat{\sigma}] \right)^{\ell-\tilde{k}}
\end{align*}
Using \cref{eq:gadget approximate original Ising formal (inside main reduction thm)} to upper bound  $\Pr_{\sigma\sim \tilde{\mu}}[Y(\sigma)=y]$ by $(1+\epsilon)\Pr_{y\sim \mu}[y],$ for each $\tau\in \tilde{\Omega},$ we have:
\begin{align*}
    \sum_{\hat{\sigma} \in \tilde{\Omega} \setminus \tilde{\Psi}^{-1}(\tilde{\Psi} (\tau)) } \Pr_{\tilde{\mu}}[\hat{\sigma}] &=  1- \Pr_{\sigma \sim \tilde{\mu}}[\tilde{\Psi}(\sigma) =\tilde{\Psi}(\tau) ] = 1- \Pr_{\sigma \sim \tilde{\mu}}[\Psi (Y(\sigma)) =\Psi(Y(\tau)) ]\\
&\geq 1- (1+\epsilon) \Pr_{y \sim \mu}[\Psi(y) =\Psi (Y(\tau)) ]\geq \exp(- 2(1+\epsilon)(\Psi_*\mu)_{\max} ), 
\end{align*}
where we use the inequality $ 1-a \geq\exp(-2a)$ for $a \in [0,1/2].$

Similarly,
\begin{align*}
   \sum_{\substack{(\tilde{y}^{(s)})_s \in \Omega^{k-\ell} \\
\Psi(\tilde{y}^{(s)}) \neq  \tilde{\Psi}(\tau) \, \forall s\in [k-\ell] } }   
\prod_{s\in [k-\ell] }
R(\tilde{y}^{(s)}) &= \left(\sum_{y \in \Omega \setminus \Psi^{-1}(\tilde{\Psi} (\tau)) } R(y) \right)^{k-\ell}
\end{align*}
Hence, we have:
\begin{align*}
    g(\mathcal{A}) &\geq  \exp(-2 k (1+\epsilon)(\Psi_*\mu)_{\max} )  \sum_{ (\tilde{\sigma}^{(r)})_{r=1}^{\tilde{k}}} 
\sum_{\substack{\tau \in\tilde{\Omega}\setminus \tilde{\Lambda} \\ 
\tilde{\Psi}(\tau) \neq  \tilde{\Psi}(\tilde{\sigma}^{(r)}) \, \forall r\in [\tilde{k}]
}}  \left(\prod_{r=1}^{\tilde{k}}
\Pr_{\tilde{\mu}} [\tilde{\sigma}^{(r)}]\right)
\tilde{\mathcal{A}}_{(\tilde{\sigma}^{(r)})_{r=1}^{\tilde{k}}} (\tau)  \\&\quad \times \left[\sum_{\ell=\tilde{k}}^k \binom{k}{\ell}  (1-p)^\ell  \left(\sum_{y \in \Omega \setminus \Psi^{-1}(\tilde{\Psi} (\tau)) } R(y) \right)^{k-\ell} \right]\\
&\geq_{(1)}\frac{7}{16}\times \exp(-4k (1+\epsilon)(\Psi_*\mu)_{\max} )  \sum_{ (\tilde{\sigma}^{(r)})_{r=1}^{\tilde{k}}} 
\sum_{\substack{\tau \in\tilde{\Omega}\setminus \tilde{\Lambda} \\ 
\tilde{\Psi}(\tau) \neq  \tilde{\Psi}(\tilde{\sigma}^{(r)}) \, \forall r\in [\tilde{k}]
}}  \left(\prod_{r=1}^{\tilde{k}}
\Pr_{\tilde{\mu}} [\tilde{\sigma}^{(r)}]\right)
\tilde{\mathcal{A}}_{(\tilde{\sigma}^{(r)})_{r=1}^{\tilde{k}}} (\tau)  
\\
&= \frac{7}{16}\times \exp(-4k (1+\epsilon)(\Psi_*\mu)_{\max} ) \Pr[\Sigma\sim\tilde{\mu}^{\otimes \tilde{k}}, \tau\sim \tilde{\mathcal{A}}_\Sigma: \tau\not\in \tilde{\Psi}^{-1} (\tilde{\Psi}(\Sigma))\cup \tilde{\Lambda} ]\\
&\geq \frac{7\tilde{\eta}}{16}\times \exp(-5k \cdot (\Psi_*\mu)_{\max} ) \\
&\geq \eta
\end{align*}
where in the penultimate inequality we use the assumption on $\tilde{\mathcal{A}}$ from \cref{eq:assumption on tilde A}, and in the last inequality we use the assumption on $ \tilde{\eta}.$

The inequality labeled as $\geq_{(1)}$ is due to the following standard concentration inequality for binomial random variables.

\begin{proposition}\label{prop:helper chernoff bound}
     Consider $p, q \in [0,1]$ where $1-p+q\leq 1.$ Let $\tilde{k} = \lceil (k+1)/2 \rceil ,$ and suppose $ p < 0.1$ then:
     \[ \sum_{\ell=\tilde{k}}^k \binom{k}{\ell}  (1-p)^\ell q^{k-\ell} \geq \frac{7}{16}\times (1-p+q)^k\]
 \end{proposition}
 \begin{proof}
    Let \[ F: = (1-p+q)^{-k}\sum_{\ell=\tilde{k}}^k \binom{k}{\ell}  (1-p)^\ell q^{k-\ell} = \sum_{\ell = \tilde{k}}^k \binom{k}{\ell}  (1-\hat{p})^\ell \hat{p}^{k-\ell} \] for $\hat{p} = 1- \frac{1-p}{1-p+q} \leq p.$
    This is precisely the probability that a binomial random variable $Z\sim \text{Binomial}(k,1-\hat{p})$ satisfies $Z \geq \tilde{k}.$ Note that $ \E[Z] = k(1-\hat{p}) $ and $\Var(Z) = k (1-\hat{p})\hat{p}.$ Let $t = \E[Z]-k/2 = k(1/2-\hat{p}).$ The Chebyshev inequality yields:
    \[\Pr[Z \leq k/2]\leq \Pr[|Z- \E[Z]| \geq t ]\leq \frac{\Var(Z)}{t^2} \leq  \frac{\hat{p}(1-\hat{p})}{(1/2 -\hat{p})^2 k } \leq \frac{9}{16k}\leq \frac{9}{16}.\]
    Hence $F \geq \Pr[Z>k/2] \geq \frac{7}{16}.$
   We note that the bound improves for large $k,$ but the basic bound when $k=1$ suffices.
 \end{proof}
Apply \cref{prop:helper chernoff bound} for $p= 1- \frac{Z_{\mathcal{H}}  }{(1+\epsilon) A \cdot Z_H }$ and $q=\sum_{y \in \Omega \setminus \Psi^{-1}(\tilde{\Psi} (\tau)) } R(y) $ and note that
 \[ p = 1- \frac{Z_{\mathcal{H}}  }{(1+\epsilon) A \cdot Z_H }\leq 1- \frac{1-\epsilon}{1+\epsilon} < 2\epsilon \leq 0.1,\]
 and since $ \sum_{y\in\Omega} R(y) =\Pr[i\not\in \mathcal{I}_A]= p,$
\begin{align*}
    1- p + \sum_{y \in \Omega \setminus \Psi^{-1}(\tilde{\Psi} (\tau)) } R(y)  &= 1-p + \sum_{y\in\Omega} R(y) - \sum_{y \in \Psi^{-1}(\tilde{\Psi} (\tau)) } R(y) = 1 - \sum_{y \in \Psi^{-1}(\tilde{\Psi} (\tau)) }
\Pr_{\mu}[y, i\not\in \mathcal{I}_A]
\end{align*}
Clearly, $ 1- p + q =  1 - \sum_{y \in \Psi^{-1}(\tilde{\Psi} (\tau)) }
\Pr_{\mu}[y, i\not\in \mathcal{I}_A] \leq 1.$
Moreover:
\begin{align*}
1-p+q&\geq 1 - \sum_{y \in \Psi^{-1}(\tilde{\Psi} (\tau)) }
\Pr_{\mu}[y]= 1- \Pr_{y\sim \mu}[\Psi(y) = \Psi(Y(\tau)) ]  \geq \exp(- 2(\Psi_*\mu)_{\max} )
\end{align*}
This finishes the proof in the idealized setting.
\subsubsection{Real setting: handling approximation errors}
Next, we show how to adapt the proof for the actual setting, where we can only approximately compute $ Z_{\mathcal{H}}(y)$ and only approximately sample from $\tilde{\mu}(\cdot |Y(\sigma) = y).$ The reason we have access to the approximate partition function and samples are the following two propositions---the proofs of which we leave to the end of the subsection.

\begin{proposition}\label{prop:sample from conditional distribution on phase}
 Given an accuracy parameter $\delta >0, $ there exists an algorithm 
 that takes as input $y\in \set{\pm 1}^{V_H},$ runs in $ O( |\mathcal{S}|^3 (|\mathcal{S}|^2 +\sum_{v\in V_H}r_v^2) \log(|\mathcal{S}|/\delta) \log (1/\delta))$ time and outputs a sample from a distribution that is $\delta$-close to $\tilde{\mu} (\cdot | Y(\sigma) = y)$ in total variation distance.
\end{proposition}
\begin{proposition}\label{prop:estimate phase probability}
    Given accuracy parameters $\varepsilon, \delta \in(0,1), $ there exists an algorithm 
    that takes as input $y\in \set{\pm 1}^{V_H},$  runs in $O( \varepsilon^{-2} |\mathcal{S}|^5  (|\mathcal{S}|^2 +\sum_{v\in V_H}r_v^2) \log(|\mathcal{S}|/\varepsilon) \log (1/\varepsilon) \log (1/\delta))$ time and outputs $\hat{Z}_{\mathcal{H}}(y)$ where
   \[ \forall y: \Pr[\hat{Z}_{\mathcal{H}}(y) = (1\pm \varepsilon) Z_{\mathcal{H}}(y)] \geq 1 -\delta \]
   where the probability is taken over the randomness of the algorithm.
\end{proposition}

First, we show how to modify the construction of $\mathcal{I}_A$ and $\Sigma $ given $U = (y^{(i)})_{i=1}^k.$

\paragraph{Algorithm for real setting:}
Let $\varepsilon, \delta\in(0,1)$ be parameters to be fixed later. Let $A $ be the normalization factor in \cref{lem:main construction probability of phase gadget}'s \cref{eq:phase partition function approximate Ising model}, which can be computed in $O(N)$ time where $N = |V_{\mathcal{H}}|$. 
\begin{enumerate}
    \item 
   For each $i$, if  the input $y^{(i)}$ has not been previously seen, run 
   the algorithm from \cref{prop:estimate phase probability} on accuracy parameters $\varepsilon,\delta$ to obtain $\hat{Z}_{\mathcal{H}}(y^{(i)})$, and store $ \hat{Z}_{\mathcal{H}}(y^{(i)})$; if $y^{(i)}$ has been previously seen then retrieve the value $\hat{Z}_{\mathcal{H}}(y^{(i)})$ from storage.  
   
   Let $\hat{\rho}_{y^{(i)}}$ be the distribution of $\hat{Z}_{\mathcal{H}}(y^{(i)}).$  Since $\Pr[\hat{Z}_{\mathcal{H}}(y^{(i)}) = (1\pm \varepsilon) Z_{\mathcal{H}}(y^{(i)})] \geq 1-\delta$; in other words, there exists a distribution $\rho_{y^{(i)}}$ supported on $ [ (1-\varepsilon)Z_{\mathcal{H}}(y^{(i)}), (1+\varepsilon)Z_{\mathcal{H}}(y^{(i)})] $ and a coupling $\Pi_{0,y^{(i)}}$ between $\hat{\rho}_{y^{(i)}} $ and $\rho_{y^{(i)}}$ so that
    \[\Pr_{(\hat{Z}_{\mathcal{H}}(y^{(i)}), \hat{Z}'_{\mathcal{H}}(y^{(i)}))\sim\Pi_{0,y^{(i)}}}[\hat{Z}_{\mathcal{H}}(y^{(i)})\neq \hat{Z}'_{\mathcal{H}} (y^{(i)}) ]\leq \delta . \]

As for $ \hat{Z}_{\mathcal{H}}(y^{(i)}),$ sample $ \hat{Z}'_{\mathcal{H}}(y^{(i)})$ using the coupling and store for later use if $y^{(i)}$ has not been previously seen, otherwise retrieve $ \hat{Z}'_{\mathcal{H}}(y^{(i)})$ from storage.
      
    With probability $ \min \left\{\frac{\hat{Z}_{\mathcal{H}}(y^{(i)}) }{(1+\varepsilon)(1+ \epsilon) A\cdot \mu(y^{(i)})}, 1\right\}$, add $i$ to $\mathcal{I}_A.$ 
    With probability $ \frac{\hat{Z}'_{\mathcal{H}}(y^{(i)}) }{(1+\varepsilon)(1+ \epsilon) A\cdot \mu(y^{(i)})}$ (which is always at most 1, since $ \hat{Z}'_{\mathcal{H}}(y^{(i)})\sim \rho_{y^{(i)}}$ and $\supp(\rho_{y^{(i)}})\subseteq  [ (1-\varepsilon)Z_{\mathcal{H}}(y^{(i)}), (1+\varepsilon)Z_{\mathcal{H}}(y^{(i)})] $), add $i$ to $\mathcal{I}'_A.$ 
    \item For each $i$, run 
    the algorithm from \cref{prop:sample from conditional distribution on phase} on input $ y^{(i)}$ and accuracy parameter  $\delta$ to obtain a sample $ \hat{\sigma}^{(i)}$ from some distribution $\nu_{y^{(i)}}$ where \[d_{TV}(\nu_{y^{(i)}}, \tilde{\mu} ( \cdot | Y(\sigma) = y^{(i)})) \leq \delta.\]
    This TV-closeness implies the existence of  a coupling $ \Pi_{1,y^{(i)}}$ between  $\nu_{y^{(i)}}$ and $\tilde{\mu} ( \cdot | Y(\sigma) = y^{(i)})$  so that 
    $\Pr_{(\hat{\sigma}^{(i)}, \hat{\sigma}^{'(i)})\sim\Pi_{1,y^{(i)}}}[\hat{\sigma}^{(i)} \neq \hat{\sigma}^{'(i)}]\leq \delta.$ 
    \item Let $ \Sigma = (\hat{\sigma}^{(i)})_{i\in \mathcal{I}_A}.$ Similarly, let $ \Sigma' = (\hat{\sigma}^{'(i)})_{i\in \mathcal{I}'_A}.$
     Sample $\tau\sim \tilde{\mathcal{A}}_{\Sigma},$ then output $Y(\tau).$
\end{enumerate}
Note that $ \mathcal{I}'_A, \Sigma'$ are not part of the algorithm $\mathcal{A},$ but they are objects useful for our analysis.
\paragraph{Analysis:}
Using the couplings $ \Pi_{0,y^{(i)}}$ and $ \Pi_{1,y^{(i)}}$, we can build a grand coupling $\Pi$ between the distribution of $(\mathcal{I}_A, \Sigma)$ and that of $(\mathcal{I}'_A, \Sigma'),$ so that the coupling succeeds (i.e. $ (\mathcal{I}_A, \Sigma) = (\mathcal{I}'_A, \Sigma')$) with probability $ \geq 1-2\delta k.$


Suppose that the coupling succeeds, we can replace $\mathcal{I}_A$ with $ \mathcal{I}'_A$ and $ \Sigma$ with $\Sigma',$ and proceed similar to the ideal setting.

More concretely, using the same argument as in the ideal case, we have:
\begin{equation}\label{eq:relate between actual and ideal case via coupling}
    \begin{split}
       &\Pr[U\sim \mu^{\otimes k}, x\sim \mathcal{A}_U: x\not\in \Psi^{-1} (\Psi(U))\cup \Lambda ]\\= & \Pr[U, \mathcal{I}_A, \Sigma, \tau \sim \tilde{\mathcal{A}}_{\Sigma}: \tilde{\Psi}(\tau)\not\in \Psi(U) \text{ and } \tau \not\in \tilde{\Lambda} ] \\
       \geq&
\Pr[U, \mathcal{I}'_A, \Sigma', \tau \sim \tilde{\mathcal{A}}_{\Sigma'}: \tilde{\Psi}(\tau)\not\in \Psi(U) \text{ and } \tau \not\in \tilde{\Lambda} ] -2 \delta k 
    \end{split}
\end{equation}

Recall the definition of the distribution $ \rho_{y}$ above, and that $ \hat{Z}'_{\mathcal{H}} (y)$ is a sample from $ \rho_{y}.$ We fix an arbitrary realization of $  \hat{Z}'_{\mathcal{H}} (y)$ for every $y\in \set{\pm 1}^{V_H},$ and prove that for this realization, \[ 
\Pr[U\sim \mu^{\otimes k}, x\sim \mathcal{A}_U: x\not\in \Psi^{-1} (\Psi(U))\cup \Lambda ] \geq \eta.\]

Since this is true for an arbitrary realization, it is also true for $\hat{Z}'_{\mathcal{H}} (y)$ that are sampled from $ \rho_{y},$ thus we have shown
\[ \Pr[U\sim \mu^{\otimes k}, x\sim \mathcal{A}_U: x\not\in \Psi^{-1} (\Psi(U))\cup \Lambda ] \geq  \eta. \]
Next, we bound the runtime of the constructed algorithm  $ \mathcal{A}$ for learning-to-generalize $\mu.$ Step 1 and 2 together take
\begin{align*}
     &k \times O\left( |\mathcal{S}|^3 (|\mathcal{S}|^2 +\sum_{v\in V_H}r_v^2) \log (1/\delta) \left( \varepsilon^{-2} |\mathcal{S}|^2  \log(|\mathcal{S}|/\varepsilon) \log (1/\varepsilon) +  \log(|\mathcal{S}|/\delta) \right)\right) \\
     =  &O_{\gamma} \left( k^3 m^{10} W^{11} \epsilon^{-4} \log^2(m/\epsilon)  \log (k/\eta) \log (m k W )  \log k + k m^8 W^9  \epsilon^{-4} \log^2(m/\epsilon) \log (k/\eta) \log (mkW/\eta) \right)
\end{align*}
where we use the fact that $\forall v\in v_H: t_v = O_{\gamma}(W), r_v = O_{\gamma} (m^2 W^3 \epsilon^{-2} \log(m/\epsilon)).$
Hence, the runtime of $\mathcal{A}$ is
\[ t=\tilde{t} +O_{\gamma} \left( k^3 m^{10} W^{11} \epsilon^{-4} \log^2(m/\epsilon)  \log (k/\eta) \log (m k W )  \log k + k m^8 W^9  \epsilon^{-4} \log^2(m/\epsilon) \log (k/\eta) \log (mkW/\eta) \right). \]

We return to the analysis.
Below, we treat the $  \hat{Z}'_{\mathcal{H}} (y)$'s as fixed numbers. The argument is similar to in the ideal case.
Let $\hat{Z}'_{\mathcal{H}} := \sum_{y} \hat{Z}'_{\mathcal{H}} (y).$
Note that
\begin{equation}
\begin{split}
  \Pr[y^{(i)}, i \in  \mathcal{I}'_A ]  &=  \Pr_{\mu}[y^{(i)}] \cdot  \Pr[i \in  \mathcal{I}'_A |y^{(i)} ] = \frac{\hat{Z}'_{\mathcal{H}} (y^{(i)}) } {(1+\varepsilon) (1+\epsilon) A \cdot Z_H}\\
  \Pr[i\in \mathcal{I}'_A]   &=  \sum_{y^{(i)}}  \Pr[y^{(i)}, i \in  \mathcal{I}'_A ]  =   
\frac{\hat{Z}'_{\mathcal{H}}  }{(1+\varepsilon)(1+\epsilon) A \cdot Z_H } \\
\Pr[y^{(i)} | i \in \mathcal{I}'_A] &= \frac{\hat{Z}'_{\mathcal{H}} (y^{(i)} ) }{ \hat{Z}'_{\mathcal{H}} }
\end{split}
\end{equation}

Let $p':= \Pr[i\not\in \mathcal{I}'_A] =1- \frac{\hat{Z}'_{\mathcal{H}}  }{(1+\varepsilon)(1+\epsilon) A \cdot Z_H }.$ 

For a fixed $U = (y^{(i)})_{i=1}^k$, $ \mathcal{I}'_A$ and $ \Sigma',$ we have:
\begin{align*}
    &\Pr[U, \mathcal{I}'_A, \Sigma']\\
    &=  \prod_{i=1}^k \Pr_{\mu}[y^{(i)}] \left (\prod_{i\in \mathcal{I}'_A} \Pr[ i\in \mathcal{I}'_A | y^{(i)}] \Pr_{\hat{\sigma}^{'(i)} \sim \tilde{\mu}(\cdot \mid Y(\hat{\sigma}^{'(i)}) = y^{(i)}) } [\hat{\sigma}^{'(i)} ]
    \right)\left(\prod_{i\in [k] \setminus \mathcal{I'}_A} \Pr[ i\not\in \mathcal{I}'_A | y^{(i)}]\right)\\
    &=  \left( \prod_{i\in \mathcal{I}_A'}  \frac{\hat{Z}'_{\mathcal{H}} (y^{(i)} ) }{ (1+\varepsilon)(1+\epsilon) A\cdot Z_{H} } \cdot \frac{\tilde{\mu}(\hat{\sigma}^{'(i)}) }{Z_{\mathcal{H}}(y^{(i)}) } \right) \left(\prod_{i\in [k] \setminus \mathcal{I}'_A } \Pr[y^{(i)}, i \not\in  \mathcal{I}'_A ]  \right) \\
    &= (1-p')^{|\mathcal{I}'_A|} \left(\prod_{i\in [k] \setminus \mathcal{I}_A' } R'(y^{(i)})\right)\left( \prod_{i\in \mathcal{I}_A'} \Pr_{\tilde{\mu}'} [\hat{\sigma}^{'(i)}] \right) 
\end{align*}
where the distribution $\tilde{\mu}'$ is defined by
\[\Pr_{\tilde{\mu}'}[\sigma] = \frac{\hat{Z}'_{\mathcal{H}}(Y(\sigma))}{\hat{Z}'_{\mathcal{H}}}\cdot \frac{\tilde{\mu}(\sigma) }{Z_{\mathcal{H}}(Y(\sigma))}  \]
and we let
\[R'(y^{(i)}) :=  \Pr[y^{(i)}, i \not\in  \mathcal{I}'_A ] = \frac{\mu(y^{(i)}) }{Z_H } - \frac{\hat{Z}'_{\mathcal{H}} (y^{(i)}) } {(1+\varepsilon)(1+\epsilon) A \cdot Z_H} ,\] which depends only on the value $y^{(i)}$ and is independent of the index $i.$
Note that since $\hat{Z}'_{\mathcal{H}} (y) = (1\pm\varepsilon) Z_{\mathcal{H}} (y) ,$  we have $  \hat{Z}'_{\mathcal{H}} =  (1\pm\varepsilon) Z_{\mathcal{H}} $ where $Z_{\mathcal{H}} =\sum_y Z_{\mathcal{H}} (y) =\sum_{\sigma} \tilde{\mu}(\sigma), $ and
\[ \frac{\Pr_{\tilde{\mu}'}[\sigma]}{\Pr_{\tilde{\mu}}[\sigma] } = \frac{\hat{Z}'_\mathcal{H} (Y(\sigma))}{Z_\mathcal{H} (Y(\sigma))}\cdot \frac{Z_\mathcal{H} }{\hat{Z}'_\mathcal{H}}\in [\frac{1-\varepsilon}{1+\varepsilon}, \frac{1+\varepsilon}{1-\varepsilon}] \subseteq [\exp(-3\varepsilon), \exp(3\varepsilon)]  \]
hence
\[  \Pr[U, \mathcal{I}'_A, \Sigma'] \geq \exp(-3 k\varepsilon) (1-p')^{|\mathcal{I}'_A|} \left(\prod_{i\in [k] \setminus \mathcal{I}'_A } R'(y^{(i)})\right)\left( \prod_{i\in \mathcal{I}'_A} \Pr_{\tilde{\mu}} [\hat{\sigma}^{'(i)}] \right)  \]
Arguing the same as in the ideal case, we have:
\allowdisplaybreaks
\begin{align*}
&\Pr[U, \mathcal{I}'_A, \Sigma', \tau \sim \tilde{\mathcal{A}}_{\Sigma'}: \tilde{\Psi}(\tau)\not\in \Psi(U) \text{ and } \tau \not\in \tilde{\Lambda} ]\\
&= \sum_{U, \mathcal{I}'_A, \Sigma'} \Pr[U, \mathcal{I}'_A, \Sigma'] \sum_{\tau \in\tilde{\Omega}\setminus \tilde{\Lambda}: \tilde{\Psi}(\tau)\not\in \Psi(U) } \tilde{\mathcal{A}}_{\Sigma'} (\tau)\\
&\geq e^{-3k\varepsilon}\sum_{\substack{\mathcal{I}'_A \subseteq [k]\\(y^{(i)})_{i} \in \Omega^{[k]\setminus \mathcal{I}'_A}  \\(\hat{\sigma}^{'(i)})_{i} \in \tilde{\Omega}^{\mathcal{I}'_A}  }}  (1-p')^{|\mathcal{I}'_A|}  \left(\prod_{i\in [k] \setminus \mathcal{I}'_A } R'(y^{(i)})\right) \left( \prod_{i\in \mathcal{I}'_A} \Pr_{\tilde{\mu}} [\hat{\sigma}^{'(i)}] \right) \sum_{\substack{\tau \in\tilde{\Omega}\setminus \tilde{\Lambda}\\ 
\tilde{\Psi}(\tau)\neq \Psi(y^{(i)})\forall i\not\in \mathcal{I}'_A \\ \tilde{\Psi}(\tau) \neq  \tilde{\Psi}(\hat{\sigma}^{'(i)}) \forall i\in \mathcal{I}'_A
}} \tilde{\mathcal{A}}_{(\hat{\sigma}^{'(i)})_{i\in \mathcal{I}'_A}} (\tau)\\
&\geq e^{-3k\varepsilon} \sum_{\ell=\tilde{k}}^k \sum_{ (\tilde{\sigma}^{(r)})_{r=1}^{\tilde{k}}} 
\sum_{\substack{\tau \in\tilde{\Omega}\setminus \tilde{\Lambda} \\ 
\tilde{\Psi}(\tau) \neq  \tilde{\Psi}(\tilde{\sigma}^{(r)}) \, \forall r\in [\tilde{k}]
}}  \binom{k}{\ell} \left(\prod_{r=1}^{\tilde{k}}
\Pr_{\tilde{\mu}} [\tilde{\sigma}^{(r)}]\right)
\tilde{\mathcal{A}}_{(\tilde{\sigma}^{(r)})_{r=1}^{\tilde{k}}} (\tau)  (1-p')^\ell\\
&\quad \times \left( \sum_{ \substack{(\tilde{\sigma}^{(r)})_{r=\tilde{k}+1}^{\ell}\\\tilde{\Psi}(\tilde{\sigma}^{(r)}) \neq \tilde{\Psi}(\tau) \forall \tilde{k}< r \leq \ell}} \prod_{r=\tilde{k}+1}^\ell   \Pr_{\tilde{\mu}} [\tilde{\sigma}^{(r)}] \right)\times \left( \sum_{\substack{(\tilde{y}^{(s)})_s \in \Omega^{k-\ell} \\
\Psi(\tilde{y}^{(s)}) \neq  \tilde{\Psi}(\tau) \, \forall s\in [k-\ell] } }   
\prod_{s\in [k-\ell] } R'(\tilde{y}^{(s)})\right)\\
&\geq \frac{7}{16}\cdot \tilde{\eta} \cdot \exp(-5k \cdot (\Psi_*\mu)_{\max} ) \cdot \exp(-3k\varepsilon)  \\
&\geq \frac{8}{7}\cdot \eta 
\end{align*}
where the final inequality is by setting $\varepsilon= \frac{1}{C k}$ for some sufficiently large constant $C>0.$ The penultimate inequality follows from the same argument as in the ideal setting, and
\begin{align*}
    1\geq 1- p' + \sum_{y \in \Omega \setminus \Psi^{-1}(\tilde{\Psi} (\tau)) } R'(y) 
&= 1 - \sum_{y \in \Psi^{-1}(\tilde{\Psi} (\tau)) }
\Pr_{\mu}[y, i\not\in \mathcal{I}'_A]
\geq 1 - \sum_{y \in \Psi^{-1}(\tilde{\Psi} (\tau)) }
\Pr_{\mu}[y]\\
&\geq \exp(- 2 (\Psi_*\mu)_{\max} )
\end{align*}
Set $\delta = \frac{\eta}{14k},$ then the above together with \cref{eq:relate between actual and ideal case via coupling} implies
\[ \Pr[U\sim \mu^{\otimes k}, x\sim \mathcal{A}_U: x\not\in \Psi^{-1} (\Psi(U))\cup \Lambda ] \geq  \eta, \]
which finishes the proof.
    
Finally, we return to proving  \cref{prop:sample from conditional distribution on phase,prop:estimate phase probability}.

We need the following propositions.
\begin{proposition}\label{prop:pinning S}
    For a given assignment $ \tilde{\sigma}_{\mathcal{S}}\in \set{\pm 1}^{\mathcal{S}}:$
    \begin{itemize}
        \item There exists an algorithm that exactly computes $\mu_{\mathcal{H}}(\tilde{\sigma}_\mathcal{S}; Y) $ in  $O(|\mathcal{S}|^2+\sum_{v\in V_H} r_v^2) $ time.
        \item For a given accuracy parameter $ \delta >0$, there exists an algorithm that in $ O (|\mathcal{S}|^2 + \sum_{v\in V_H} r_v ^2 \log(m r_v/\delta) )$ time, outputs a sample from a distribution that is $\delta$-close to \[\mu_{\mathcal{H}} (\sigma_{\mathcal{R}}=\cdot | Y(\sigma) = y \, \text{and} \, \sigma_{\mathcal{S}} = \tilde{\sigma}_{\mathcal{S}})\] in total variation distance.
    \end{itemize}
    
\end{proposition}
\begin{proof}
    Removing $\mathcal{S}$ factors the graph $\mathcal{H}$ into $m$ connected components on $ R_v = W_v \setminus \mathcal{S} $ for $v\in V_H.$  
    Hence
\[ \mu_{\mathcal{H}} (\sigma_{\mathcal{R}}=\cdot | Y(\sigma) = y \, \text{and} \, \sigma_{\mathcal{S}} = \tilde{\sigma}_{\mathcal{S}}) = \bigotimes_{v\in V_H} \mu_{\mathcal{H}}(\sigma_{R_v} =\cdot | Y(\sigma) = y \, \text{and} \, \sigma_{\mathcal{S}} = \tilde{\sigma}_{\mathcal{S}}).\]
where $\mu_{\mathcal{H}}(\sigma_{R_v} =\cdot | Y(\sigma) = y \, \text{and} \, \sigma_{\mathcal{S}} = \tilde{\sigma}_{\mathcal{S}})$ is an Ising model on $R_v$, with interaction matrix $J_v:= \frac{\beta}{r_v}\cdot \mathbf{1}_{R_v} \mathbf{1}_{R_v}^\intercal$ and external field $\tilde{h}'_i =\tilde{h}_i + \sum_{j\in S_v} \tilde{J}_{ij}\tilde{\sigma}_j $ for $i\in R_v.$
For $k\in \set{0,\cdots, r_v}, $ let
    \begin{align*}
        Z_v^{k/r_v}(\tilde{\sigma}_{\mathcal{S}})  &=\sum_{\sigma \in \set{\pm 1}^{R_v}: m_{R_v}(\sigma) = 2k-r_v } \exp(\frac{1}{2} \sigma^\intercal J_v \sigma +  \sigma^\intercal \tilde{h}') \\
        &= \sum_{\sigma \in \set{\pm 1}^{R_v}: m_{R_v}(\sigma) = 2k-r_v }  \exp\left(\frac{\beta}{2 r_v} \cdot (2k-r_v)^2 + \sum_{i\in R_v} \tilde{h}'_{i} \sigma_i\right) \\
        &=  \exp\left(\frac{\beta}{2 r_v} \cdot (2k-r_v)^2 -\sum_{i\in R_v} \tilde{h}'_{i}\right)\cdot   e_{k} ((\exp(2\tilde{h}'_{i}) )_{i\in R_v}) 
    \end{align*}
    where $e_k((\exp(2 \tilde{h}'_{i}))_{i\in R_v}) =\sum_{U\in \binom{R_v}{k} } \exp(2\sum_{i\in U} \tilde{h}'_{i} )$ is the $k$-th symmetric polynomial.
Let
\begin{align*}
L_v^+ =  (r_v/2,r_v] \cap \N, &\quad L_v^- = [0,r_v/2) \cap \N\\
    Z_v^{+} = \sum_{k\in L_v^+ } Z_v^{k/r_v}, &\quad Z_v^{-} = \sum_{k\in L_v^-  } Z_v^{k/r_v}
\end{align*}
We have
    \begin{align*}
        \mu_{\mathcal{H}} (\tilde{\sigma}_\mathcal{S}; Y)=\exp\left(\frac{1}{2} \sum_{i, j\in \mathcal{S}} \tilde{J}_{ij}\tilde{\sigma}_i \tilde{\sigma}_j  +\sum_{i\in \mathcal{S}} \tilde{h}_{i}\tilde{\sigma}_i  \right) \prod_{v\in V_H} Z_v^{Y_v} (\tilde{\sigma}_{\mathcal{S}})
    \end{align*}

   For each $v \in V_H,$ we can compute all $e_{k}: = e_k((\exp(2\tilde{h}'_{i}))_{i\in R_v}) $ for $k\in \set{0,\cdots, r_v}$ in $ O(r_v^2 + t_v^2)$ time.
    
    Indeed, we can compute $ \tilde{h}'_i$ for $i\in R_v$ in $O(r_v t_v)$ time. For $k\in [r_v],$ let $ p_k = \sum_{i\in R_v} \exp(2 k \tilde{h}'_i)$ then Newton's identity gives:
    \[ k e_k =  e_{k-1}p_1 - e_{k-2}p_2 + e_{k-3}p_3 -\dots \pm p_k.\]
   We compute all $p_k$'s in $O(r_v^2)$ time, then recursively compute the $e_k$'s in $ O(r_v^2)$ time.

    Given the $e_k$'s, we compute each $ Z_v^{k/r_v}(\tilde{\sigma}_{\mathcal{S}}) $ in $O(1)$ time, $  Z_v^{Y_v} (\tilde{\sigma}_{\mathcal{S}})$ in $\Theta(r_v)$ time and $  \mu_{\mathcal{H}} (\tilde{\sigma}_\mathcal{S}; Y)$ in $ \Theta( |\mathcal{S}|^2 + m )$ time. 
    
    Hence, the total time to compute $ \mu_{\mathcal{H}} (\tilde{\sigma}_\mathcal{S}; Y)$ is
    $ \Theta(|\mathcal{S}|^2 + \sum_{v\in V_H} r_v^2).$ 

For part 2, we only need to approximately sample $\tilde{\sigma}_{R_v}$ from each $ \mu_{\mathcal{H}}(\sigma_{R_v} =\cdot | Y(\sigma) = y \, \text{and} \, \sigma_{\mathcal{S}} = \tilde{\sigma}_{\mathcal{S}}),$ then concatenate these $\tilde{\sigma}_{R_v}$ to obtain the output $\tilde{\sigma}_{\mathcal{R}}.$  

Let $L_v^+ = \set{\frac{r_v+1}{2}, \cdots, r_v}, L_v^- =\set{0, \cdots,\frac{r_v-1}{2}}.$ 
To sample from $ \mu_{\mathcal{H}}(\sigma_{R_v} =\cdot | Y(\sigma) = y \, \text{and} \, \sigma_{\mathcal{S}} = \tilde{\sigma}_{\mathcal{S}}),$
we sample $k\in L_v^{Y_v}$ with probability $ \frac{Z_v^{k/r_v}(\tilde{\sigma}_{\mathcal{S}}) }{Z_v^{Y_v}(\tilde{\sigma}_{\mathcal{S}}) },$ which we have computed earlier, then approximately sample from the distribution $\nu_v$ defined by:
\[  \nu_v:=\mu_{\mathcal{H}}(\sigma_{R_v} =\cdot | Y(\sigma) = y \, \text{and} \, \sigma_{\mathcal{S}} = \tilde{\sigma}_{\mathcal{S}} \, \text{and} \,  \sum_{i\in R_v}\sigma_i = 2k -r_v). \]
Let $ \nu'_v = \text{Uniform}(\binom{R_v}{k}) $ and $ \nu''_v = \mathcal{T}_{2\tilde{\mathbf{h}}'} \nu'_v$ be the tilt of $\nu'_v$ by the vector $2\tilde{\mathbf{h}}' = (2\tilde{h}'_i)_{i\in R_v};$ equivalently, $\forall D_v \in \binom{R_v}{k}: \nu''_v(D_v)= \prod_{i\in D_v} \exp(2 \tilde{h}'_i).$ 
The distribution  $ \nu_v$ is the push-forward of $\nu''_v: \binom{R_v}{k} \to \R_{\geq 0}$ by the map $S \mapsto  ((-1)^{\mathbf{1}[i\not\in S]})_{i\in R_v}.$   Since the set $\binom{R_v}{k}$ corresponds to the bases of a matroid,  $\nu''_v$ has a strongly log-concave generating polynomial (see e.g. \cite{logconcaveII}). Thus, using the down-up walk Markov chain (see e.g. \cite[Theorem 1]{logconcaveIV}), we can sample from within $\delta'$-total variation distance of $\nu''_v $ in $O(r_v \log(r_v/\delta'))$ Markov chain steps, where each step takes $O(r_v)$ time. Thus, we can sample from within $ \delta':=\frac{\delta}{m r_v}$-total variation distance of $  \nu_v$ in $O(r_v^2 \log(m r_v/\delta))$ time. Taking everything together, the output distribution of our algorithm is within $ \delta$ of the target distribution $ \mu_{\mathcal{H}} (\sigma_{\mathcal{R}}=\cdot | Y(\sigma) = y \, \text{and} \, \sigma_{\mathcal{S}} = \tilde{\sigma}_{\mathcal{S}}). $ The runtime of this algorithm is 
\[ O (\sum_{v\in V_H} (r_v t_v + r_v^2) +\sum_{v\in V_H} r_v^2 \log(m r_v/\delta) ) =  O(|\mathcal{S}|^2 + \sum_{v\in V_H} r_v^2 \log(m r_v/\delta) ) \]
\end{proof}
\begin{proposition}\label{prop:sample from distribution on subset of S}
    For a given accuracy parameter $\delta >0, $ there exists an algorithm that takes as input $y\in \set{\pm 1}^{V_H},$ and $\mathcal{T}\subseteq \mathcal{S}, \tilde{\sigma}_{\mathcal{T}}\in \set{\pm 1}^{\mathcal{T}},$  runs in $  O( |\mathcal{S}|^3 (|\mathcal{S}|^2 +\sum_{v\in V_H}r_v^2) \log(|\mathcal{S}|/\delta) \log (1/\delta))$ time and outputs a sample from a distribution that is $\delta$-close to $\mu_{\mathcal{H}} (\sigma_{\mathcal{S}\setminus \mathcal{T}}=\cdot | Y(\sigma) = y\, \text{and}\, \sigma_{\mathcal{T}} = \tilde{\sigma}_{\mathcal{T}})$ in total variation distance.
\end{proposition}

To prove \cref{prop:sample from distribution on subset of S}, we need the following classical sampling-to-counting reduction from \cite{JS89}.
\begin{theorem}[{\cite{JS89}}]\label{thm:sampling to counting}
   Consider an un-normalized density function $\nu: \set{\pm 1}^n\to \R_{\geq 0}.$ For $S \subseteq[n], x_S \in\set{\pm 1}^S,$ let
   \[ \nu(x_S) = \sum_{x' \in \set{\pm 1}^n:x'_S =x_S} \nu(x').\]
   Let $ [k]=\set{1,\cdots, k}.$
   Suppose we have $\hat{\nu} $ satisfying that, for some $\kappa \geq 1$
   \[\forall k \in [n], x_{[k]} \in\set{\pm 1}^{[k]}: \kappa^{-1} \leq \frac{\hat{\nu}( x_{[k]})}{\nu(x_{[k]}) }  \leq \kappa \quad \text{and}\quad \forall x\in \set{\pm 1}^n: \hat{\nu}(x) = \nu(x).\]
   Fix a parameter $\delta >0.$ Let $M := O(n^3\kappa^5 \log(n \kappa/\delta) \log \delta^{-1}).$ There exists an algorithm that runs in $O(M)$ time and uses $O(M)$ oracle calls to $\hat{\nu}(\cdot)$ and outputs a sample from a distribution that is $\delta$-close to $ \nu$ in total variation distance.
\end{theorem}

\begin{proof}[Proof of \cref{prop:sample from distribution on subset of S}]
    Fix $ \mathcal{T}\subseteq \mathcal{S}.$ Let $ \mathcal{S}'= \mathcal{S}\setminus \mathcal{T}.$ Let $\nu \equiv  \mu_{\mathcal{H}} (\sigma_{\mathcal{S}\setminus \mathcal{T}}=\cdot | Y(\sigma) = y\, \text{and}\, \sigma_{\mathcal{T}} = \tilde{\sigma}_{\mathcal{T}}):\set{\pm 1}^{\mathcal{S}'}\to \R_{\geq 0}$ where
    \[\forall \tilde{\sigma}_{ \mathcal{S}'} \in \set{\pm 1}^{\mathcal{S}'} : \nu(\tilde{\sigma}_{ \mathcal{S}'} ) = \mu_{\mathcal{H}}(\tilde{\sigma}_{ \mathcal{S}' \cup \mathcal{T}}; Y) \]
    By \cref{prop:pinning S}, we can exactly compute $\nu(\tilde{\sigma}_{ \mathcal{S}'} )  $ in $O (|\mathcal{S}|^2 +\sum_{v\in V_H}r_v^2) $ time.
    
    We use the classical Jerrum-Sinclair sampling-to-counting reduction in \cref{thm:sampling to counting} to approximately sample from $\nu.$
    Note that for any subset $\mathcal{T}'\subseteq \mathcal{S}'$ and partial configuration $\tilde{\sigma}_{ \mathcal{T}'} \in \set{\pm 1}^{\mathcal{T}'} ,$ we have:
    \[\nu(\tilde{\sigma}_{ \mathcal{T}'} ) =\sum_{\tilde{\sigma}'_{ \mathcal{S}'}\in  \set{\pm 1}^{\mathcal{S}'} : \tilde{\sigma}'_{\mathcal{T'}} =\tilde{\sigma}_{\mathcal{T'}} }  \mu_{\mathcal{H}}(\tilde{\sigma}'_{ \mathcal{S}' \cup\mathcal{T}  }; Y) =\mu_{\mathcal{H}}(\tilde{\sigma}_{ \mathcal{T}' \cup\mathcal{T}  }; Y).\]
   By \cref{lem:main construction probability of phase gadget}, with  $O(N)$ pre-processing time to approximate $ Z_{\hat{\mathcal{H}}}(y)$, for each query $\tilde{\sigma}_{ \mathcal{T}'}$ where $\mathcal{T}'\subseteq \mathcal{S}'$ and $\tilde{\sigma}_{ \mathcal{T}'} \in \set{\pm 1}^{\mathcal{T}'}$, we can compute $\hat{\nu}(\tilde{\sigma}_{ \mathcal{T}'} ) \equiv \hat{\mu}_{\mathcal{H}}(\tilde{\sigma}_{ \mathcal{T}'\cup \mathcal{T}}; Y)$ that approximates $\nu(\tilde{\sigma}_{ \mathcal{T}'} ) $ within $(1\pm \epsilon)$-factor in $O(|\mathcal{S}|^2)$ time. For full configurations $ \tilde{\sigma}_{ \mathcal{S}'}\in  \set{\pm 1}^{ \mathcal{S}'},$  we re-define $\hat{\nu}(\tilde{\sigma}_{ \mathcal{S}'} )$ to be $ \nu(\tilde{\sigma}_{ \mathcal{S}'} ).$ Let $M:= |\mathcal{S}'|^3\log^2 (1/\delta).$ \cref{thm:sampling to counting} with $\kappa = \max \set{(1+\epsilon),(1-\epsilon)^{-1}} = O(1)$ implies that: the Jerrum-Sinclair algorithm takes  $O(M)$ time, $O(M)$ oracle calls to $ \hat{\nu}(\cdot),$ where each oracle call takes $ O(|\mathcal{S}|^2 +\sum_{v\in V_H}r_v^2)$ time.
   Thus, the total runtime is \[ O( |\mathcal{S}|^3 (|\mathcal{S}|^2 +\sum_{v\in V_H}r_v^2) \log(|\mathcal{S}|/\delta) \log (1/\delta)).\qedhere\]
\end{proof}

\begin{proof}[Proof of \cref{prop:sample from conditional distribution on phase}]
We first use the algorithm in \cref{prop:sample from distribution on subset of S} with $\mathcal{T}=\emptyset$ and accuracy parameter $ \delta/2$ to produce a sample $ \tilde{\sigma}_{\mathcal{S}}\in \set{\pm 1}^{\mathcal{S}},$ then use the algorithm in part 2 of \cref{prop:pinning S} with accuracy parameter $ \delta/2$  to produce a sample $ \tilde{\sigma}_{\mathcal{R}}\in \set{\pm 1}^{\mathcal{R}},$ then output $\tilde{\sigma}_{\mathcal{R}\cup\mathcal{S} } \in \set{\pm 1}^{\mathcal{V}}.$ The output distribution of this algorithm is within $\delta$ of $\tilde{\mu} (\cdot | Y(\sigma)=y) \equiv \mu_{\mathcal{H}}(\cdot | Y(\sigma)=y) $ in total variation distance.

The algorithm's runtime is $ O( |\mathcal{S}|^3 (|\mathcal{S}|^2 +\sum_{v\in V_H}r_v^2) \log(|\mathcal{S}|/\delta) \log (1/\delta)).$
\end{proof}
To prove \cref{prop:estimate phase probability}, we need the following classical counting-to-sampling reduction from \citep{JS89}.
\begin{proposition}[{\cite{JS89}}]\label{prop:counting-to-sampling}
Let $\nu: \{ \pm 1 \}^n \to \mathbb{R}_{\geq 0}$ be an unnormalized density function, and let $Z = \sum_{x \in \{ \pm 1 \}^n} \nu(x)$ denote its normalization factor. 

For any accuracy parameter $\varepsilon \in (0,1)$, there exists an algorithm that:
\begin{itemize}
    \item Makes one oracle call to $\nu(x)$ for $x \in \{ \pm 1 \}^n$,
    \item Uses $K = O\left(\frac{n^2}{\varepsilon^2}\right)$ samples from distributions of the form $\nu(\cdot \mid X_{[k]} = x_{[k]})$ (for any $k \leq n$ and any partial configuration $x_{[k]} \in \{ \pm 1 \}^{[k]}$),
    \item Runs in $O(K)$ time, and
    \item Outputs an estimate $\hat{Z}$ such that:
    \[
    \Pr[\hat{Z} = (1 \pm \varepsilon) Z] \geq \frac{3}{4}.
    \]
\end{itemize}

If the algorithm instead accesses samples from an approximate distribution $\hat{\nu}(\cdot \mid X_{[k]} = x_{[k]})$ satisfying $d_{TV}(\hat{\nu}(\cdot \mid X_{[k]} = x_{[k]}), \nu(\cdot \mid X_{[k]} = x_{[k]})) \leq \frac{1}{8K}$, then:
\[
\Pr[\hat{Z} = (1 \pm \varepsilon) Z] \geq \frac{5}{8}.
\]

The success probability can be amplified to $1 - \delta$ for any $\delta \in (0,1)$ by running the algorithm $\Theta(\log(1/\delta))$ times and returning the median of the estimates $\hat{Z}$.
\end{proposition}
We include the proof of \Cref{prop:counting-to-sampling} for the sake of completeness in Appendix \ref{sec:defer proof from ising sparsification}.


\begin{proof}[Proof of \cref{prop:estimate phase probability}]
     Let $\nu \equiv  \mu_{\mathcal{H}} (\sigma_{\mathcal{S}}=\cdot | Y(\sigma) = y):\set{\pm 1}^{\mathcal{S}}\to \R_{\geq 0}$ where
    \[\forall \tilde{\sigma}_{ \mathcal{S}} \in \set{\pm 1}^{\mathcal{S}} : \nu(\tilde{\sigma}_{ \mathcal{S}} ) = \mu_{\mathcal{H}}(\tilde{\sigma}_{ \mathcal{S}}; Y). \]

    We compute $Z = \sum_{ \tilde{\sigma}_{ \mathcal{S}} \in \set{\pm 1}^{\mathcal{S}} } \nu(\tilde{\sigma}_{ \mathcal{S}})$ using the standard counting-to-sampling reduction in \cref{prop:counting-to-sampling}. 

    Wlog, assume $ \mathcal{S} = \set{1, \cdots, |\mathcal{S}|}.$ 

    Let $K:= O(|\mathcal{S}|^2/\varepsilon^2)
    $ as in \cref{prop:counting-to-sampling}.
     By \cref{prop:pinning S}, we can exactly compute $\nu(\tilde{\sigma}_{ \mathcal{S}} )  $ in $O (|\mathcal{S}|^2 +\sum_{v\in V_H}r_v^2) $ time.
      For any $k \leq |\mathcal{S}|,$ and partial configuration $x_{[k]} \in \set{\pm 1}^{[k]},$ the conditional distribution $ \nu(\cdot |X_{[k]} = x_{[k]})$ is exactly $\mu_{\mathcal{H}} (\sigma_{\mathcal{S}\setminus \mathcal{T}}=\cdot | Y(\sigma) = y\, \text{and}\, \sigma_{\mathcal{T}} = \tilde{\sigma}_{\mathcal{T}})$ for $\mathcal{T}=[k] \subseteq \mathcal{S}.$  Hence, by \cref{prop:sample from distribution on subset of S},  we can approximately sample from within $ \delta' = \frac{1}{8K}$ total variation distance of $\nu(\cdot |X_{[k]} = x_{[k]}) $ in time \[ O( |\mathcal{S}|^3 (|\mathcal{S}|^2 +\sum_{v\in V_H}r_v^2) \log(|\mathcal{S}|/\delta') \log (1/\delta')) = O( |\mathcal{S}|^3 (|\mathcal{S}|^2 +\sum_{v\in V_H}r_v^2) \log(|\mathcal{S}|/\varepsilon) \log (1/\varepsilon)).\] Hence, the total runtime is 
    \[ O( \varepsilon^{-2} |\mathcal{S}|^5  (|\mathcal{S}|^2 +\sum_{v\in V_H}r_v^2) \log(|\mathcal{S}|/\varepsilon) \log (1/\varepsilon) \log (1/\delta)).\qedhere\]
\end{proof}

\subsection{Putting things together: proof of \cref{cor:main hardness of near critical Ising model from computational diffie hellman}}
\label{sec:puttingtogether} 

In this subsection, we prove \cref{cor:main hardness of near critical Ising model from computational diffie hellman} by putting together all the components gathered so far.
We also discuss the implications of \cref{cor:main hardness of near critical Ising model from computational diffie hellman} on the learning-to-sample and sample-amplification tasks---see \cref{remark:learning pk from samples and relation to sample amplification}. In fact, we show the following more general theorem, which is a direct corollary of \cref{lem:Ising model is hard to generalize if signature scheme is secure} and \cref{cor:sparisfy a family of Ising models}. Below, we import the notations from \cref{sec:embed general signature scheme into Ising}.

\begin{theorem}[Main result, general signature scheme]\label{thm:main hardness of near critical Ising based on signature scheme}

    Consider a signature scheme that is $(t',q', \varepsilon')$-secure, has deterministic verification, and satisfies uniformity and exact regularity.
Let $m$ be as in \eqref{eq:signature scheme param}.
Suppose  $w = \Omega(\max \set{m, \log(1/\varepsilon') })$ and let $(\mu^{\pk})_{\pk \sim \keygen}$ be the family of Ising models on $m$ vertices, constructed from this signature scheme using the procedure in \cref{lem:Ising model is hard to generalize if signature scheme is secure}
with parameter $w$. 

Fix a parameter $\gamma > 1$. Let $(\tilde{\mu}^{\pk})_{\pk\sim \keygen}:\set{\pm 1}^N\to \R_{\geq 0}$ be the family of Ising models constructed from $(\mu^{\pk})_{\pk \sim \keygen}$ using the procedure in \cref{cor:sparisfy a family of Ising models} with  parameter $\gamma.$ Let $Y$ be the phase map in \cref{subsec:notation for Ising embedding}, $\Psi_{\msg}$ as in \cref{lem:Ising model is hard to generalize if signature scheme is secure}, and $ \tilde{\Psi}_{\msg} = \Psi_{\msg}\circ Y.$

Then $(\tilde{\mu}^{\pk})_{\pk\sim \keygen}$ is a family of Ising models on $N= \tilde{O}_{\gamma}(m^{9} w^3) $  vertices and is $(\tilde{\eta},\tilde{\zeta},\tilde{k},\tilde{t} , \tilde{\Psi}_{\msg} )$-hard for learning-to-generalize if: 
    \[\tilde{\eta} = 20 \varepsilon', \,\, \tilde{\zeta} = 1.05 \exp(-w) , \,\,  \min \set{\frac{q'}{2},c 2^{\ell_{\msg}}, \frac{\varepsilon' \exp(w)}{2}} \geq\tilde{k}, \,\,\tilde{t} = t' -C_{\gamma} \tilde{k}^3 m^{32} w^{12}  \mbox{polylog}(\tilde{k},m,w),\]
    for an absolute constant $c>0$ and a constant $ C_{\gamma} >0$ that depends only on $\gamma.$
    
    Furthermore, for all $ \pk,$ the parameter $ (\tilde{J}^{\pk}, \tilde{h}^{\pk})$ of $\tilde{\mu}^{\pk}$ satisfies 
     \[ 1< \lambda_{\max}(\tilde{J}^{\pk}) -\lambda_{\min}(\tilde{J}^{\pk}) \leq \gamma\quad \text{and} \quad \norm{\tilde{J}^{\pk}}_\infty  + \norm{\tilde{h}^{\pk}}_\infty = O(1).\]
\end{theorem}
\begin{remark}\label{remark:the map Psi}
The map $\widetilde{\Psi}_{\mathsf{msg}}:\{\pm1\}^N\to\{0,1\}^{\ell_{\msg}}$ is explicit
and computable in $O(N)$ time. If the verification procedure is efficient, then $ N = \poly(\ell_{\msg}).$ In particular, each output bit is a
linear-threshold function of the input coordinates.
\end{remark}

If desired, we can slightly modify the construction so that the resulting interaction matrices have zero diagonal entries (see \cref{remark:set zero diagonal entries}).

So far, we have allowed the Ising models in our constructions to have real-valued parameters. By applying a simple rounding step, we can also obtain Ising models with finite-bit parameters and qualitatively the same guarantees, so that the reduction from an algorithm for learning-to-generalize to an adversary for the signature scheme in \cref{lem:Ising model is hard to generalize if signature scheme is secure} can be carried out in the finite-bit RAM model. Hence, if the CDH assumption holds against adversaries in the finite-bit RAM model, then the resulting family of finite-bit Ising models is hard for learning-to-generalize against algorithms in the finite-bit RAM model.
\begin{remark}
Fix $\gamma >1.$
Consider a family of Ising model $ (\nu_{\theta})_{\theta \sim \mathcal{P}}$ that is $(\eta', \zeta', k, t, \Psi)$-hard for learning-to-generalize with respect to the family $(\Lambda_{\theta})_{\theta \sim \mathcal{P}},$ so that for each model $\nu_{\theta}$ in this family parameterized by $\theta \equiv (J',h'),$ $\lambda_{\max}(J')-\lambda_{\min}(J') \leq \gamma' $ and $ \norm{J'}_{\infty} + \norm{h'}_{\infty}\leq O(1).$ For each $\theta=(J',h')$, by rounding every entry of $J'$ and $h'$ within additive error $1/M$, we obtain an Ising model $ \nu_{\varphi(\theta)}$ parameterized by  $(J,h)$ where, by choosing $\gamma'=\frac{\gamma+1}{2}$ and $M = C_\gamma N^2 k$ for some sufficiently large constant $ C_{\gamma}>0,$ we ensure that:
\begin{align*}
    \lambda_{\max}(J)-    \lambda_{\min}(J)&\leq  \gamma'+O(N/M)\leq \gamma \\
    \norm{J}_{\infty} + \norm{h}_{\infty}&\leq O(1) +O(N/M)= O(1)\\
    \forall \sigma: \exp(\frac{-2N^2}{M}) \Pr_{\sigma \sim \nu_{\theta}}[\sigma]   &\leq \Pr_{\sigma \sim \nu_{\varphi(\theta)}}[\sigma] \leq \exp(\frac{2N^2}{M}) \Pr_{\sigma \sim \nu_{\theta}}[\sigma]
\end{align*}
Hence, $ \Pr_{\sigma \sim \nu_{\varphi(\theta)}} [\sigma \in \Lambda_{\theta}] \leq \exp(\frac{2N^2}{M}) \zeta'\leq 2\zeta' $ and no algorithm $\mathcal{A}$ can $(\eta, k, t, \Psi, (\Lambda_{\theta})_{\theta \sim \mathcal{P}})$-learn-to-generalize for $ (\nu_{\varphi(\theta)})_{\theta \sim \mathcal{P}}$ with $ \eta =\eta'  \exp(\frac{2k N^2}{M}) \leq 2\eta'. $ Suppose such a $\mathcal{A}$ exists, then the algorithm $\mathcal{A'},$ which runs $\mathcal{A}$ with input parameter $\varphi(\theta)$,  can $(\eta', k, t, \Psi, (\Lambda_{\theta})_{\theta \sim \mathcal{P}})$-learn-to-generalize for $ (\nu_{\theta})_{\theta \sim \mathcal{P}},$ a contradiction. Indeed:
\begin{align*}
    &\Pr_{\theta \sim \mathcal{P}, X\sim\nu_{\theta}^{\otimes k}, y\sim \mathcal{A'}_{\theta, X} } [y \not\in  (\Psi^{-1}(\Psi(X)) \cup \Lambda_{\theta} )]  \\
    &\geq \exp(\frac{-2k N^2}{M}) \Pr_{\theta \sim \mathcal{P}, X\sim\nu_{\varphi(\theta)}^{\otimes k}, y\sim \mathcal{A}_{\varphi(\theta), X} } [y \not\in (\Psi^{-1}(\Psi(X)) \cup \Lambda_{\theta} )]\\
    &\geq \eta'
\end{align*}
\end{remark}

If we establish this Theorem, then \cref{cor:main hardness of near critical Ising model from computational diffie hellman} follows directly from it by applying it with \cite{water05}'s signature scheme (see \cref{sec:signature}), which satisfies the conditions outlined in \cref{thm:main hardness of near critical Ising based on signature scheme}.
Specifically, \cite{water05}'s signature scheme has deterministic verification, and satisfies uniformity and exact regularity (see \cref{lem:water scheme}), and is secure under the computational Diffie-Hellman assumption on bilinear-pairing friendly groups (see \cref{thm:water scheme security}).

In \cref{remark:relaxing the regularity condition}, we discuss \cref{thm:main hardness of near critical Ising based on signature scheme}'s application to other signature schemes besides \citep{water05}.

We note that \cref{cor:sparisfy a family of Ising models} follows from the same argument as in the proof of \cref{thm:set parameter choice}. Next, we prove \cref{thm:main hardness of near critical Ising based on signature scheme}.
\begin{proof}[Proof of \cref{thm:main hardness of near critical Ising based on signature scheme}]
Consider an arbitrary public key $\pk.$
Note that by the uniformity and exact regularity of the signature scheme, $\Psi_{\msg}\ast\nu^{\pk}$ is the uniform distribution over $\set{0,1}^{\ell_{\msg}} ,$ hence $ (\Psi_{\msg}\ast\nu^{\pk})_{\max}\leq 2^{-\ell_{\msg}}.$ By the triangle inequality and data processing inequality for total variation distance:
    \[(\Psi_{\msg}\ast\mu^{\pk})_{\max} \leq(\Psi_{\msg}\ast\nu^{\pk})_{\max} + d_{TV}(\mu^{\pk},\nu^{\pk}) \leq 2^{-\ell_{\msg}} + \exp(-w) =   2^{-\ell_{\msg}+1}. \]
Hence for $\Psi \equiv \Psi_{\msg},$ $\Psi_{\max} \leq 2^{-\ell_{\msg}+1}. $

By \cref{lem:Ising model is hard to generalize if signature scheme is secure}, the family of distributions $(\mu^{\pk})_{\pk \sim \keygen}$ is $(\eta,\exp(-w), k,t, \Psi_{\msg} )$-hard for learning-to-generalize, where $\eta = 2\varepsilon'$, $t' =t+O(mk + m^2)$ and $k \leq \min\set{q', \varepsilon' \exp(w) }$. 

Recall from \cref{lem:dense-valid-pairs-fixed-pk}, that for any Ising model $\mu^{\pk}$ in this family, its interaction matrix $J$ and external field $ h^{\pk}$ satisfy:
\[ \max \set{|J_{ij}|, |h^{\pk}_i|} = O(m w).  \]

Let $\tilde{\zeta} = 1.05 \times \exp(-w),$ $\tilde{k}= \lceil (k+1)/2\rceil.$
By \cref{cor:sparisfy a family of Ising models}, the family of Ising models $(\tilde{\mu}^{\pk})_{\pk\sim \keygen}$ is $ (\tilde{\eta}, \tilde{\zeta}, \tilde{k}, \tilde{t}, \tilde{\Psi}_{\msg}) $-hard for learning-to-generalize, with $\tilde{\eta} =  5 \eta \exp(5k \Psi_{\max} ) = 20 \varepsilon', $ where we choose $ k \leq c 2^{\ell_{\msg}}$ for an appropriately small constant $c>0,$ and using the  fact that $\log(k/\eta) \leq O(w),$ due to the assumption $ \exp(w)\varepsilon' \geq k,$ we have:
\begin{align*}
   t &=  \tilde{t}+ O_{\gamma}(k^3 m^{32} w^{12} \log^2 m \log (m w k) \log k)
\end{align*}
Hence
\begin{align*}
    t' = t+O(mk + m^2) =\tilde{t} + C_{\gamma} \tilde{k}^3 m^{32} w^{12}\polylog(m,\tilde{k},w)
\end{align*}
for $C_{\gamma}$ dependent only on $\gamma.$

Next, we bound $N$, the number of vertices of the family of Ising models $(\tilde{\mu}^{\pk})_{\pk\sim \keygen}$.
The  bound on the number of vertices from \cref{cor:sparisfy a family of Ising models} implies:
    \[ N = \tilde{O}_{\gamma}(m^3  ( m^2 w)^3 ) =\tilde{O}_{\gamma}(m^{9} w^3) .\qedhere \]
\end{proof}
Finally we derive our main theorem, \cref{cor:main hardness of near critical Ising model from computational diffie hellman}.
\begin{proof}[Proof of \cref{cor:main hardness of near critical Ising model from computational diffie hellman}]
Consider  the \cite{water05}'s signature scheme  on $\mathbb{G}$ with message size $ \ell_{\msg} = d.$ By \cref{thm:water scheme security}, the signature scheme is $ (t',q',\varepsilon')$ secure for $t' = c' \cdot t_{\CDH}$ and $ \varepsilon'= C'\cdot\varepsilon_{\CDH} \cdot q' \cdot \sqrt{\ell_{\msg}}, $  for some constants $c',C'>0.$ 
    By the efficiency of the verification algorithm for the signature scheme, let $m$ be as in \eqref{eq:signature scheme param} then $ m = \poly(d).$
    
Set $ w = C_0 m + \log (1/\varepsilon_{\CDH})$ for a sufficiently large constant $C_0 >0.$ Clearly, $ w =\Omega(\max\set{m ,\log 1/\varepsilon'}).$ Let $(\tilde{\mu}^{\pk})_{\pk\sim \keygen}$ be the family of distributions obtained by applying the construction from \cref{thm:main hardness of near critical Ising based on signature scheme} to the above signature scheme with parameters $\gamma$ and $w.$

    Note that $w =\Theta(m),$ since we assume $ \varepsilon_{\CDH}\geq 1/p.$ By \cref{thm:main hardness of near critical Ising based on signature scheme}, $(\tilde{\mu}^{\pk})_{\pk\sim \keygen}$ is a family of Ising models on $N = \poly(d)$ vertices, and is $ (\tilde{\eta}, \tilde{\zeta}, \tilde{k}, \tilde{t}, \tilde{\Psi}_{\msg})$-hard for learning-to-generalize where, setting $q ' = 2\tilde{k},$ we have: 
\begin{align*}
    \tilde{\eta} &= 20 \varepsilon' = 20 C' \cdot \varepsilon_{\CDH} \tilde{k} \sqrt{d}  = C \varepsilon_{\CDH} \tilde{k} N^{\Theta(1)} \\
    \tilde{\zeta} &= 1.05 \exp(-w) \leq \varepsilon_{\CDH} \exp(-m) =  \varepsilon_{\CDH} \exp(-N^{\Theta(1)})\\
    \tilde{k}&\leq \min\set{ c 2^d, \varepsilon' \exp(w)/2}\\
    \tilde{t} &= t' -C'_{\gamma} \tilde{k}^3 m^{32} w^{12}  \polylog(k,m,w) = c'\cdot t_{\CDH} - C'_{\gamma} \tilde{k}^3 d^{\Theta(1)}
\end{align*}
for some absolute constant  $c>0$ and  constant $C'_{\gamma}>0$ dependent only on $\gamma.$

Hence, for
\begin{align*}
    \tilde{k} &\leq \min \left\{ c\, 2^d, \left(\frac{c' \cdot t_{\CDH}}{2 C'_{\gamma} d^{\Theta(1)}}\right)^{1/3} , \frac{1}{ C \varepsilon_{\CDH} \sqrt{d}} \right \} =\min  \left\{c\exp(N^{\Theta(1)}), \left(\frac{c'\cdot t_{\CDH}}{2C'_{\gamma} N^{\Theta(1)}}\right)^{1/3} ,  \frac{1}{ C \varepsilon_{\CDH} N^{\Theta(1)}} \right\},
\end{align*}
    we obtain $ \tilde{t}=\Theta( t_{\CDH}),$ and $\tilde{\eta}\leq 1$ as desired.
    
\end{proof}

\begin{remark}\label{remark:learning pk from samples and relation to sample amplification}

Recall the definition of $\Lambda^{\mathsf{pk},\mathsf{valid}}$ from \cref{subsec:embed signature scheme into a general Ising model}.
Let $\Lambda^{\mathsf{pk},\mathsf{invalid}}= \set{\pm 1}^m \setminus \Lambda^{\mathsf{pk},\mathsf{valid}},$ and $ \tilde{\Lambda}^{\mathsf{pk},\mathsf{invalid}} = Y^{-1}(\Lambda^{\mathsf{pk},\mathsf{invalid}}) .$ From the proof of \cref{lem:Ising model is hard to generalize if signature scheme is secure}, it is easy to see that $\Lambda^{\mathsf{pk},\mathsf{invalid}}$ plays the role of the hallucination set for $\mu^{\pk}$ in \cref{def:hard distribution for learning to generalize}. From \cref{thm:set parameter choice}, the set  $\tilde{\Lambda}^{\mathsf{pk},\mathsf{invalid}} $ plays the role of the hallucination set for $\tilde{\mu}^{\pk}.$

Clearly, the map $\widetilde{\Psi}_{\mathsf{msg}}$ can be implemented in $O(N)$ time. 
Moreover, given samples from $\tilde{\mu}^{\pk}$, one can recover
$\pk$ by, e.g., taking a majority vote among the samples for the sites corresponding to the public key (after applying the map $Y$). 
Given the public key $\pk,$ membership in the hallucination sets $\Lambda^{\mathsf{pk},\mathsf{invalid}}$, $\tilde{\Lambda}^{\mathsf{pk},\mathsf{invalid}}$
 can be efficiently tested.
 Hence, by \cref{remark:relating hard for learning to generalize in family of distribution vs for single distribution}, this family does not admit an efficient learning-to-sample   or a sample-amplification algorithm.
\end{remark}

\begin{remark}\label{remark:relaxing the regularity condition}
Besides \cite{water05}, several other signature schemes satisfy the preconditions of \cref{thm:main hardness of near critical Ising based on signature scheme} (see \cite{PS16,PS18}). In general, our techniques apply to any signature scheme with deterministic verification, which further satisfies that the distribution over $ (\pk,\msg, \sigma)$---obtained by sampling a uniformly random message $\msg$, then sampling  $ \sigma $ from  $\sign(\sk,\msg)$---is close to the uniform distribution over $\set{ (\msg,\sigma):\verify(\pk,\msg,\sigma)=1}$ in total variation distance. This condition can hold even when uniformity or exact regularity fails (see \cite{BB04} for an example).
\end{remark}
\section{Conclusion} 
\label{s:conclusion}
In this paper, we studied learning-to-sample as a clean algorithmic abstraction of the task underlying generative modeling, and showed that this problem can remain computationally intractable even when the learner is given both polynomially many training samples and explicit access to the model parameters. Under standard cryptographic assumptions, we construct bounded-width Ising models lying arbitrarily close to the spectral threshold for efficient sampling for which no efficient learner can generate fresh typical samples. This yields a cryptographic separation between parameter learning and learning-to-sample. Moreover, we formalize a memorize-or-hallucinate dichotomy: on our hard instances, any efficient learner must either replay information already present in the training data or place substantial mass on very rare configurations. 

There are many natural open directions. Most immediately, we prove hardness for every fixed $\gamma>1$. A natural question is what happens for $\gamma=1+o(1)$? Are there sharper quantitative limits on sample amplification or the memorize-or-hallucinate tradeoff? A broader question is whether there are other families of distributions for which a natural sampling tractability threshold coincides with the learning-to-sample threshold---and are there families of distributions in which these thresholds are different? 
\section*{Acknowledgements}
We thank Frederic Koehler and Niki Hasrati for insightful discussions. AR is supported in part by NSF awards IIS-2211907, CCF-2238523, IIS-2403275, an Amazon Research
Award, ONR award N000142512124, a Google Research Scholar Award, and an OpenAI Superalignment Fast Grant. The authors also thank the Simons Institute for the Theory of Computing, and part of this work was done while both authors were participants in the program on ``Modern Paradigms in Generalization''.
\bibliographystyle{plainnat}
\bibliography{references}

\appendix

\section{Relegated preliminaries on signature schemes}\label{sec:signature}

In this section, we recall the signature scheme of \cite{water05}, denoted by $(\keygen_{\mathrm{Wat}}, \sign_{\mathrm{Wat}}, \verify_{\mathrm{Wat}})$, and show that it satisfies the properties needed for our hardness reduction. The scheme is defined over a bilinear-pairing-friendly group $\mathbb{G}$ of prime order $p$, with generator $g$ and a bilinear map\footnote{See Subsection~\ref{subsec:prelim signature scheme}.} $e : \mathbb{G} \times \mathbb{G} \to \mathbb{G}_T$, where $\mathbb{G}_T$ is also of order $p$.

The scheme supports messages $m \in \set{0,1}^\ell, $ where $\ell\equiv \ell_{\msg}$ denotes the message length.

\begin{itemize}
    \item $\keygen_{\mathrm{Wat}}$:  
    The key generation samples
\[
h_0, h_1, \ldots, h_{\ell} \xleftarrow{\$} \mathbb{G}
\quad\text{and}\quad
\alpha, \beta \xleftarrow{\$} \mathbb{Z}_p.
\]
The public key is defined as
\[
\pk := (\mathbb{G}, g, g^\alpha, g^\beta, h_0, h_1, \ldots, h_{\ell}),
\]
and the secret key is $\sk := (pk, g^{\alpha\beta})$.

Next, we denote by $H : \{0,1\}^{\ell} \to \mathbb{G}$ the function mapping
\[
m \mapsto h_0 \prod_{i=1}^{\ell} h_i^{m_i},
\]
where for $i\in[\ell]$, we denote by $m_i \in \{0,1\}$ the $i$-th bit of $m$.
\item $\sign_{\mathrm{Wat}}(\sk, m):$ The signing algorithm takes as input a message $m \in \{0,1\}^{\ell}$.
The algorithm samples $r \xleftarrow{\$} \mathbb{Z}_p$ and computes
\[
\sigma_1 = g^r
\qquad\text{and}\qquad
\sigma_2 = g^{\alpha\beta} H(m)^r.
\]
Then it returns the signature $\sigma = (\sigma_1, \sigma_2)$.
\item $\verify_{\mathrm{Wat}}(\pk, m, \sigma):$ The verification algorithm returns $1$ (equivalently, $\valid$) if the equation
\[
e(g^\alpha, g^\beta)\cdot e(\sigma_1, H(m)) = e(g, \sigma_2)
\]
holds. Otherwise $0$ is returned.
\end{itemize}
The security of \cite{water05}'s signature scheme is based on the intractability of the computational Diffie-Hellman problem (see \cref{subsec:prelim signature scheme}).

\begin{theorem}[{\cite{water05,HK08}, restated in \cite[Theorem 2.3]{HJK12}}]\label{thm:water scheme security}
    Suppose there exists a forger $\mathcal{F}$ that $(t, q, \varepsilon)$-breaks the EUF-CMA security of $\mathsf{Sig}_{\mathsf{Wat}}$. Then there exists an algorithm $\mathcal{A}$ $(\varepsilon', t')$-solving the computational Diffie--Hellman problem in $\mathbb{G}$ in time $t' = O( t)$ with success probability $\varepsilon' =\Omega \left( \varepsilon \cdot \tfrac{1}{q \sqrt{\ell}  }\right)$.
\end{theorem}
 
\begin{lemma}\label{lem:water scheme}
    The \cite{water05} signature scheme has deterministic verification and satisfies uniformity and exact regularity.
\end{lemma}
We need the following simple facts about bilinear-pairings.
\begin{proposition}\label{prop:homomorphic}
    For $u\in \mathbb{G}_1$ and $v, v'\in \mathbb{G}_2,$
    \[ e(u, v) e(u, v') = e(u, v v')\]
    Similarly, for $u,u'\in \mathbb{G}_1$ and $v\in \mathbb{G}_2,$
    \[ e(u, v) e(u', v) = e(uu', v)\]
\end{proposition}
\begin{proposition}\label{prop:uniqueness}
    If $g_1$ is a generator of $\mathbb{G}_1,$ and $e(g_1,v) = e(g_1,v')$ for some $v,v'\in \mathbb{G}_2$ then $ v= v'.$

    Similarly, if $g_2$ is a generator of $\mathbb{G}_2,$ and $e(u,g_2) = e(u',g_2)$ for some $u,u'\in \mathbb{G}_1$ then $ u= u'.$ 
\end{proposition}

\begin{proof}[Proof of \cref{lem:water scheme}]
    Clearly, verification is deterministic. We show that for any message $m$ and signature $ (\sigma_1, \sigma_2) \in \mathbb{G}\times \mathbb{G}$ where $ e(g^\alpha, g^\beta)\cdot e(\sigma_1, H(m)) = e(g, \sigma_2) ,$ then there exists a unique $ r\in\Z_p$ s.t.  $ \sigma_1= g^r $ and $\sigma_2 = g^{\alpha \beta} H(m)^r.$ Indeed, uniqueness of $r$ follows since $g$ is a generator of $\mathbb{G}.$ Now, using the homomorphic property in \cref{prop:homomorphic} of the pairing, we have
    \[ e(g^\alpha, g^\beta)\cdot e(\sigma_1, H(m))  = e(g, g^{\alpha\beta } H(m)^r) = e(g,\sigma_2) \Rightarrow \sigma_2 = g^{\alpha\beta } H(m)^r \]
    where we use \Cref{prop:uniqueness}. 

    Since $g$ is a generator, $g^r \neq g^{r'}\forall r\neq r',$ hence the support of $\sign(\sk, m)$ i.e. $ \set{ (g^r, g^{\alpha\beta} H(m)^r) |r\in \Z_p}$ has size exactly $ p.$ Clearly, $\sign(\sk, m)$ generates a uniformly random signature from the support of $\sign(\sk, m)$, since $g$ is a generator of $ \mathbb{G}$ and $r$ is sampled uniformly from $\Z_p.$
\end{proof}

\section{Relegated proofs from Section~\ref{sec:embed general signature scheme into Ising}} \label{sec:defer proofs from Section embedding signature into Ising}
We prove \cref{lem:embed general circuit into Ising} from Section~\ref{sec:embed general signature scheme into Ising}. We need the following definitions from \cite{moitra2021learning}.

\begin{definition}[NAND gadget graph,  {\cite[Definition 40]{moitra2021learning}}]\label{def:NAND gadget}
For a parameter $w > 0$, let $G_w$ be the weighted graph defined as follows.
$G_w$ has three vertices, $v_1, v_2,$ and $v_3$. There is an edge of
weight $-w$ between $v_1$ and $v_2$ and edges of weight $-2w$
between $v_1, v_3$ and $ v_2, v_3$. 
We give $v_1$ and $v_2$ biases of $w$ and $v_3$ a bias of $2w$.
\end{definition}

We can compose the NAND gadget graphs according to the following procedure.
\begin{definition}[{\cite[Lemma 42]{moitra2021learning}}]\label{def:composing NAND gate}
Fix a parameter $w > 0.$ Let $G$ be a weighted graph and $i$ and $j$ be two of its vertices. Next, let $G'$ be the weighted graph formed by making the following changes to $G$.
\begin{enumerate}
    \item The biases of $i$ and $j$ are increased
by $w$. 
\item Decrease the weight of the edge between $i$ and $j$ by $w$. If there is no edge between $i$ and $j$ in $G$, this means add an edge of weight $-w$ between $i$ and $j.$
\item Finally, $G'$ has a new vertex $k$ which has a bias of $2w$, is connected to
$i$ and $j$ by edges of weight $-2w$, and has no other edges. 
\end{enumerate}

For a circuit $\mathcal{C}$ on $n$ input bits consisting of $r$ NAND gates, let $ G_{\mathcal{C},w}$ be the weighted graph on $n+r$ vertices constructed by applying the above procedure to each NAND gate, starting from the empty graph on $n$ vertices.
\end{definition}

\begin{proof}[Proof of \cref{lem:embed general circuit into Ising}]
    Construct the weighted graph on $m:=n+r$ vertices corresponding to the circuit of $f$. Let $ \mu:\set{\pm 1}^m \to \R_{\geq 0}$ be the corresponding Ising model. Clearly, each edge weight and vertex weight has magnitude at most $ O(mw)$, so the entries of the interaction matrix and external field have magnitude at most $ O(mw).$ We label each NAND gate $g=((i,j)\to k)$ where $i,j$ are the two input
bits of gate $g$, and $k$ is its output bit.   The (un-normalized) density function of $\mu$ is given by
    \[ \mu(x) = \exp( w \sum_{g=((i,j)\to k)} A(x_i, x_j, x_k)) \]
    where the sum is over the NAND gates of the circuit, and 
    \[ A(x_i, x_j, x_k) = -x_i x_j -2x_k (x_i +x_j) + x_i + x_j + 2x_k.  \]
    Define $\mathrm{NAND}_{\pm} : \set{\pm 1}^2 \to \set{\pm 1}$ by 
    $
    \mathrm{NAND}_{\pm}(a,b)
    =
    \begin{cases}
        -1 &\text{ if } a=b=+1\\
        +1 & \text{ otherwise}
    \end{cases}$
then 
\begin{equation}\label{eq:NAND gadget log density}
  A(x_i,x_j,x_k) = \begin{cases}
        3 &\text{ if }  \mathrm{NAND}_{\pm}(x_i,x_j) =x_k\\
        \leq -1 & \text{otherwise}.
    \end{cases}  
\end{equation}
 Let  $\Lambda$ be the set of $x \in \set{\pm 1}^m,$ so that for each NAND gate $ g = ((i,j)\to k),$ $\mathrm{NAND}_{\pm}(x_i,x_j) =x_k.$
    Clearly, $ \Lambda \neq \emptyset.$ Observe that $ \nu$ is the uniform distribution over $ \Lambda.$ By \eqref{eq:NAND gadget log density}, $\mu(x) = W:= \exp(3 w r)$ for $x\in \Lambda,$ and $\mu(x)\leq W \exp(-4w)$ for $x\not\in \Lambda,$ so the conditional distribution of $\mu$ conditioned on the event $x\in \Lambda$ is exactly $\nu$ and
    \[ d_{TV}(\nu, \mu) \leq \Pr_{x\sim \mu} [x\not\in \Lambda] \leq \frac{\sum_{x\not\in \Lambda} \mu(x)}{\sum_{x\in \Lambda} \mu(x)} \leq 2^m \exp (-4 w ). \]

    Similarly, for $S\subseteq [m]$ and $\tau\in \set{\pm 1}^S,$ the (un-normalized) density function of $\mathcal{T}_{\mathbf{\lambda}^\tau} \mu$ is given by
    \[ \mathcal{T}_{\mathbf{\lambda}^\tau} \mu (x) = \exp( w \sum_{i,j,k} A(x_i, x_j, x_k) + w\sum_{i\in S} \tau_i x_i ). \]
     
      Let  $\Lambda^\tau$ be the intersection of $\Lambda$ with $\set{x: x_i = \tau_i \forall i\in S}.$ The assumption $ \Pr_{x\sim \nu}[x_S=\tau] >0$ implies $ \Lambda^{\tau} \neq \emptyset.$ Observe that $ \nu^{\tau}$ is the uniform distribution over $ \Lambda^{\tau}.$ By \eqref{eq:NAND gadget log density}, $\mathcal{T}_{\mathbf{\lambda}^\tau} \mu(x) =W':= \exp(3 w r + w |S|)$ for $x\in \Lambda^{\tau},$ and $\mathcal{T}_{\mathbf{\lambda}^\tau} \mu(x) \leq W'\exp(-2w)$ for $x\not\in \Lambda^\tau,$ so the conditional distribution of $\mathcal{T}_{\mathbf{\lambda}^\tau} \mu$ conditioned on the event $x\in \Lambda^{\tau}$ is exactly $\nu^{\tau}$ and
    \[ d_{TV}(\nu^{\tau}, \mathcal{T}_{\mathbf{\lambda}^\tau} \mu ) \leq \Pr_{x\sim \mathcal{T}_{\mathbf{\lambda}^\tau} \mu } [x\not\in \Lambda^{\tau}] \leq \frac{\sum_{x\not\in \Lambda^{\tau}} \mathcal{T}_{\mathbf{\lambda}^\tau} \mu(x)}{\sum_{x\in \Lambda^{\tau}} \mathcal{T}_{\mathbf{\lambda}^\tau} \mu(x)} \leq 2^m \exp (-2 w ). \]
    Choosing $ w = Cm$ for a sufficiently large constant $C>0,$  we obtain the desired conclusions.
    
\end{proof}

\section{Relegated proofs from Section~\ref{sec:embed general Ising into near critical Ising}}
\label{sec:defer proof from ising sparsification}

\begin{proof}[Proof of \cref{prop:property of q^+}]
Note that
\[ f_\beta'' (\alpha) = -\frac{1}{\alpha(1-\alpha)} + 4\beta \]
When $\beta >1,$ $f_{\beta}''$ has exactly two roots $ \alpha^+\in (\frac{1}{2},1), \alpha^-\in (0,\frac{1}{2})$ where $ \alpha^+ + \alpha^-= 1.$ Note that $f_{\beta}''(\alpha) >0$ for $ \alpha \in (\alpha^-, \alpha^+) $ and $f_{\beta}''(\alpha) <0$ for $\alpha \in (\alpha^+,1)\cup (0,\alpha^-).$ Note also that $ f_{\beta}'(1/2) =0$ and $f_{\beta}''(\alpha) >0$ for $ \alpha \in (\alpha^-, \alpha^+) $ hence $f_{\beta}'(\alpha^+)>0$ and  $ f_{\beta}'(\alpha^-)<0.$ 

Next, since $ \lim_{\alpha \to 1^-} f_{\beta}'(\alpha) = -\infty$ and $f'_{\beta}(\alpha)$ is strictly decreasing in $ (\alpha^+,1),$ the intermediate value theorem implies that there exists a unique $ q^+_{\beta}\in (\alpha^+,1)\subseteq (1/2, 1)$ where $f_{\beta}'(q^+_{\beta} )=0.$ Similarly, since $ \lim_{\alpha \to 0^+} f_{\beta}'(\alpha) = +\infty$ and $f'_{\beta}(\alpha)$ is strictly decreasing in $ (0,\alpha^-),$ the intermediate value theorem implies that there exists a unique $ q^-_{\beta}\in (0,\alpha^-)\subseteq (0,1/2)$ where $f_{\beta}'(q^-_{\beta} )=0.$ Since $ f_{\beta}'(1-q^+_\beta)=0,$ and $ 1-q^+_\beta \in  (0,\alpha^-),$ we must have that $ 1 -q^+_\beta = q^-_{\beta}.$

Since $q^+_{\beta}\in (\alpha^+,1)$, $ q^-_{\beta}\in (0,\alpha^-)$ and $f_{\beta}''(\alpha) <0$ for $\alpha \in (\alpha^+,1)\cup (0,\alpha^-),$ we have $ f_{\beta}'' (q^+_{\beta})<0$ and $  f_{\beta}'' (q^-_{\beta})<0.$ Since $f'_\beta$ is strictly decreasing in $(0, \alpha^-)$, $f'_\beta(\alpha) >f'_\beta(q^-_\beta) =0$ for $\alpha < q^-_\beta,$ and $f'_\beta(\alpha) <f'_\beta(q^-_\beta) =0 $ for $\alpha \in (q^-_\beta, \alpha^-].$ Since $ f'_\beta $ is strictly increasing in $ (\alpha^-,\alpha^+),$ $ f'_\beta(\alpha) < f'_\beta(1/2) = 0$ for $\beta \in (\alpha^-, 1/2).$ Hence, $ f_\beta$ is strictly increasing on $ (0,q^-_\beta),$ and strictly decreasing on $(q^-_\beta, 1/2).$ By an analogous argument, $ f_\beta$ is strictly increasing on $ (1/2, q^+_\beta),$ and strictly decreasing on $ (q^+_\beta, 1).$
This establishes the claim that $q^+_\beta$ and $q^-_\beta$ are maximizers of $ f_\beta.$

    Let $ g(\alpha) = \frac{  \ln\!\left(\frac{\alpha}{1-\alpha}\right)}{2\alpha -1 } $ then $ g(q^+_\beta) = \beta.$ To show that $ q^+_\beta$ is a strictly increasing function of $\beta,$ we only need to show that $ g$ is a strictly increasing function on $(1/2,1).$
    \begin{align*}
        g'(\alpha) &= \frac{(\frac{1}{\alpha} + \frac{1}{1-\alpha}) (2\alpha-1) -2 \ln\!\left(\frac{\alpha}{1-\alpha}\right)  }{(2\alpha-1)^2} :=\frac{h(\alpha)}{(2\alpha-1)^2} \\
     \text{where}\quad   h(\alpha) &= \frac{1}{1-\alpha}-\frac{1}{\alpha}-2\ln\alpha +2\ln(1-\alpha)\\
        h'(\alpha) &= \frac{1}{(1-\alpha)^2}+ \frac{1}{\alpha^2} -\frac{2}{\alpha}-\frac{2}{1-\alpha} = \frac{(2\alpha-1)^2}{(1-\alpha)^2 \alpha^2}
    \end{align*}
    Note that $h(1/2)=h'(1/2)=0$ and $h'(\alpha) > 0$ for $\alpha \in (1/2,1),$ hence $ h(\alpha)>0 $ for $\alpha \in (1/2,1),$ which implies $ g'(\alpha) >0$ for $\alpha\in (1/2,1).$ This establishes the second claim.
\end{proof}
\begin{proof}[Proof of \cref{prop:deviation bound for f alpha for bounded beta}]
We omit the subscript $\beta$ in the proof for convenience.
    First, we compute    
    \[ f^{(3)} (\alpha) =(f''(\alpha))' = (-\frac{1}{\alpha(1-
        \alpha)} + 4\beta)' = \frac{1}{\alpha^2} - \frac{1}{(1-\alpha)^2} .\]
        Since $ \beta \in (1,\beta_0)$ where $ \beta_0=2,$
        by \cref{prop:property of q^+},  $ q^+_\beta \leq q^+_{\beta_0} <1 .$

   By Taylor's theorem, there exists $\delta'_0>0$ that depends only on $\beta$ so that: for $\alpha\in [q^+_{\beta}-\delta'_0, q^+_{\beta} +\delta'_0]$, we have
    \[ |f(\alpha) - f(q^+_{\beta}) - \frac{1}{2} f''(q^+_{\beta}) (\alpha-q^+_{\beta})^2 |\leq |\alpha - q^+_{\beta}|^3 O(\max_{\alpha': |\alpha' - q^+_{\beta}| \leq \delta'_0}  |f^{(3)}(\alpha') |)\leq C_2 |\alpha - q^+_{\beta}|^3 , \]
    for some constant $C_2>0$ dependent only on $\beta.$
    
    Note that by \cref{prop:property of q^+}, $f''(q^+_{\beta})<0 . $ Choosing $ C_1 = - \frac{1}{2} f''(q^+_{\beta})>0$ yields: 
    \begin{equation}\label{eq:helper inequality for f deviation}
        |f(\alpha)  - f(q^+_\beta) + (\alpha-q^+_\beta)^2 C_1| \leq  C_2 |\alpha-q^+_\beta|^3 .
    \end{equation}

Let $\delta_0 = \min \set{\delta'_0, \frac{C_1}{3C_2}, q^+_{\beta}-1/2, 1- q^+_{\beta}},$ and take  $\delta\in (0, \delta_0].$ 

    By \cref{prop:property of q^+}, $ f$ is increasing on $ [1/2, q^+_{\beta})$ and decreasing on $(q^+_{\beta}, 1)$, so we only need to prove the claim for $ \alpha'\in \set{q^+_{\beta} -\delta, q^+_{\beta}+\delta}. $ Using \cref{eq:helper inequality for f deviation} for $\alpha'\in \set{q^+_{\beta}-\delta, q^+_{\beta} +\delta}$ and $\alpha \in [q^+_{\beta}-\delta/2, q^+_{\beta} +\delta/2]$, we have 
    \begin{align*}
         f(\alpha') - f(\alpha)&= f(\alpha')-f(q^+_{\beta}) + f(q^+_{\beta})-  f(\alpha)\\
         &\leq -\frac{3 C_1}{4}\cdot \delta^2 + \frac{9C_2}{8} \cdot \delta^3 \leq - C\delta^2  
    \end{align*}
    when we set $C= \frac{3C_1}{8}.$ 
\end{proof}
\begin{proof}[Proof of \Cref{prop:counting-to-sampling}]
We build a configuration $x^*\in \set{\pm 1}^n$ using the following recursive procedure.
For $k\in \set{1,\cdots, n}$ do the following
\begin{itemize}
    \item For $x_k \in \set{\pm 1},$ using $t$ i.i.d.\ samples from $\nu(\cdot\mid X_{[k-1]}=x_{[k-1]}^*)$, estimate $\nu(X_k=x_k\mid X_{[k-1]}=x_{[k-1]}^*)$ by the fraction of samples with coordinate $k$ equal to $x_k.$  Let $ \hat{\nu}(X_k=x_k\mid X_{[k-1]}=x_{[k-1]}^*)$ be the estimate. Note that for $k=1,$ this is equivalent to estimating the probability $\nu(X_1=x_1)$ for $x_1\in\set{\pm 1}.$  
    \item Choose $x_k^*= \arg\max_{x_k \in \set{\pm 1}} \hat{\nu}(X_k=x_k\mid X_{[k-1]}=x_{[k-1]}^*) $  (break ties arbitrarily).
\end{itemize}
    
    Let $ x^* = x_1^*x_2^*\cdots x_n^*\in \set{\pm 1}^n.$ For $1\leq k\leq n,$ let $ W_k = \nu(X_k = x_k^*|X_{[k-1]} = x_{[k-1]}^*)$ and $ \hat{W}_k $ be the empirical estimator of $W_k$ using $t$ fresh i.i.d. samples from $\nu(\cdot\mid X_{[k-1]}=x_{[k-1]}^*).$
 Let $W = \prod_{k=1}^n W_k = \frac{\nu(x^*)}{Z}$ and $\hat{W}= \prod_{k=1}^n \hat{W}_k, $
 then $ \E[\hat{W}]=W.$ The algorithm outputs $\hat{Z}= \frac{\nu(x^*)}{\hat{W}}$ as the estimator for the partition function $Z$ of $\nu.$
    
For the analysis, we need the following propositions.
\begin{proposition}[Unbiased mean-estimator theorem]\label{prop:unbiased mean estimator}
Let $X_1, \ldots, X_t$ be iid non-negative r.v.'s with common 
mean $\mu = \E[X_i]$ and variance $\sigma^2 = \Var(X_i)$, 
and let $Y$ be the sample mean $\frac{1}{t}\sum_{i=1}^t X_i$. 
Then we have
\[
\Pr\left[ |Y - \mu| \ge \varepsilon \mu \right] \le \frac{1}{4}
\]
provided $t \ge \frac{4}{\varepsilon^2}\frac{\sigma^2}{\mu^2}.$    
\end{proposition}
\begin{proof}
    Note that $\E[Y] = \mu$ and $\Var(Y) = \frac{\sigma^2}{t}$. 
Chebyshev's inequality then gives
\[
\Pr\left[ |Y - \mu| \ge \varepsilon \mu \right] 
    \le \frac{\mathrm{Var}(Y)}{(\varepsilon \mu)^2}
    = \frac{\sigma^2/t}{\varepsilon^2 \mu^2}. \]
\end{proof}
 Since the $\hat{W}_k$ are independent:
 \begin{equation}\label{eq:counting to sampling helper bound gamma hat W}
     \gamma(\hat{W}) = \frac{\Var(\hat{W})}{\E[\hat{W}]^2} = \frac{\E[\hat{W}^2] }{\E[\hat{W}]^2} -1 = \prod_{k=1}^n \frac{\E[\hat{W}_k^2] }{\E[\hat{W}_k]^2} -1 =\prod_{k=1}^n  (1+  \frac{\Var(\hat{W}_k)}{\E[\hat{W}_k]^2 }) - 1
 \end{equation}
 
 Since $ \hat{W}_k$ is the average of $t$ independent Bernoulli trials with mean $W_k,$ $\hat{W}_k$ is an unbiased estimator for $ W_k$ i.e. $ \E[\hat{W}_k] =W_k$ and moreover,
 \begin{equation}\label{eq:counting to sampling helper bound gamma hat W k}
  \gamma(\hat{W}_k):=\frac{\Var(\hat{W}_k)}{\E[\hat{W}_k]^2 } \leq \frac{W_k(1-W_k)/t}{W_k^2} \leq \frac{1}{t W_k}
 \end{equation}
We lower bound $W_k$ using the following proposition.
\begin{proposition}\label{prop:lower bound W k}
With probability at least $ 1- \frac{8}{t},$ $W_k \geq 1/4.$
\end{proposition}

\begin{proof}
    For $c \in \set{\pm 1},$ 
    let $W_k^c: =\nu(X_k = c |X_{[k-1]} = x_{[k-1]}^*)$ and let $\hat{W}_k^c$ be an empirical estimator of $ W_k^c$ using $t$ samples from $ \nu(\cdot |X_{[k-1]} = x_{[k-1]}^*). $ 
    Suppose $W_k^c < 1/4.$ By Chebyshev's inequality:
    \[\Pr[\hat{W}_k^c \geq 1/2] \leq \Pr[\hat{W}_k^c - W_k^c > 1/4]\leq \frac{\Var(\hat{W}_k^c)}{(1/4)^2} \leq \frac{W_k^c/t}{(1/4)^2}\leq \frac{4}{t}. \]
    Taking union over $ c\in \set{\pm 1},$ then with probability $ \geq 1-\frac{8}{t},$ $W_k \geq 1/4.$
\end{proof}
Let $\mathcal{E}$ be the event that $ W_k\geq \frac{1}{4}\forall k \in [n],$ which happens with probability $ 1-\frac{8n}{t}$ by a union bound.

Conditioned on $\mathcal{E}$, \cref{eq:counting to sampling helper bound gamma hat W} and \cref{eq:counting to sampling helper bound gamma hat W k} together implies that for $ t \geq \frac{128 n}{\varepsilon^2} $, $\gamma(\hat{W}) \leq \frac{\varepsilon^2}{32}$
and thus, by \cref{prop:unbiased mean estimator}, with probability $ \geq 1-1/8,$  $\hat{W} = (1\pm \varepsilon/2) W$ and thus $\hat{Z}= (1\pm \varepsilon ) Z.$ Hence, for $ t\geq  \frac{128 n}{\varepsilon^2} , $
\[ \Pr[\hat{Z} = (1\pm \varepsilon) Z] \geq \Pr[\mathcal{E}]-1/8 \geq 7/8 - \frac{8n}{t} \geq 3/4. \]
Take $t = \frac{128 n}{\varepsilon^2}.$ The total number of samples is $K = \Theta(nt) =\Theta(\frac{n^2}{\varepsilon^2}).$

If we only have access to samples from $\hat{\nu}(\cdot |X_{[k]} = x_{[k]}^*)$ that is $\delta'$-close to $\nu(\cdot |X_{[k]} = x_{[k]}^*) $ in total variation distance, then by a standard coupling argument, the samples from $\hat{\nu}(\cdot |X_{[k]} = x_{[k]}^*)$  are exactly equal to samples from $\nu(\cdot |X_{[k]} = x_{[k]}^*)$ for all $k\in [n]$ with probability $\geq 1-K\delta',$ hence the estimator $\hat{Z}$ constructed with the samples from $\hat{\nu}(\cdot)$ satisfies,
\[  \Pr[\hat{Z} = (1\pm \varepsilon) Z] \geq  3/4 - K\delta' \geq 5/8\]
when we choose $\delta'=\frac{1}{8K}.$
\end{proof}

\begin{proposition}\label{prop:inverse function of edge and vertex weight}
    Let $ \varphi_{\mathrm{vertex},\beta},  \varphi_{\mathrm{edge},\beta}$ be as defined in \cref{eq:define weight map function}, then their inverses are as defined in \cref{eq:inverse of edge and weight map function}.
\end{proposition}

\begin{proof}
Recall that:
  \begin{align*}
      \varphi_{\mathrm{vertex},\beta}(h) &= \frac{1}{2}\ln \left(\frac{\cosh(h +\beta (2 q^+_\beta-1) )}{\cosh (h +\beta (2 q^-_\beta -1) 
      )}\right)
      \\
      \varphi_{\mathrm{edge},\beta}(w) &= \frac{1}{2}\ln \left(\frac{((q^+_\beta)^{2}+(1-q^+_\beta)^{2})\,\exp(w)
      + 2q^+_\beta(1-q^+_\beta)\,\exp(-w) }{((q^+_\beta)^{2}+(1-q^+_\beta)^{2})\,\exp(-w)
      + 2q^+_\beta(1-q^+_\beta)\,\exp(w) }\right)
  \end{align*}
  Let $ a:= \varphi_{\mathrm{vertex},\beta}(h) $ and $u = \beta  (2 q^+_\beta-1) = -\beta  (2 q^-_\beta-1)$ then
  \begin{align*}
     \exp(2a) &=  \frac{\cosh(h +\beta (2 q^+_\beta-1) )}{\cosh (h +\beta (2 q^-_\beta -1) 
      )} = \frac{\exp(h +u) + \exp(- h - u) }{\exp(h - u ) + \exp(- h + u)  }\\
      \Rightarrow \tanh(a) = \frac{\exp(2a) -1}{\exp(2a)+1} &= \frac{(\exp(h) - \exp(-h)) (\exp(u) -\exp(-u)) }{(\exp(h) + \exp(-h)) (\exp(u) +\exp(-u))}  = \tanh (h) \tanh(u)\\
      \Rightarrow h = \tanh^{-1}\left(\frac{\tanh(a)}{\tanh(u)}\right) &=\tanh^{-1}\left(\frac{\tanh(a)}{2 q^+_\beta -1}\right) 
  \end{align*}
  For the last equality, recall from the definition of $q^+_\beta, q^-_\beta$ (see \cref{eq:def of q+ and q-}) that  $2u = \ln \frac{q^+_\beta}{q^-_\beta }$ and $q^+_\beta + q^-_\beta=1.$ Hence  \[\tanh (u) = \frac{e^{2u}-1}{e^{2u}+1} = \frac{\frac{q^+_\beta}{q^-_\beta } - 1}{ \frac{q^+_\beta}{q^-_\beta } + 1} = \frac{q^+_\beta - q^-_\beta}{q^+_\beta + q^-_\beta} = 2 q^+_\beta-1\]
Similarly, let $b := \varphi_{\mathrm{edge},\beta}(w) $ then
\begin{align*}
    \tanh(b) = \frac{\exp(2b)-1}{\exp(2b)+1} &= \frac{\left((q^+_\beta)^{2}+(1-q^+_\beta)^2 -  2q^+_\beta(1-q^+_\beta) \right) (\exp(w)-\exp(-w)) }{\left((q^+_\beta)^{2}+(1-q^+_\beta)^2 + 2q^+_\beta(1-q^+_\beta) \right) (\exp(w)+\exp(-w)) }\\
    &= (2 q^+_\beta -1)^2 \tanh(w)\\
    \Rightarrow w = \tanh^{-1}\left(\frac{\tanh(b)}{(2 q^+_\beta -1)^2}\right).
\end{align*}
\end{proof}


\end{document}